\documentclass{article}

\usepackage[preprint]{neurips_2026}

\usepackage{tcolorbox}
\usepackage[utf8]{inputenc} 
\usepackage[T1]{fontenc}    
\usepackage{hyperref}       
\usepackage{url}            
\usepackage{booktabs}       
\usepackage{amsfonts}       
\usepackage{nicefrac}       
\usepackage{microtype}      
\usepackage{xcolor}         
\usepackage{graphicx}
\usepackage{subcaption}
\usepackage{booktabs} 
\usepackage{multirow} 
\usepackage{threeparttable}
\usepackage{enumitem}
\definecolor{darkblue}{RGB}{0, 0, 139}
\usepackage{hyperref}
\hypersetup{
    colorlinks=true,
    linkcolor=orange,
    citecolor=darkblue,
    urlcolor=blue
}

\usepackage{amsmath}
\usepackage{amssymb}
\usepackage{mathtools}
\usepackage{amsthm}
\usepackage{bm} %
\usepackage{algorithmicx}%
\usepackage{wrapfig}

\theoremstyle{plain}
\newtheorem{theorem}{Theorem}
\theoremstyle{definition}
\newtheorem{intuition}{Intuition}
\theoremstyle{remark}
\newtheorem{remark}{Remark}

\usepackage[capitalize,noabbrev]{cleveref}

\usepackage[]{todonotes}
\definecolor{HaColor}{RGB}{127, 255, 212}

\title{Embracing Biased Transition Matrices for Complementary-Label Learning with Many Classes}

\newcommand{\affil}[1]{\textsuperscript{\textbf{#1}}}

\author{%
  \textbf{Tan-Ha Mai}\affil{1,3}\thanks{Portions of this work were conducted at NTU and during Tan-Ha Mai's internship at the RIKEN AIP.} \quad
  \textbf{Chao-Kai Chiang}\affil{2} \quad
  \textbf{Han-Hwa Shih}\affil{1}
  \\[4pt]
  \textbf{Gang Niu}\affil{3} \quad
  \textbf{Masashi Sugiyama}\affil{2,3} \quad
  \textbf{Hsuan-Tien Lin}\affil{1}\thanks{Correspondence to: Hsuan-Tien Lin \texttt{<htlin@csie.ntu.edu.tw>}.}
  \\[6pt]
  \textsuperscript{1}National Taiwan University \\
  \textsuperscript{2}The University of Tokyo \\
  \textsuperscript{3}RIKEN Center for Advanced Intelligence Project
}

\begin{document}

\maketitle


\begin{abstract}
\emph{Complementary-label learning} (CLL) is a weakly supervised paradigm where instances are labeled with classes they do not belong to. Despite a decade of research, CLL methods remain competitive mainly on 10-class classification, with scaling to large label spaces continuing to be an enduring bottleneck. This limitation stems from the common assumption of uniform label generation in traditional methods, which fatally dilutes the learning signal in many-class settings. In this paper, we demonstrate that this long-standing barrier can be overcome by deliberately designing a biased (non-uniform) generation process that restricts complementary labels to a subset of classes. This finding motivates us to propose Bias-Induced Constrained Labeling (BICL), a principled framework spanning data collection to training that leverages this bias. BICL enables effective learning on CIFAR-100 and TinyImageNet-200, achieving more than sevenfold accuracy improvements over traditional methods. Our findings establish a new trajectory for making CLL feasible for many classes in real-world applications.
\end{abstract}
\section{Introduction}
\label{introduction}

Large and accurately annotated data have driven the development of machine learning in recent decades. To address the high cost and limited efficiency of human labeling, weakly supervised learning (WSL) leverages \emph{imperfect supervision} i.e., training signals that are incomplete, ambiguous, or noisy, yet cheaper to acquire than ground-truth labels via humans or agents~\cite{ aclimage2025}. WSL contains various scenarios~\cite{sugiyama2022machine}, including positive-unlabeled learning~\cite{elkan2008learning,kiryo2017positive}, confidence information~\cite{cao2021learning,wang2023binary}, pairwise relationship~\cite{BNS18, bao2022pairwise}, complementary-label learning (CLL)~\cite{ishida2017learning,scl2020}, partial-label learning~\cite{jin2002learning,lv2020progressive}, or noisy-label learning data~\cite{noisy-ll-2013,patrini2017making}. 

In this work, we focus on CLL, where weak supervision specifies a class that an instance does \textit{not} belong to. CLL has attracted increasing attention~\cite{aclimage2025, ishida2017learning, wang2024learning} as it is potentially appealing when identifying an \textit{incorrect} or \textit{opposite} label is easier than determining the true class, making it promising for real-world applications such as crowdsourcing and preference-based feedback~\cite{wang2024climage}. For example, in crowdsourcing, annotators may find it easier to label a tomato as \textit{not an animal} than to decide whether it should be categorized as a ``vegetable'' or a ``fruit''. However, CLL remains one of the most challenging WSL paradigms because each label provides only indirect negative evidence about the true class. As the label space grows, this supervision becomes increasingly ambiguous, and many existing methods struggle to scale beyond small or moderate settings, such as 10 or 20 classes~\cite{aclimage2025,ishida2017learning,scl2020,wang2024climage}, hindering practical deployment. 

Existing CLL models generally assume that complementary labels (CLs) are generated from the true class according to a transition matrix~\cite{ishida2017learning,fwd2018,kim2019nlnl,gao2021discriminative,cpe2023}, which plays a central role in CLL studies.~\cite{ishida2017learning} pioneered the theoretical study of CLL by assuming a zero-diagonal uniform transition matrix, under which each label complementary to the true class is generated with equal probability. \citep{fwd2018} relaxed the assumption by introducing a biased (non-uniform) transition matrix to capture systematic biases during annotation. Such biased transition matrices better reflect realistic annotation mechanisms and have often yielded stronger empirical performance than the uniform ones~\citep{ishida2017learning}, as observed in subsequent studies~\citep{fwd2018,libcll_2024}.

Despite these empirical gains, biased transition matrices remain difficult to exploit in large-scale CLL settings, such as 100-class classification, because it is still unclear how to effectively control and scale the bias in CLs during data collection. 
As a result, existing CLL pipelines lack a systematic approach for \textit{designing} annotation protocols that deliberately induce beneficial biased transition structures.
In this work, we first initiate a pioneering study that uncovers the significant potential of designing an annotation process that concentrates on a few candidate CLs per true class. The potential motivates us to propose \emph{Bias-Induced Constrained Labeling} (BICL), a principled and practical framework for introducing and controlling beneficial bias in the transition matrix during CLL data collection. 
As illustrated on the left-hand side of Figure~\ref{fig:2}, the annotation process of BICL concentrates on a few candidate CLs per \textit{cluster} instead of the true class, allowing it to approximate the success of the pioneering study without requiring access to true labels. 

\begin{figure*}[t]
    \includegraphics[width=1.0\textwidth]{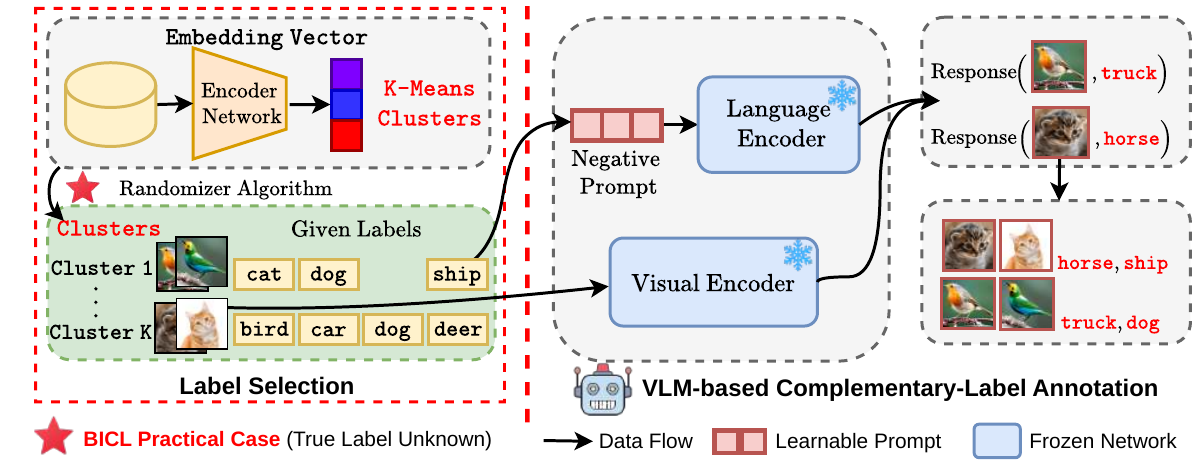}
    \caption{\textbf{BICL Practical Case}: Overview of the proposed practical design for bias-induced constrained labeling (BICL) that operates without true label access.}
    \label{fig:2}
\end{figure*}

BICL is designed to flexibly integrate with a wide range of annotators. We demonstrate its effectiveness using VLM-based annotators to evaluate its scalability potential at reasonable cost, as illustrated on the right-hand side of Figure~\ref{fig:2}.
Empirical results across different VLMs consistently verify that BICL significantly improves CLL performance, particularly in large label spaces where existing CLL methods struggle to scale.
This suggests that carefully structured bias can serve as a design principle for scalable CLL. 

Our main contributions are summarized as follows.
\begin{itemize}
    \item We demonstrate that BICL overcomes a \emph{critical scalability bottleneck} in CLL. In particular, BICL achieves a 7$\times$ improvement (from 6 to 46 percentage points) on CIFAR-100 and an 8$\times$ improvement (from 4 to 32 percentage points) on TinyImageNet-200 compared to 
    \citep{ishida2017learning}, whereas prior methods \emph{struggle to scale beyond} datasets with more than 20 classes.
    \item We develop an information-theoretic lower bound to characterize the potential of leveraging a biased transition matrix.
    This analysis explains why \citep{fwd2018} outperforms \citep{ishida2017learning}, and further shows that the reduction in conditional entropy induced by BICL is a key factor underlying its performance gains over \citep{fwd2018}.
    \item We conduct a systematic ablation study to isolate the contributions of the key components in BICL. The results support our theoretical analysis and demonstrate how candidate-set construction and transition-matrix bias can be controlled during CLL data collection to drive the observed performance gains.
\end{itemize}

\section{Background}
\label{Background}

In this section, we establish the notation and mathematical background for classification, moving from the standard supervised setting to the more specialized CLL framework.

\subsection{Ordinary-Label Classification}
\label{sec:ord-label-classification}

In the standard supervised classification setting, we assume that a $d$-dimensional pattern $\mathbf{x} \in \mathbb{R}^d$ and its corresponding class label $y \in \mathcal{Y} = \{1, \ldots, C\}$ are sampled jointly from an unknown underlying density $p(\mathbf{x},y)$. The dataset is denoted as $\mathcal{D} = \{(\mathbf{x}_i, y_i)\}_{i=1}^{n}$, where $n$ represents the number of training data, and these $n$ instance-label pairs are sampled independently. 

The goal of ordinary multi-class classification is to learn a classifier that minimizes the classification risk. To achieve this, we employ a prediction function vector $\mathbf{g}(\mathbf{x}) = (g_1(\mathbf{x}), \ldots, g_C(\mathbf{x}))^\top$, where $^\top$ denotes the transpose and $g_k(\mathbf{x})$ acts as the prediction score for class $k$. The final classifier $f(\mathbf{x})$ is typically determined by
\begin{equation}
    f(\mathbf{x}) = \operatorname*{arg\,max}_{k \in \{1, \ldots, C\}} g_k(\mathbf{x}).
    \nonumber
\end{equation}
The classification risk is defined as the expectation of a multi-class loss function $\mathcal{L}(\mathbf{g}(\mathbf{x}), y)$ computed on the soft label vector $\mathbf{g}(\mathbf{x})$ over the joint density $p(\mathbf{x}, y)$:
\begin{equation}
    R(\mathbf{g}) = \mathbb{E}_{p(\mathbf{x},y)}[\mathcal{L}(\mathbf{g}(\mathbf{x}), y)],
    \nonumber
\end{equation}
where $\mathbb{E}$ denotes the expectation.
In the context of Empirical Risk Minimization (ERM), since $p(\mathbf{x}, y)$ is unknown, we approximate the risk using the training sample average. 

\subsection{Complementary-Label Classification}
\label{sec:complementary-classification}

In the setting of complementary-label classification, instead of a ground-truth label $y$, we observe a complementary label $\bar{y}$, which specifies a class that the pattern $\mathbf{x}$ does \textit{not} belong to. Formally, given a complementary-label training set $\bar{\mathcal{D}} = \{(\mathbf{x}_i, \bar{y}_i)\}_{i=1}^{n_{\text{tr}}}$, our goal is to learn a classifier $f$ such that the classification error on an unseen ordinary-label test set $\mathcal{D}_{\text{te}} = \{(\mathbf{x}_i, y_i)\}_{i=1}^{n_{\text{te}}}$ is minimized.

To utilize such data effectively, we often model the relationship between the latent true labels and the observed CLs using a transition matrix $Q \in [0, 1]^{C \times C}$, where $Q_{kj} = p(\bar{Y}=j \mid Y=k)$. Since a complementary label cannot be the true label, we typically assume $Q_{kk} = 0$. This transition model allows us to map the ordinary class posteriors $\bm{\eta}(\mathbf{x}) = (p(Y=1|\mathbf{x}), \ldots, p(Y=C|\mathbf{x}))^\top$ to the complementary posteriors $\bar{\bm{\eta}}(\mathbf{x}) = (p(\bar{Y}=1|\mathbf{x}), \ldots, p(\bar{Y}=C|\mathbf{x}))^\top$ via $\bar{\bm{\eta}}(\mathbf{x}) = Q^\top \bm{\eta}(\mathbf{x})$. 

Transition-matrix-based methods, such as forward correction (FWD)~\cite{fwd2018}, exploit this relationship by applying the transition operator to the model predictions before computing the loss on CLs. In the biased-CL setting, these methods typically rely on either a know transition matrix or an estimate obtained from a small clean seed set of ordinary-labeled examples. This small amount of supervision---often only 2 to 5 true-labeled examples per class in existing CLL protocol~\cite{fwd2018,cpe2023,libcll_2024} (see Appendix~\ref{compared_few_true_labels})---has become a common practical assumption for enabling transition-matrix estimation. Given such an estimate, FWD can leverage the biased transition structure encoded in $Q$, rather than assuming uniform CL generation, and thereby improve learning when CLs are systematically biased.


Alternatively,~\citep{ishida2017learning} proposed a direct risk minimization framework without estimating the transition matrix under a stronger assumption that the complementary labels are generated \textit{uniformly}. Specifically, this setting assumes that CLs are sampled uniformly from the set of incorrect classes, meaning $Q_{kj} = \frac{1}{C-1}$ for all $j \neq k$. 
Under this assumption, they derived an unbiased estimator for the classification risk under a moderate condition:
\begin{equation}
    R(\mathbf{g}) = (C-1) \mathbb{E}_{p(\mathbf{x}, \bar{y})} [\bar{\mathcal{L}}(\mathbf{g}(\mathbf{x}), \bar{y})] + M, 
    \nonumber
\end{equation}
where $\bar{\mathcal{L}}$ is the complementary loss computed on $\mathbf{g}(\mathbf{x})$ and $M$ represents a constant. This formulation allows for consistent risk minimization using only CL data.
\section{BICL: Bias-Induced Constrained Labeling}
\label{our-approach}
In this section, we present a biased complementary-label collection protocol designed to systematically acquire CLs with structured bias. The protocol comprises two variants: the \textit{BICL Analysis Case} and the \textit{BICL Practical Case}. Using this protocol, we construct biased complementary-label datasets on four benchmarks: CIFAR-10, CIFAR-20, CIFAR-100~\cite{cifar10}, and TinyImageNet-200~\cite{tinyImageNet}. Notably, to the best of our knowledge, this is the \emph{first work} to scale CLL data collection to a 200-class setting. Prior CLL studies have focused on datasets with at most 20 classes, as evidenced in~\citep{aclimage2025,wang2024climage,libcll_2024}.

\subsection{Biased Complementary-Label Collection Protocol}
\label{sec:3.1}

Applying CLL to large-scale classification remains a major bottleneck, raising the question: \emph{Is CLL practical in many-class settings}? Existing studies focus on datasets with at most $20$ classes~\cite{aclimage2025,wang2024climage}, as performance on larger benchmarks is poor. For example, on CIFAR-100, CLL methods typically achieve only 5 to 7 percentage points of accuracy, even when CLs are synthetically generated under uniform and noiseless assumptions~\cite{ishida2017learning}.

\begin{figure*}[ht]
\centering
\includegraphics[width=1.0\textwidth]{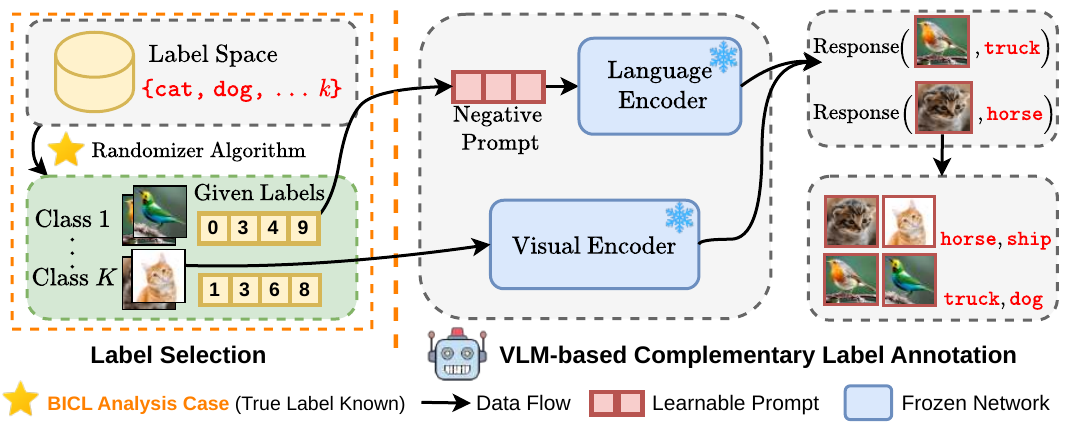}
\caption{\textbf{BICL Analysis Case:} Overview of the proof of concept design for complementary candidate-label selection that operates with true label access. It consists of (1) reducing candidate label selection, (2) VLM-based complementary-label annotation with a \textit{negative prompt}.}
  \label{fig:1}
\end{figure*}

Extending CLL to many-class settings raises a key question: ``\textit{what is the main difficulty?}'' The challenge is twofold. First, from the labeling perspective, whether using human annotators or vision-language model (VLM) labelers, it is impractical to present hundreds of candidate classes for complementary labeling. Second, even though existing protocols, such as those in CLImage~\cite{wang2024climage} and ACLImage~\cite{aclimage2025}, address the first challenge by restricting the candidate space by randomly selecting a small subset (e.g., four) of CLs from the full $C$ classes, the candidate space is offered in a per-instance basis. Then, the approach still produces a diffuse, highly scattered transition matrix, weakening the learning signal.


Our key idea is to examine whether further constraining the effective choice space can induce a stronger and more informative bias in CLs. Motivated by this insight, we propose a novel data collection protocol that systematically introduces structured bias into CLs. The overall pipeline, illustrated in Figure~\ref{fig:1}, comprises two main components, described below.


\textit{Label Selection.}
This component is central to our data-collection protocol and is crucial for effectively collecting biased CLs. Its role is to determine the candidate complementary label set presented to the VLM for annotation, as illustrated on the \emph{\textcolor{blue}{left-hand side}} of Figure~\ref{fig:1}. First, for each class, we randomly sample four labels from the $C-1$ possible CLs\footnote{This is a proof-of-concept setting in which the true label is assumed to be excluded from the candidate complementary label set.} using a randomization algorithm. The further details of why we sample four labels can be found in Appendix~\ref{number_sampled_labels}. Second, we assign the same candidate label set (e.g., \texttt{\{0,3,4,9\}}) to all images belonging to that class, forming an initial pool of plausible complementary-label candidates. The key idea is to fix a small candidate set per class, thereby restricting the effective label space and inducing stronger bias in the resulting CLs. We provide a information-theoretic perspective for why restricting the label space can be beneficial in Section~\ref{sec:intuitive_justification}. To prevent the VLM from developing class-specific label preferences, we randomly shuffle the order of candidate labels for each image when requesting annotation from VLM.

\textit{VLM-based Complementary-Label Annotation.}
Next, we employ a VLM~\cite{LLava} as the annotator in our framework. Prior work by \citep{aclimage2025} shows that VLM-based annotation can be more efficient and scalable than human annotation for collecting CLs. We therefore adopt a VLM as our starting point, while noting that the same protocol can be naturally extended to human annotators. Following \citep{aclimage2025}, we adopt the \textit{negative prompt}: ``\textcolor{darkblue}{\textit{\textbf{Question:}<image> Which label does not belong to this image? Answer the question with a single word:} \{\textit{{label[0]}, {label[1]}, {label[2]}, {label[3]}}\}}''. For each image, the VLM is prompted to select a label $\bar{y}$ that \emph{does not} correspond to the image content from the candidate set determined by the \textit{label selection} component.
This procedure produces a \emph{biased complementary-label dataset}, consisting of image–label pairs where the assigned label is complementary label (i.e., an image of a bird labeled as a ``truck''), while the overall annotation distribution is systematically shaped by our structured candidate-label design. Additional analyses on the effects of VLM choice and negative-prompt variants are provided in Appendices~\ref{vlm_choice} and~\ref{additional_negative_prompt}, respectively.



To systematically evaluate the effectiveness of \textit{BICL Analysis Case} approach illustrated in Figure~\ref{fig:1}, we conduct an ablation study by benchmarking three complementary-label generation assumptions: Uniform~\cite{ishida2017learning}, Biased~\cite{fwd2018}, and BICL (ours). Experiments are performed on CIFAR-10 and CIFAR-20 using two representative CLL algorithms, complementary probability estimation (CPE)~\cite{cpe2023} and forward correction (FWD)~\cite{fwd2018}. The empirical results reported in Table~\ref{tab:1} demonstrate that our BICL-based approach consistently outperforms both Uniform and Biased settings across datasets and algorithms. In particular, BICL achieves improvements of more than 6 percentage points and 3 percentage points over the Biased setting on CIFAR-10 and CIFAR-20, respectively. The performance gap is even larger when compared to Uniform CLs, exceeding 19 percentage points on CIFAR-10 and 50 percentage points on CIFAR-20.

\begin{table*}[htb]
\centering
\caption{Accuracy (mean$\pm$std) on CIFAR-10/20 for FWD and CPE-F under Uniform, Biased, and our BICL Analysis complementary-label assumptions. $\Delta$ is the absolute gain of BICL over the stronger baseline between Uniform and Biased. Best results are bolded.}
\label{tab:1}
\small
\setlength{\tabcolsep}{1.4pt}
\begin{tabular}{l|cccc|cccc}
\toprule
\textbf{Datasets} 
& \multicolumn{4}{c|}{\textbf{CIFAR-10}} 
& \multicolumn{4}{c}{\textbf{CIFAR-20}} \\
\midrule
\textbf{Algorithm}
& Uniform & Biased & \textbf{BICL Analysis} & $\Delta$
& Uniform & Biased & \textbf{BICL Analysis} & $\Delta$ \\
\midrule
FWD       
& 66.38\scriptsize{$\pm$2.16} & 80.04\scriptsize{$\pm$1.88}& \textbf{86.31}\scriptsize{$\pm$0.05} & \textcolor{blue}{$\uparrow$ 6.27}
& 20.50\scriptsize{$\pm$0.75} & 67.79\scriptsize{$\pm$0.36}& \textbf{71.32}\scriptsize{$\pm$0.25} & \textcolor{blue}{$\uparrow$ 3.53} \\
CPE-F  
& 66.24\scriptsize{$\pm$2.04}& 79.59\scriptsize{$\pm$1.87} & \textbf{85.71}\scriptsize{$\pm$0.18} & \textcolor{blue}{$\uparrow$ 6.12}
& 20.73\scriptsize{$\pm$1.17}& 68.40\scriptsize{$\pm$0.19} & \textbf{71.58}\scriptsize{$\pm$0.21} & \textcolor{blue}{$\uparrow$ 3.18} \\
\bottomrule
\end{tabular}
\begin{tablenotes}
\item \textit{\textbf{Biased}}: Following~\cite{fwd2018}, we consider a biased transition setting. As the original transition matrix is not publicly available, we adopt the reproducible construction of~\cite{gao2021discriminative}, assigning three probability levels $(\frac{0.75}{3}, \frac{0.24}{3}, \frac{0.01}{3})$ to CLs for each class $y$ as also used in~\cite{libcll_2024}.

\end{tablenotes}
\end{table*}

These results validate the effectiveness of the oracle-based formulation shown in Figure~\ref{fig:1}. However, it is important to emphasize that this setting is \textit{not deployable} in practice, as it assumes access to the true label to construct noise-free CLs.\footnote{All three dataset assumptions in Table~\ref{tab:1} rely on oracle knowledge of the true label to exclude it from the complementary-label set, resulting in a zero transition probability for the true class.} In realistic scenarios, such oracle access is unavailable and annotation noise is unavoidable.

\subsection{Practical BICL Design}

To move toward a practical and deployable solution, we extend the oracle formulation in Figure~\ref{fig:1} to the realistic protocol in Figure~\ref{fig:2}, which we refer to as the \textit{BICL Practical Case} (hereafter, we use \textit{BICL} for brevity). This practical variant is designed to operate without access to true labels. The key distinction from the idealized framework lies in the \textbf{\textit{label selection}} strategy.
In standard CLL, true labels are not available during training~\cite{ishida2017learning}. Consequently, unlike the \textit{BICL Analysis Case}, we cannot group samples by their true class to assign a shared candidate set and then sample CLs accordingly. Moreover, prior work shows that CLL is highly sensitive to annotation noise, which can substantially degrade performance due to overfitting~\cite{mai2025intra}. These considerations motivate the need for a practical mechanism that (i) estimates groups of samples likely to share the same underlying class and (ii) better controls the noise introduced during annotation.

We adopt representations from a pre-trained \emph{encoder network}~\cite{chen2020simsiam} to group samples with similar feature characteristics. This clustering step assigns a cluster index to each sample and serves as a preprocessing phase for candidate label selection. We provide additional analysis on how the choice of encoder affects BICL performance in Appendix~\ref{effect_encoder}.
As illustrated in Figure~\ref{fig:2}, we cluster these representations using $K$-means to obtain groups of visually similar samples. Each image is then assigned to a cluster $Z_c$, for $c \in \{1,\ldots,C\}$. We provide in-depth analysis on optimal number of clusters for each dataset in Appendix~\ref{number_cluster}.
For each cluster $Z_c$, we randomly sample a fixed-size candidate label set (four labels in our experiments) from the global label space. When querying the VLM for complementary-label annotation, we randomly shuffle the order of candidate labels for each image to prevent order-induced preferences.

\begin{figure*}[htb]
  \centering
  \begin{subfigure}[b]{0.22\textwidth}
    \includegraphics[width=\textwidth]{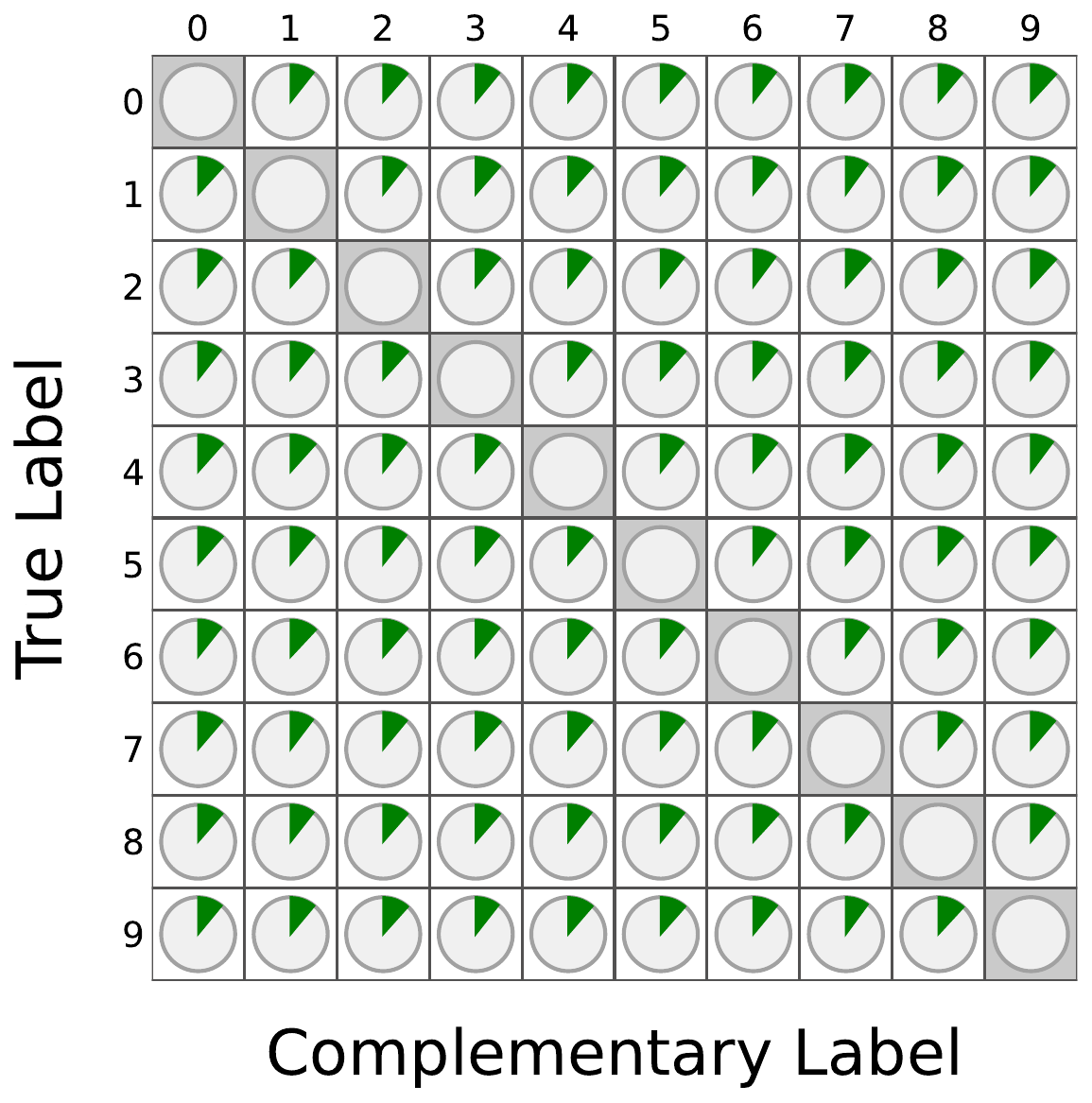}
    \caption{CIFAR-10}
    \label{fig:img1}
  \end{subfigure}
  \begin{subfigure}[b]{0.22\textwidth}
    \includegraphics[width=\textwidth]{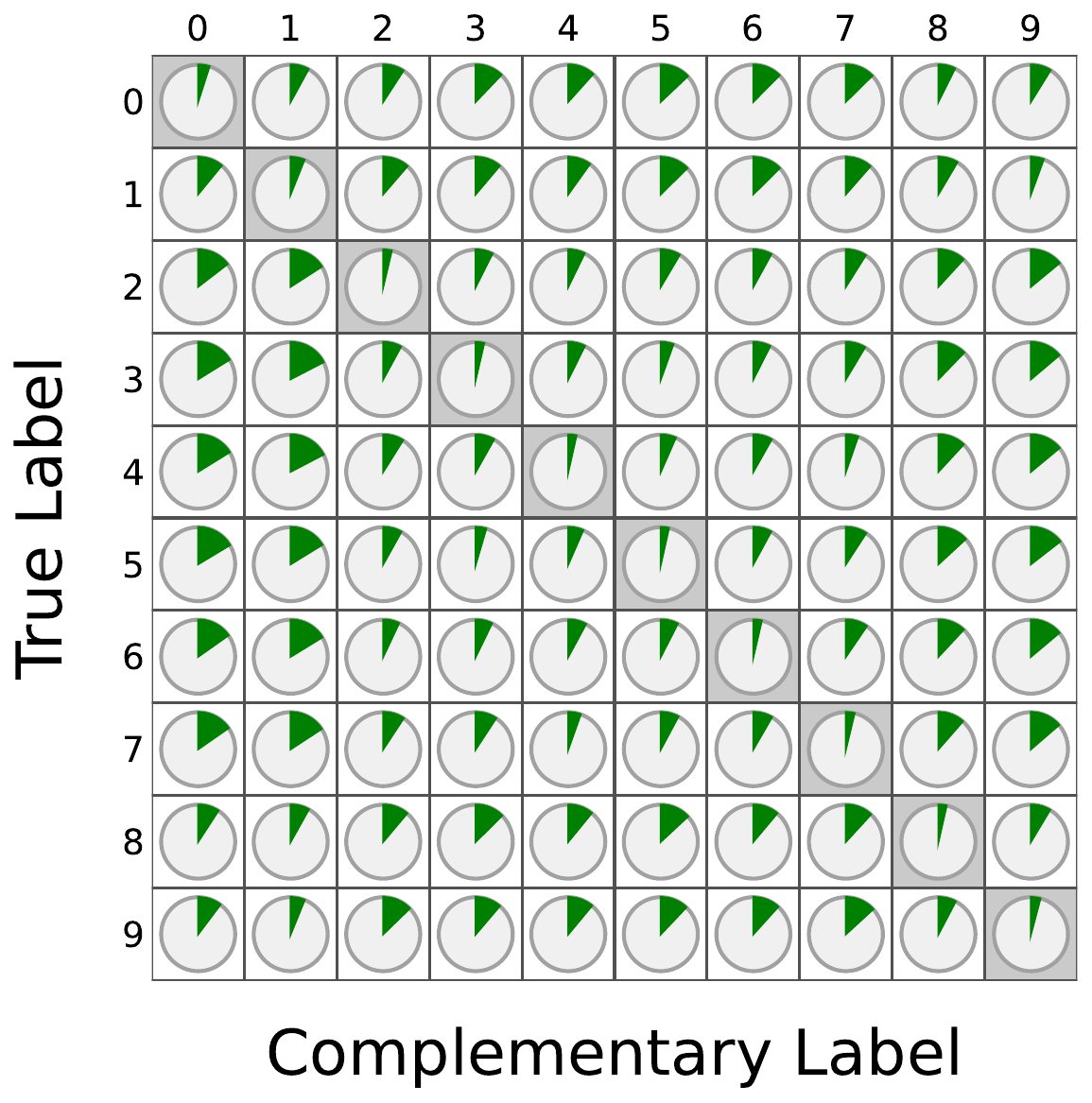}
    \caption{CLCIFAR-10}
    \label{fig:img2}
  \end{subfigure}
  \begin{subfigure}[b]{0.22\textwidth}
    \includegraphics[width=\textwidth]{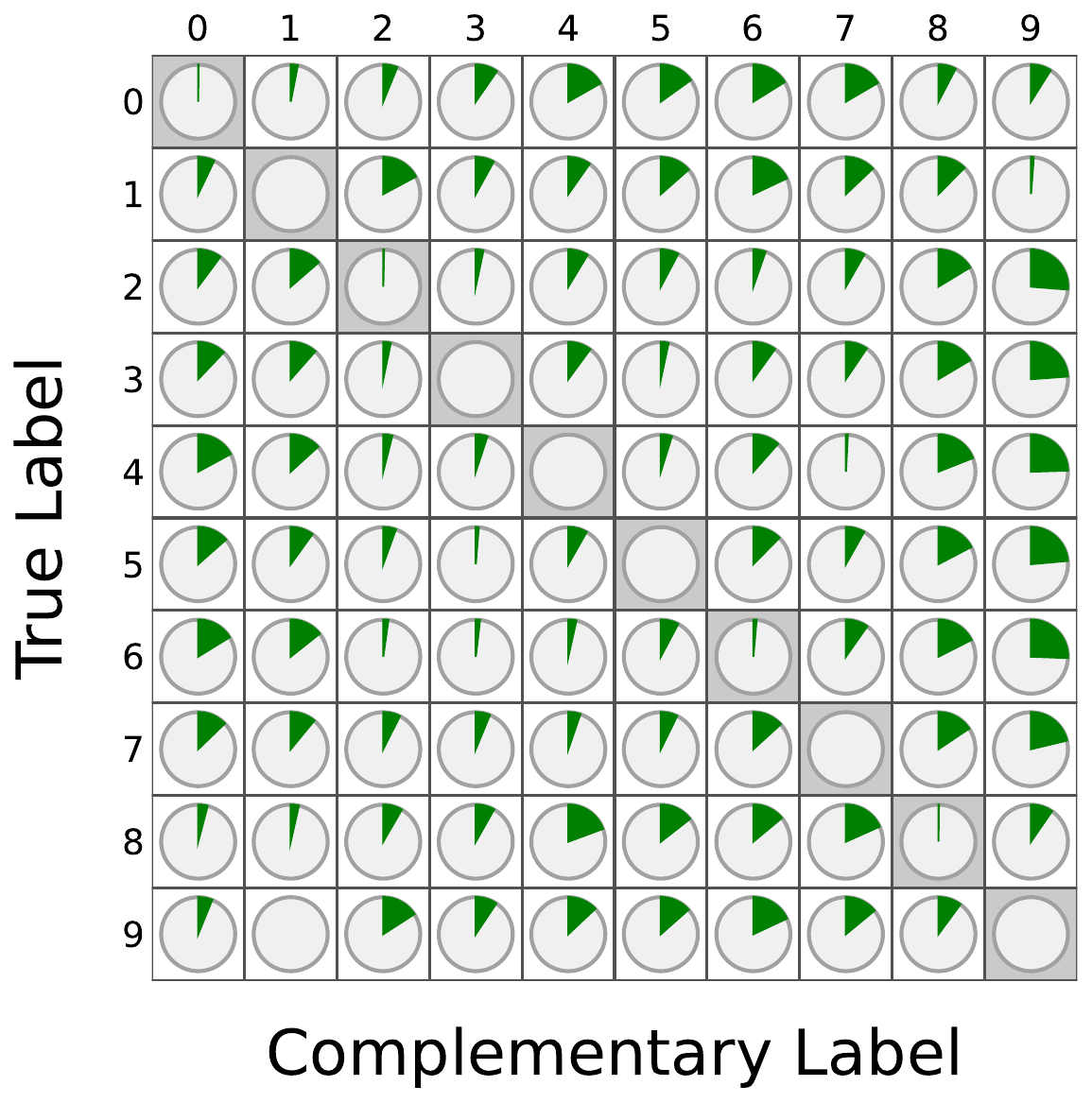}
    \caption{ACLCIFAR-10}
    \label{fig:img3}
  \end{subfigure}
  \begin{subfigure}[b]{0.22\textwidth}
    \includegraphics[width=\textwidth]{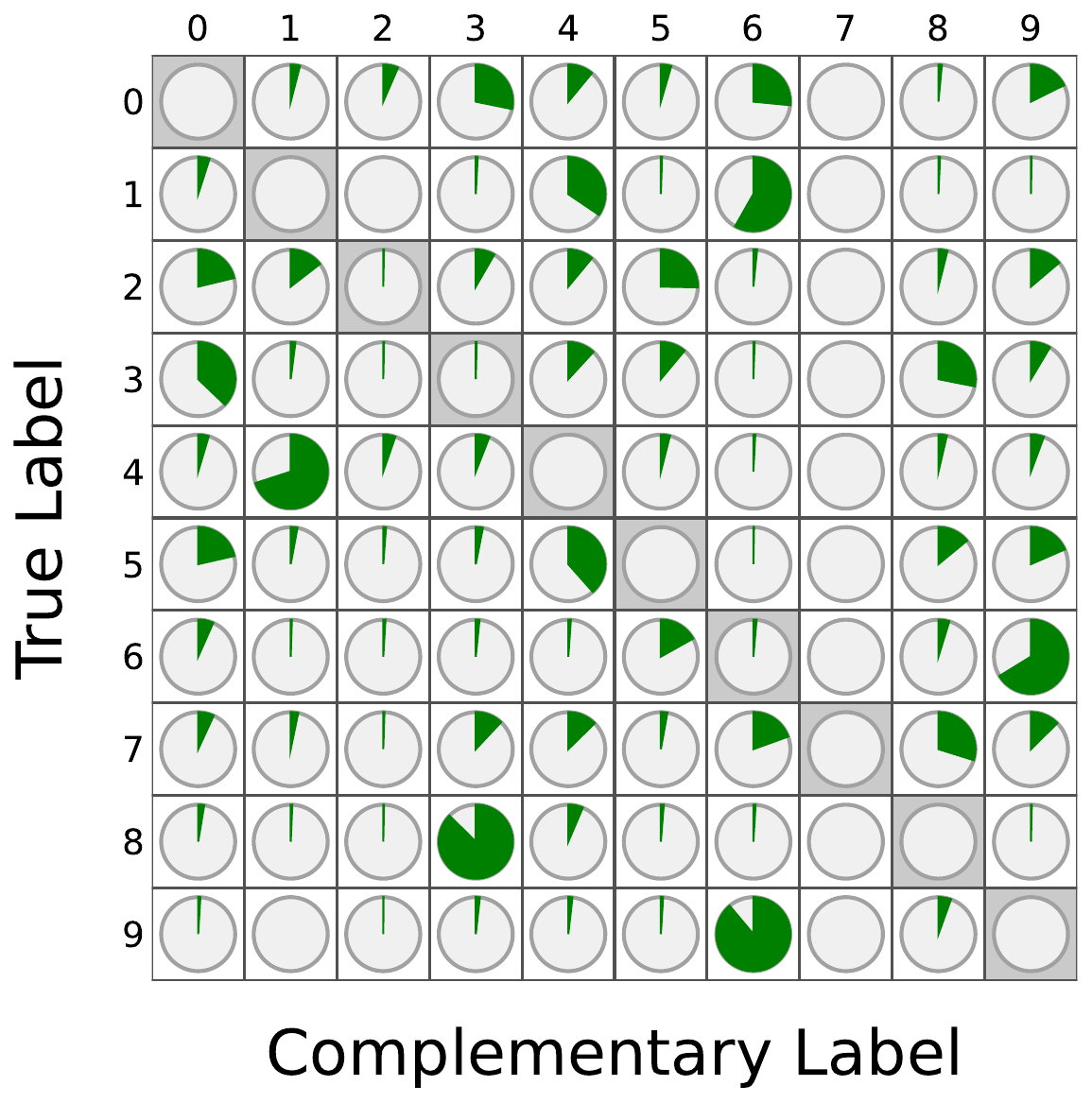}
    \caption{BCLCIFAR-10}
    \label{fig:img4}
  \end{subfigure}
  \caption{BICL transition matrix compared to other label collection approaches.}
  \label{fig:transition-matrix}
\end{figure*}

Additionally, we analyze the properties of the biased CLs collected via VLM-based annotation. Specifically, we compare the transition matrices, and label noise rates across four settings: the idealized uniform assumption~\citep{ishida2017learning}, human-annotated CLs~\cite{wang2024climage}, VLM-annotated CLs with a uniform design~\cite{aclimage2025}, and our proposed dataset. These comparisons provide insight into the structural characteristics induced by different complementary-label generation mechanisms.

\textbf{Highly biased transition matrix.}
As shown in Figure~\ref{fig:transition-matrix}, the empirical transition matrix of our biased complementary-label CIFAR-10 dataset (BCLCIFAR-10) exhibits substantially stronger bias than those observed under the idealized, human-annotated, and uniformly designed VLM-annotated settings. This pronounced bias is a direct consequence of our design, which combines two sources of asymmetry: (i) an intentionally bias-induced constrained labeling and (ii) systematic biases inherent to VLM-based annotation. Importantly, additional ablation results show that the performance gains of BICL are primarily driven by the designed constrained-labeling mechanism rather than by VLM bias alone (Appendix~\ref{source_bias}). Due to space constraints, we report label-noise rates, additional transition matrices, and label distributions for the BICL datasets in Appendix~\ref{extended_bicl_dataset}.



\section{Intuition from an Information-Theoretic Perspective}
\label{sec:intuitive_justification}


We provide an information-theoretic perspective to understand BICL.
First, we recall the complementary-label generation process $Q$.
Second, we quantify the uncertainty induced by $Q$ using conditional entropy.
Finally, we connect this quantity to the classification error.

Let $f: \mathbb{R}^d \to \mathcal{Y}$ be a classifier, and denote the prediction error by $p(y \neq f(\mathbf{x}))$.
To facilitate the discussion, we assume that the complementary-label generation process is instance-independent. Under this assumption, the transition matrix $Q_{kj} = p(\bar{Y} = j \mid Y = k)$, defined in Section~\ref{sec:complementary-classification}, serves as a tractable model of the complementary-label generation mechanism.

To quantify the uncertainty introduced by $Q$, we consider an information-theoretic characterization.
In particular,
Fano’s inequality~\cite{cover2006elements} enables us to derive a lower bound on the prediction error in terms of the ambiguity of the supervision signal. The \textit{proof} is deferred to Appendix~\ref{app:proof_thm1}.
\begin{theorem}[Lower Bound on Supervision Error]
\label{thm:fano_bound}
Let the complementary-label $\bar{Y}$ be generated by a transition mechanism $Q$. For any classifier $f$ learned from $\bar{\mathcal{D}}$, the prediction error probability $p(y \neq f(\mathbf{x}))$ is lower bounded by
$
\underline{p_{\text{err}}^{Q}} = \frac{H^{Q}(Y \mid \bar{Y}) - I(Y; X) - 1}{\log(C - 1)},
$
where $H^{Q}(Y \mid \bar{Y})$ is the conditional entropy of the true label given the complementary-label induced by $Q$.
\end{theorem}

\begin{intuition}
\label{notice:intuition}
The conditional entropy $H^{Q}(Y \mid \bar{Y})$ quantifies the uncertainty of the true label $Y$ after observing the complementary label $\bar{Y}$.
Consider two transition mechanisms $Q_1$ and $Q_2$. If $H^{Q_1}(Y \mid \bar{Y}) < H^{Q_2}(Y \mid \bar{Y})$, then  $\underline{p_{\text{err}}^{Q_1}} < \underline{p_{\text{err}}^{Q_2}}$ since $I(Y; X)$ and $\log(C - 1)$ are fixed for a given dataset.
Therefore, under the instance-independent assumption, a transition mechanism that induces a smaller conditional entropy yields a tighter (i.e., lower) bound on the prediction error.
\end{intuition}

\subsection{The Advantages of Bias}
\label{sec:bias_advantage}

The intuition above helps explain the advantage of introducing bias.
We begin by explaining why \citep{fwd2018} outperforms \citep{ishida2017learning}.
The mechanism in \citep{ishida2017learning} is denoted by $Q^{\text{Unif}}$, where the diagonal entries are zero and the off-diagonal entries are $\frac{1}{C-1}$.
Thus, each row corresponds to a uniform distribution. 
This distribution attains the maximum conditional entropy 
$H^{\text{Unif}}(Y \mid \bar{Y}) = \log(C-1)$. 
In contrast, a biased mechanism $Q^{\text{Bias}}$ induces a smaller conditional entropy, i.e.,
$H^{\text{Bias}}(Y \mid \bar{Y}) \le H^{\text{Unif}}(Y \mid \bar{Y})$. 
This reduction in uncertainty implies a tighter lower bound on the error,
$\underline{p_\text{err}^{\text{Bias}}} \le \underline{p_\text{err}^{\text{Unif}}}$.
Therefore, by incorporating semantic bias as in \citep{fwd2018}, the model retains more information about the true label, which is consistent with improved empirical performance reported in prior work.

Second, we examine the label selection strategy proposed in Section~\ref{sec:3.1}.
While \citep{fwd2018} introduces bias, it often results in a dense transition matrix, where a complementary label may still weakly correspond to many possible true classes.
In contrast, our label selection strategy mitigates this issue by randomly selecting a small candidate set (of size four) for each class. 
A natural question is whether the conditional entropy induced by our method satisfies $H^{\text{Ours}}(Y \mid \bar{Y}) \le H^{\text{Bias}}(Y \mid \bar{Y}).$ 
Given the counterexample in Appendix~\ref{app:counterexample}, such an ordering cannot be established in general.
Nevertheless, simulations in Appendix~\ref{app:simulation} suggest that, with high probability, our method empirically satisfies
$H^{\text{Ours}}(Y \mid \bar{Y}) \le H^{\text{Bias}}(Y \mid \bar{Y})$\footnote{This phenomenon is also consistently observed in experiments; see Tables~\ref{tab:entropy_1}, \ref{tab:entropy_2}.}. 
Based on this empirical evidence, our method satisfies
$\underline{p_{\text{err}}^{\text{Ours}}} \le \underline{p_{\text{err}}^{\text{Bias}}} \le \underline{p_{\text{err}}^{\text{Unif}}}$ in practice.
Therefore, by explicitly reducing the size of the candidate set, our label selection strategy produces a supervision signal that is empirically more informative than those induced by the biased and uniform frameworks~\citep{fwd2018, ishida2017learning}, in the sense of inducing lower uncertainty about the true label.

\subsubsection{Addressing Realistic Situations}
\label{sec:connect_to_real}


\begin{remark}[Instance Dependence]
\label{notice:rmk_instance_dep}
In the setting considered in this section, the transition $Q$ is assumed to be instance-independent.
In practice, however, it may be instance-dependent, i.e., $Q(X)$.
Our working hypothesis is that the dataset is predominantly composed of CLs governed by an instance-independent transition $Q$, with a smaller portion influenced by an instance-dependent $Q(X)$.
Although this assumption does not fully capture the real data-generating process, it provides a tractable approximation that enables the design of a label collection protocol that is easy to construct (Figure~\ref{fig:1} and Section~\ref{sec:3.1}), intuitively explainable (Section~\ref{sec:intuitive_justification}), and achieves scalable and state-of-the-art performance (Tables~\ref{tab:2} and \ref{tab:3}).
\end{remark}

\begin{remark}[Role of Conditional Entropy]
\label{notice:rmk_confounder}
In the setting considered in this section, the conditional entropy $H^{Q}(Y \mid \bar{Y})$ is related to the classification error.
In practice, however, the protocol pipeline consists of a label selection step followed by an annotation step, involving multiple factors that jointly contribute to the final performance.
It is therefore natural to ask whether $H^{Q}(Y \mid \bar{Y})$ remains a dominant explanatory factor in this more realistic setting. To examine this, we conduct auxiliary experiments comparing the conditional entropy of different complementary-label datasets, as reported in Tables~\ref{tab:entropy_1} and~\ref{tab:entropy_2} in Appendix~\ref{sec:entroy_mutualinfor}. Across these datasets, we observe that higher accuracies under our proposed protocol consistently coincide
with lower values of $H^{Q}(Y \mid \bar{Y})$.
\end{remark}

\section{Experiments}
\label{sec:6}
In this section, we describe the experimental setup used to evaluate the proposed method. We then present empirical results and provide a detailed analysis by comparing our approach with representative baselines and competing methods. These experiments are designed to assess the effectiveness, robustness, and scalability of our method across different settings.

\subsection{Experimental Setups}
\label{sec:6.1}

\textbf{Datasets.} We evaluated the effectiveness of our proposed biased CLL approach across four synthesis labeled datasets with varying number of class. These include CIFAR-10, CIFAR-20, CIFAR-100~\cite{cifar10}, TinyImageNet-200~\cite{tinyImageNet} which contain 10, 20, 100 and 200 classes, respectively. Each CIFAR-based dataset consist of 50,000 training samples and 10,000 testing samples. TinyImageNet-200 has 100,000 training images and 10,000 testing images. 

\textbf{Implementation details.} All experiments were conducted using ResNet34~\cite{He2015} backbone network. We applied standard data augmentation techniques, namely \textit{RandomHorizontalFlip}, \textit{RandomCrop}, and normalization, to all images.
Model training was conducted using the Adam optimizer with $10^{-3}$ as selected learning rate, and weight decay $10^{-5}$. All experiments used a batch size of 512 over 300 epochs on NVIDIA A6000 48GB GPU. 
The experiments were repeated three times with different random seeds to ensure robustness.

\textbf{Baseline methods.}
We evaluate \emph{Q-aware} CLL algorithms, which explicitly use the transition matrix $Q$ to define the training loss. During training, these methods are provided with additional information in the form of an estimated transition matrix. Specifically, we consider six representative baselines: FWD~\cite{fwd2018}, CPE-F/I/T~\cite{cpe2023}, and URE-TNN/TGA~\cite{ishida2019complementarylabel}.
For all \textit{Q-aware} baselines, we estimate the transition matrix using a small set of ground-truth labels, namely five labeled samples per class. Additional details of this estimation procedure are provided in Appendix~\ref{compared_few_true_labels}. Due to space constraints, related-work and the complete experimental setup are provided in Appendices~\ref{background},~\ref{details_setup}, respectively.



\subsection{Results and Analysis}

As shown in Tables~\ref{tab:2} and~\ref{tab:3}, our proposed BICL annotation consistently yields substantial performance improvements across a wide range of CLL algorithms under four annotation settings: the idealized uniform assumption, human-annotated CLCIFAR, uniformly designed VLM-annotated ACLCIFAR, and our BCLCIFAR constructed via BICL.
On CIFAR-10 and CIFAR-20, Table~\ref{tab:2}, BICL achieves the best performance for all FWD and CPE-based methods, with large absolute gains over the second best baselines, approximately 16--21 percentage points on CIFAR-10 and 22--26 percentage points on CIFAR-20. These improvements are more pronounced on CIFAR-20, indicating that BICL is particularly effective in more challenging settings with finer-grained class partitions. This trend aligns with our earlier analysis showing that controlled transition bias reduces label ambiguity and increases the informativeness of CLs.

\begin{figure*}[htb]
  \centering
  \includegraphics[width=0.55\textwidth]{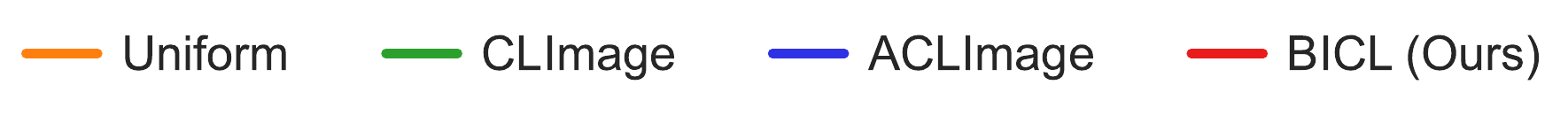} \\
  \vspace{0.05cm} 
  \begin{subfigure}[b]{0.22\textwidth}
    \includegraphics[width=\textwidth]{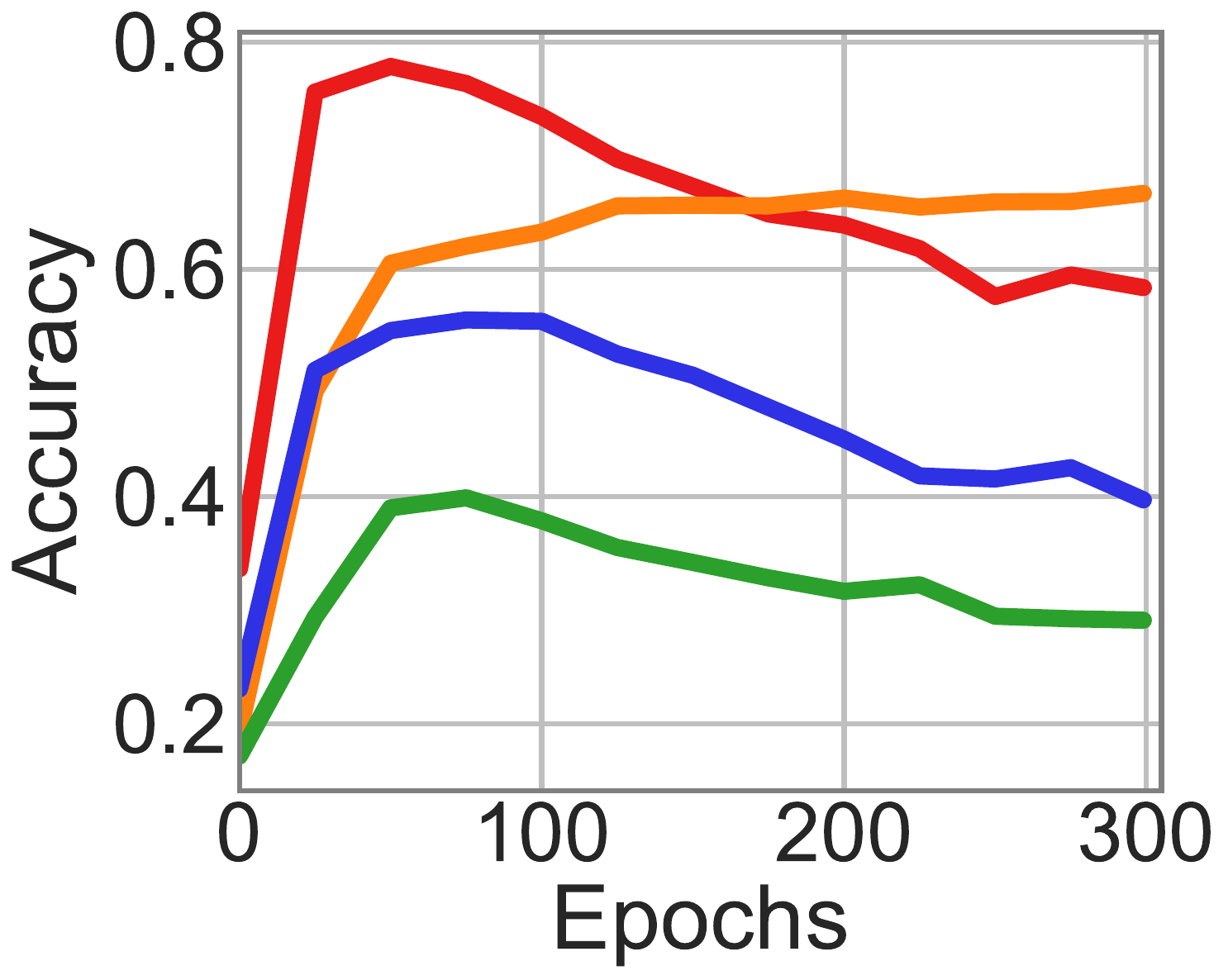}
    \caption{CIFAR-10}
    \label{fig:bcl1}
  \end{subfigure}
  \begin{subfigure}[b]{0.22\textwidth}
    \includegraphics[width=\textwidth]{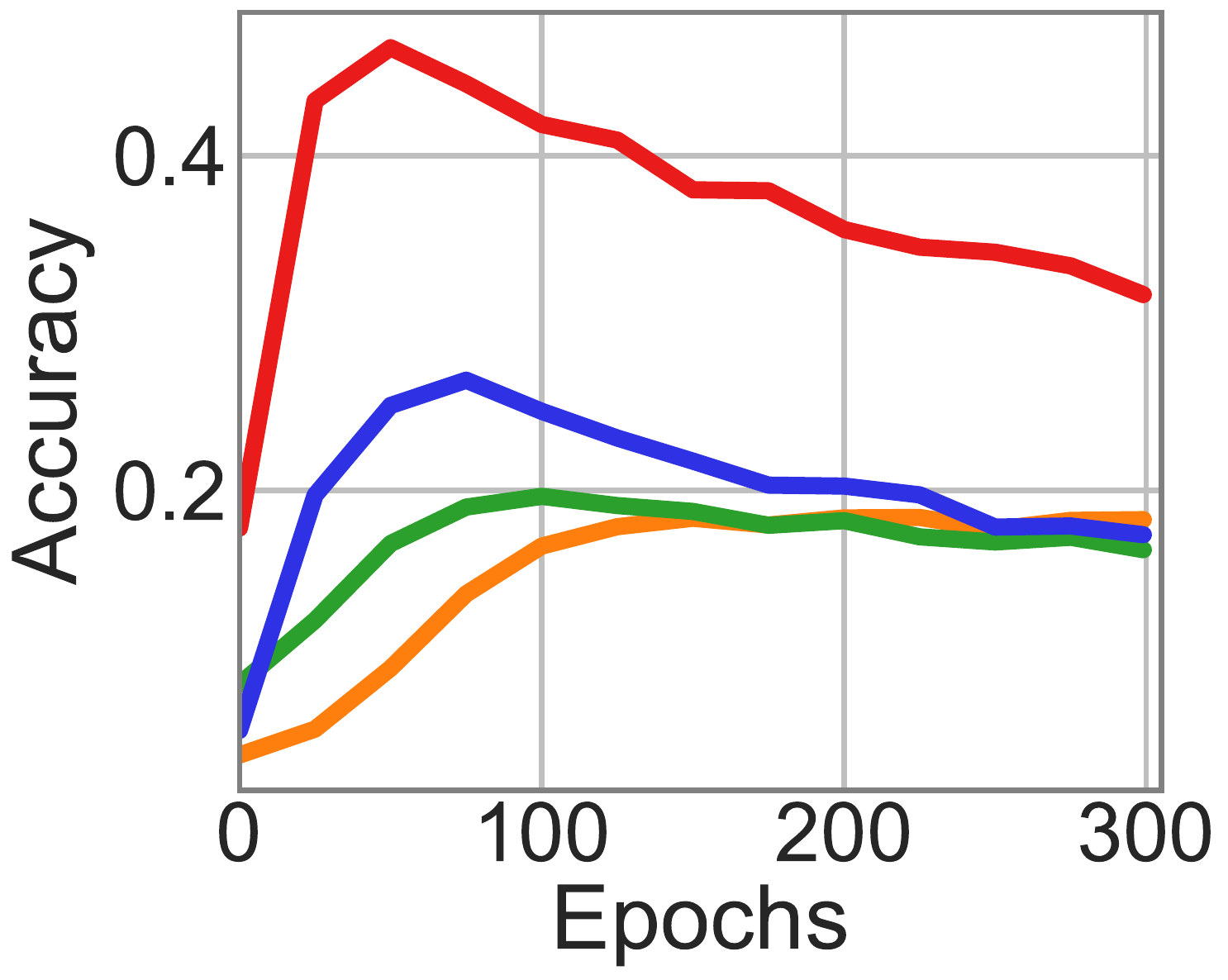}
    \caption{CIFAR-20}
    \label{fig:bcl2}
  \end{subfigure}
  \begin{subfigure}[b]{0.22\textwidth}
    \includegraphics[width=\textwidth]{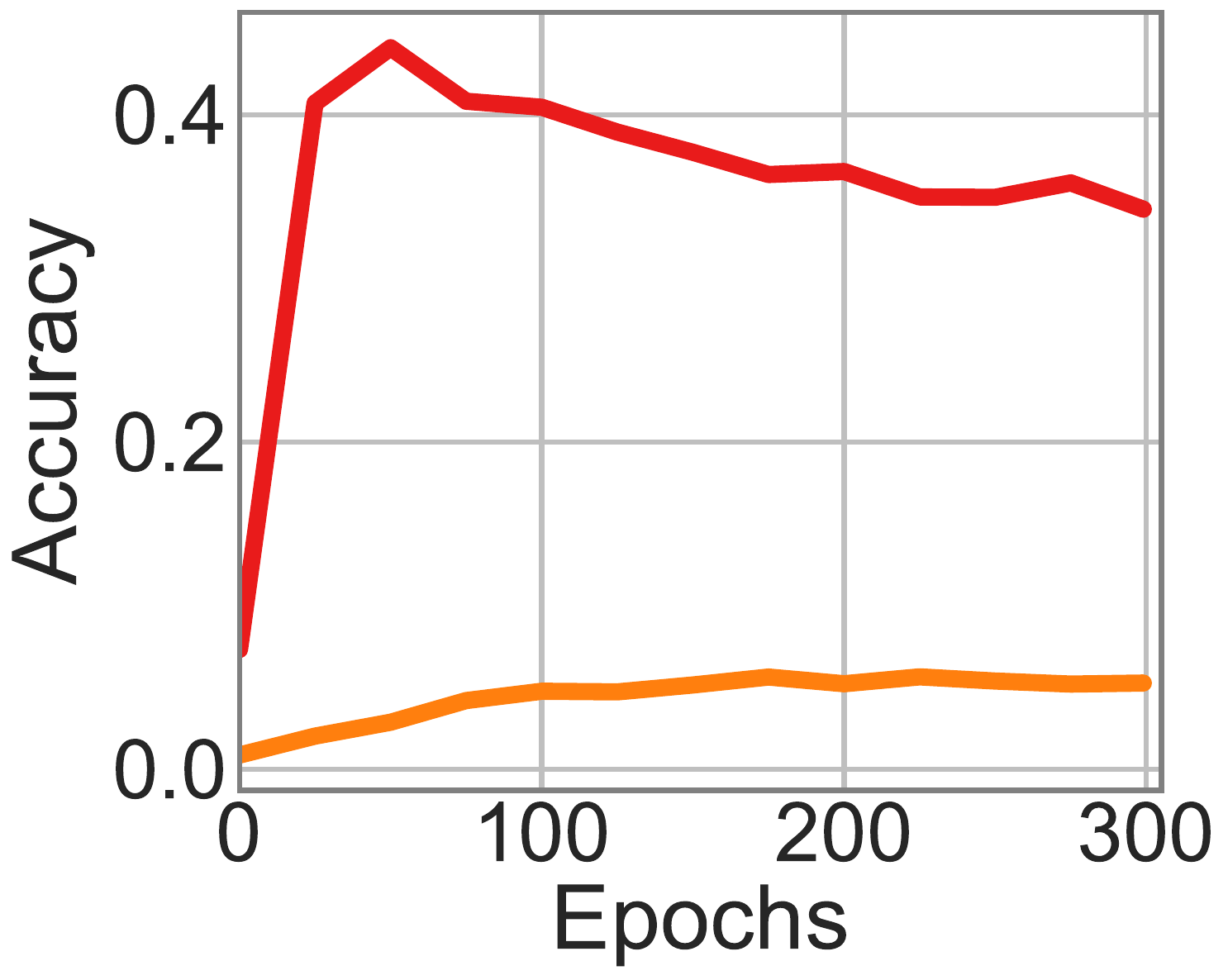}
    \caption{CIFAR-100}
    \label{fig:bcl3}
  \end{subfigure}
  \begin{subfigure}[b]{0.22\textwidth}
    \includegraphics[width=\textwidth]{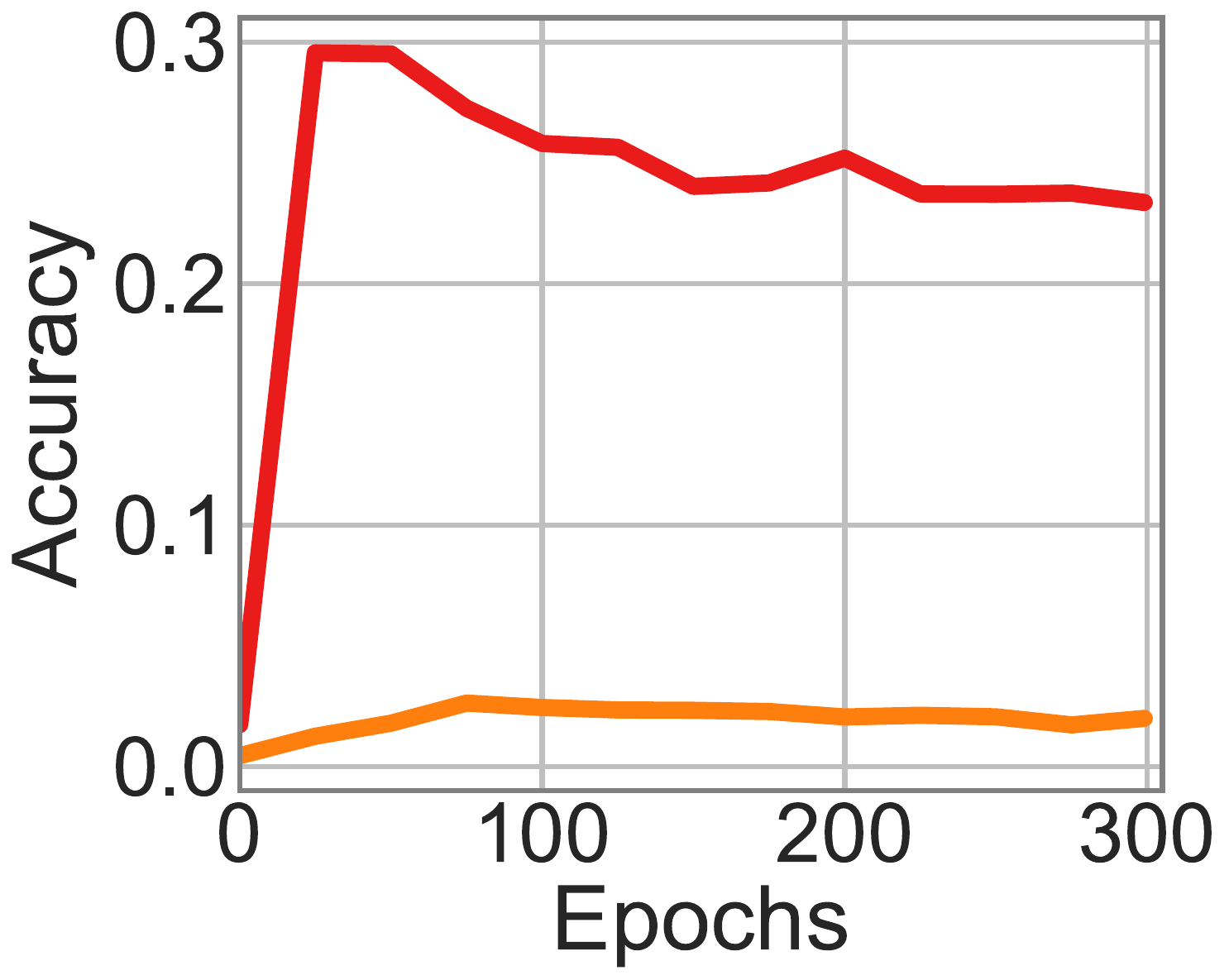}
    \caption{TinyImageNet-200}
    \label{fig:bcl4}
  \end{subfigure}
\caption{Test accuracy during training on four datasets. Our method reaches its peak performance faster than Uniform, CLImage, and ACLImage when applicable. CLImage and ACLImage are limited to at most 20 classes, so panels (c) and (d) omit them.}
  \label{fig:learning-curve}
\end{figure*}


URE-based methods exhibit mixed responses to BICL: URE-TNN benefits in some settings, whereas URE-TGA generally prefers the uniform assumption, reflecting the strong reliance of URE-based approaches on uniform complementary-label models. This result underscores that the effectiveness of biased CLs is algorithm-dependent and that BICL is best suited to methods that can leverage non-uniform transition structures.

Importantly, Table~\ref{tab:3} reveals that BICL enables CLL to scale to large label spaces that were previously infeasible, such as CIFAR-100 and TinyImageNet-200. For FWD and CPE-based methods, BICL boosts accuracy from near-random performance under the uniform assumption to over 44--47 percentage points on CIFAR-100 and above 30 percentage points on TinyImageNet-200. These results underscore the critical role of structured bias in alleviating severe ambiguity in large-class CLL.

\begin{table*}[t]
\centering
\caption{Performance comparison of various algorithms FWD, CPE-F/I/T, URE-TNN/TGA across different datasets and annotation settings on CIFAR-10 and CIFAR-20. $\Delta$ is the difference between the performance of the second best and our \textbf{BICL}.}
\label{tab:2}
\resizebox{1.0\textwidth}{!}{
\small
\setlength{\tabcolsep}{1.5pt}
\begin{tabular}{l|ccccc|ccccc}
\toprule
\textbf{Datasets} 
& \multicolumn{5}{c|}{\textbf{CIFAR-10}} 
& \multicolumn{5}{c}{\textbf{CIFAR-20}} \\
\midrule
\textbf{Algorithm}
& Uniform & CLImage & ACLImage & \textbf{BICL} & $\Delta$
& Uniform & CLImage & ACLImage & \textbf{BICL} & $\Delta$ \\
\midrule
FWD       
& 64.99\tiny{$\pm$0.64}& 39.25\tiny{$\pm$1.55} & 56.58\tiny{$\pm$1.18}& \textbf{81.23}\tiny{$\pm$0.05} & \textcolor{blue}{$\uparrow$ 16.24}
& 20.50\tiny{$\pm$0.75}& 19.82\tiny{$\pm$0.38} & 25.35\tiny{$\pm$0.35} & \textbf{50.84}\tiny{$\pm$0.32} & \textcolor{blue}{$\uparrow$ 25.49}\\
CPE-F  
& 64.97\tiny{$\pm$0.61}& 38.94\tiny{$\pm$1.39} & 56.43\tiny{$\pm$0.84} & \textbf{80.98}\tiny{$\pm$0.18} & \textcolor{blue}{$\uparrow$ 16.01}
& 20.73\tiny{$\pm$1.17} & 19.48\tiny{$\pm$0.35} & 25.02\tiny{$\pm$0.68}& \textbf{50.89}\tiny{$\pm$0.30} & \textcolor{blue}{$\uparrow$ 25.87}\\
CPE-I  
& 54.44\tiny{$\pm$1.95} & 33.94\tiny{$\pm$1.40} & 52.79\tiny{$\pm$1.34} & \textbf{75.64}\tiny{$\pm$0.59} & \textcolor{blue}{$\uparrow$ 21.20}
& 12.01\tiny{$\pm$0.93} & 16.79\tiny{$\pm$0.48} & 21.58\tiny{$\pm$0.97}& \textbf{46.42}\tiny{$\pm$0.54} & \textcolor{blue}{$\uparrow$ 24.84}\\
CPE-T  
& 57.94\tiny{$\pm$0.97} & 38.80\tiny{$\pm$1.16} & 52.84\tiny{$\pm$1.05} & \textbf{77.47}\tiny{$\pm$1.28} & \textcolor{blue}{$\uparrow$ 19.53} & 20.58\tiny{$\pm$0.78} & 19.33\tiny{$\pm$0.75} & 23.39\tiny{$\pm$0.68}& \textbf{45.31}\tiny{$\pm$0.66}& \textcolor{blue}{$\uparrow$ 21.92}\\
URE-TNN      
& \textbf{47.25}\tiny{$\pm$2.35} & 31.21\tiny{$\pm$1.11} & 30.70\tiny{$\pm$4.56} & 42.84\tiny{$\pm$2.27} & \textcolor{red}{$\downarrow$ 4.41}
& 15.29\tiny{$\pm$1.11} & 9.47\tiny{$\pm$2.67} & 8.63\tiny{$\pm$1.70} & \textbf{17.33}\tiny{$\pm$1.68} & \textcolor{blue}{$\uparrow$ 2.02}\\
URE-TGA       
& \textbf{57.37}\tiny{$\pm$1.12} & 33.68\tiny{$\pm$1.06} & 42.28\tiny{$\pm$5.19} & 45.65\tiny{$\pm$6.72} & \textcolor{red}{$\downarrow$ 11.72}
& \textbf{13.61}\tiny{$\pm$4.51} & 5.17\tiny{$\pm$0.25} & 5.03\tiny{$\pm$0.10} & 5.05\tiny{$\pm$0.00} & \textcolor{red}{$\downarrow$ 8.56}\\
\bottomrule
\toprule
Std-Supp & \multicolumn{5}{c|}{85.09\tiny{$\pm$0.39}} & \multicolumn{5}{c}{33.19\tiny{$\pm$0.21}} \\
\bottomrule
\end{tabular}
}
\begin{tablenotes}
\item \textbf{\textit{Std-Supp (Standard Supervision)}}: Instead of collecting CLs, we collect the true labels using the same VLM. We then train the model with the cross-entropy loss and report the results to compare standard supervision with CLL.
\end{tablenotes}
\end{table*}

\begin{table*}[htb]
\centering
\caption{Performance comparison of various algorithms FWD, CPE-F/I/T, URE-TNN/TGA across different datasets and annotation settings on CIFAR-100 and TinyImageNet-200. $\Delta$ is the difference between the performance of the second best and our \textbf{BICL}.}
\label{tab:3}
\resizebox{1.00\textwidth}{!}{
\small
\setlength{\tabcolsep}{1.5pt}
\begin{tabular}{l|ccccc|ccccc}
\toprule
\textbf{Datasets} 
& \multicolumn{5}{c|}{\textbf{CIFAR-100}} 
& \multicolumn{5}{c}{\textbf{TinyImageNet-200}} \\
\midrule
\textbf{Algorithm}
& Uniform & CLImage & ACLImage & \textbf{BICL} & $\Delta$
& Uniform & CLImage & ACLImage & \textbf{BICL} & $\Delta$ \\
\midrule
FWD       
& 5.53\tiny{$\pm$0.47} & N/A & N/A & \textbf{46.70}\tiny{$\pm$0.60} & \textcolor{blue}{$\uparrow$ 41.17}
& 4.00\tiny{$\pm$0.43} & N/A & N/A & \textbf{32.15}\tiny{$\pm$0.30} & \textcolor{blue}{$\uparrow$ 28.15}\\
CPE-F  
& 5.06\tiny{$\pm$0.76} & N/A & N/A & \textbf{46.57}\tiny{$\pm$0.09} & \textcolor{blue}{$\uparrow$ 41.51}
& 2.14\tiny{$\pm$0.87} & N/A & N/A & \textbf{31.89}\tiny{$\pm$0.07} & \textcolor{blue}{$\uparrow$ 29.75} \\
CPE-I  
& 5.88\tiny{$\pm$1.27} & N/A & N/A & \textbf{34.03}\tiny{$\pm$0.15} & \textcolor{blue}{$\uparrow$ 28.15}
& 1.67\tiny{$\pm$0.20} & N/A & N/A & \textbf{21.70}\tiny{$\pm$0.40} & \textcolor{blue}{$\uparrow$ 20.03} \\
CPE-T  
& 3.91\tiny{$\pm$1.09} & N/A & N/A & \textbf{44.35}\tiny{$\pm$0.33} & \textcolor{blue}{$\uparrow$ 40.44}
& 1.19\tiny{$\pm$0.06} & N/A & N/A & \textbf{28.51}\tiny{$\pm$1.40} & \textcolor{blue}{$\uparrow$ 27.32} \\
URE-TNN      
& 1.39\tiny{$\pm$0.05} & N/A & N/A & \textbf{9.58}\tiny{$\pm$1.06} & \textcolor{blue}{$\uparrow$ 8.19}
& 0.60\tiny{$\pm$0.05} & N/A & N/A & \textbf{5.84}\tiny{$\pm$0.49} & \textcolor{blue}{$\uparrow$ 5.24}\\
URE-TGA       
& 1.03\tiny{$\pm$0.03} & N/A & N/A & \textbf{1.05}\tiny{$\pm$0.01} & \textcolor{blue}{$\uparrow$ 0.02 }
& \textbf{0.54}\tiny{$\pm$0.03} & N/A & N/A & 0.49\tiny{$\pm$0.00} & \textcolor{red}{$\downarrow$ 0.05}\\
\bottomrule
\toprule
Std-Supp & \multicolumn{5}{c|}{36.49\tiny{$\pm$0.49}} & \multicolumn{5}{c}{16.54\tiny{$\pm$0.28}} \\
\bottomrule
\end{tabular}
}
\begin{tablenotes}
    \item \textbf{\textit{N/A}}: CIFAR-100 and TinyImageNet-200 are not available for CLImage and ACLImage. Most prior work on CLL has focused on datasets with at most 20 classes. Consequently, complementary-label versions of large-scale benchmarks have not been collected.
\end{tablenotes}
\end{table*}

Additionally, Table~\ref{tab:2} shows a clear discrepancy between the commonly used \emph{Uniform} assumption~\cite{ishida2017learning} and CLs collected in practice. Historically, the uniform transition has been treated as a reasonable default, and often as an implicit \emph{upper bound}, because it assumes noise-free CLs. Motivated by this assumption, recent pipelines such as CLImage~\cite{wang2024climage} and ACLImage~\cite{aclimage2025} aim to collect CLs that approximate a uniform distribution on real-world data. However, due to unavoidable annotation noise and systematic biases introduced by the collection process, the resulting ``uniform-like'' labels lead to substantially lower accuracy than the idealized \emph{Uniform} setting across most algorithms in Table~\ref{tab:2}.

In contrast, BICL consistently outperforms CLImage and ACLImage and achieves substantial gains over the idealized \emph{Uniform} setting. Notably, BICL can even exceed the performance obtained under the \emph{Uniform} assumption, indicating that \emph{carefully controlled bias} can provide more beneficial than uniformly sampled CLs, as shown in Tables~\ref{tab:2}, and~\ref{tab:3}.
Moreover, Figure~\ref{fig:learning-curve} shows that BICL reaches its peak performance faster than the Uniform baseline, reaching strong performance in fewer training epochs.
These findings challenge the prevailing assumption that uniform CLs are optimal in practice and instead highlight biased, which is more promising label collection as a more effective strategy for scalable CLL.
Additional discussions, experiments are provided in Appendices~\ref{further_discussion},~\ref{additional_ablation}, and~\ref{details_setup}.

\section{Conclusion}
\label{Conclusion}

In this paper, we introduced bias-induced constrained labeling (BICL), a principled framework for CLL that introduces and controls beneficial bias in transition matrix. Guided by an information-theoretic lower-bound analysis, BICL shapes the induced label distributions via explicit constraints, improving learning stability and enabling CLL to scale to larger label spaces. Empirically, BICL achieves 7× (6 $\rightarrow$ 46 percentage points) and 8× (4 $\rightarrow$ 32 percentage points) significant improvements over the uniform setting on CIFAR-100 and TinyImageNet-200, respectively. These results demonstrate that BICL alleviates \emph{a key scalability bottleneck} of prior CLL approaches, which have largely been evaluated on problems with at most 20 classes. 
Overall, BICL advances the practicality of CLL for large-scale classification with complementary supervision.

\paragraph{Broader Impact.} It is known that weakly-supervised learning can be used for some privacy-preserving applications. While our works do not fall under such privacy-preserving applications, we suggest that practitioners exercise caution when exploring these use cases.



\medskip
{
\small
\bibliographystyle{unsrt} 
\bibliography{example_paper}
}

\clearpage
\appendix

\section{Proofs of Theoretical Results}

\subsection{Proof of Theorem~\ref{thm:fano_bound}}
\label{app:proof_thm1}

\begin{proof}
We start with the standard Fano's Inequality~\cite{cover2006elements}, which bounds the conditional entropy $H(Y \mid X)$ with respect to the prediction error probability $p_{\text{err}}(f) = p(Y \neq f(\mathbf{x}))$:
\begin{equation}
    H(Y \mid X) \leq H_b(p_{\text{err}}(f)) + p_{\text{err}}(f) \log(C - 1)
    \nonumber
\end{equation}
where $H_b(P_e)$ is the binary entropy function. Bounding $H_b(p_{\text{err}}(f)) \leq 1$, we obtain:
\begin{equation}
    p_{\text{err}}(f) \geq \frac{H(Y \mid X) - 1}{\log(C - 1)}
    \label{eq:pe_basic_app}
\end{equation}

By the definition of mutual information, we express the conditional entropy as:
\begin{equation}
    H(Y \mid X) = H(Y) - I(Y; X)
    \nonumber
\end{equation}

To relate this to the complementary label $\bar{Y}$, we invoke Theorem 2.6.5 from Cover and Thomas~\cite{cover2006elements}, which states that conditioning reduces entropy (information can not hurt). Specifically, $H(Y) \geq H^{\mathcal{A}}(Y \mid \bar{Y})$. Applying this inequality to our context:
\begin{equation}
    H(Y \mid X) = H(Y) - I(Y; X) \geq H^{\mathcal{A}}(Y \mid \bar{Y}) - I(Y; X)
    \nonumber
\end{equation}
This step is crucial as it incorporates the supervision quality of the CLs into the lower bound.

Finally, substituting this result back into Eq.~\eqref{eq:pe_basic_app}, we arrive at the theorem statement:
\begin{equation}
     \underline{p^{\mathcal{A}}_{\text{err}}} = \frac{H^{\mathcal{A}}(Y \mid \bar{Y}) - I(Y; X) - 1}{\log(C - 1)}.
     \nonumber
\end{equation}
\end{proof}

\subsection{Detailed Comparison of Supervision Ambiguity}
\label{app:bias_comparison}

In this section, we compare the supervision ambiguity between the standard uniform assumption and the biased scenario. The quantity of interest is the conditional entropy $H(Y \mid \bar{Y})$, which governs the tightness of the error lower bound derived in Theorem~\ref{thm:fano_bound}.

\paragraph{Uniform Assumption \citep{ishida2017learning}.}
Standard complementary label learning (CLL) approaches~\citep{ishida2017learning} assume that the complementary label $\bar{y}$ is sampled uniformly from the set of incorrect classes $\mathcal{Y} \setminus \{y\}$. Under this assumption, the posterior probability is uniform over the support of size $C-1$.

As established in \citep{cover2006elements}, the uniform distribution is the \textbf{maximum entropy distribution} over a fixed finite support. Therefore, the supervision ambiguity reaches its theoretical maximum:
\begin{equation}
    H^{\text{Unif}}(Y \mid \bar{Y}) = -\sum_{y \neq \bar{y}} \frac{1}{C-1} \log\left(\frac{1}{C-1}\right) = \log(C-1).
    \nonumber
\end{equation}

\paragraph{Biased Scenario \citep{fwd2018}.}
In \citep{fwd2018}, the complementary label is selected based on semantic similarity, resulting in a non-uniform transition probability. Consequently, the posterior distribution $P(Y \mid \bar{Y} = \bar{y})$ becomes non-uniform over the support $\mathcal{Y} \setminus \{\bar{y}\}$.

Since the entropy is strictly maximized only by the uniform distribution, any deviation from uniformity on the same finite support necessarily results in a reduction of entropy. Thus, strictly:
\begin{equation}
    H^{\text{Bias}}(Y \mid \bar{Y}) \leq \log(C-1) = H^{\text{Unif}}(Y \mid \bar{Y}).
    \nonumber
\end{equation}

\paragraph{Conclusion.}
Substituting these terms back into the lower bound $\underline{p^{\mathcal{A}}_{\text{err}}}$, we obtain the inequality:
\begin{equation}
    \underline{p^{\text{Bias}}_{\text{err}}} \leq \underline{p^{\text{Unif}}_{\text{err}}}.
    \nonumber
\end{equation}
This proves that incorporating semantic bias theoretically lowers the fundamental limit of the prediction error compared to the uniform assumption.

\subsection{A Counterexample}
\label{app:counterexample}

We use the following $10 \times 10$ biased transition matrix to make a counterexample.
\[
Q^{\text{Bias}} =
\begin{pmatrix}
0 & .02 & .02 & .02 & .02 & .90 & .005 & .005 & .005 & .005 \\
.02 & 0 & .02 & .02 & .02 & .005 & .90 & .005 & .005 & .005 \\
.02 & .02 & 0 & .02 & .02 & .005 & .005 & .90 & .005 & .005 \\
.02 & .02 & .02 & 0 & .02 & .005 & .005 & .005 & .90 & .005 \\
.02 & .02 & .02 & .02 & 0 & .005 & .005 & .005 & .005 & .90 \\
.02 & .02 & .02 & .02 & .90 & 0 & .005 & .005 & .005 & .005 \\
.02 & .02 & .02 & .02 & .005 & .90 & 0 & .005 & .005 & .005 \\
.02 & .02 & .02 & .02 & .005 & .005 & .90 & 0 & .005 & .005 \\
.02 & .02 & .02 & .02 & .005 & .005 & .005 & .90 & 0 & .005 \\
.02 & .02 & .02 & .02 & .005 & .005 & .005 & .005 & .90 & 0
\end{pmatrix}.
\]
From $Q^{\text{Bias}}$, we take each row, keep 4 classes, zero out the rest of the row entries, and re-normalize the non-zero entries as follows.
\[
Q^{\text{Ours}} =
\begin{pmatrix}
0 & .25 & .25 & .25 & .25 & 0 & 0 & 0 & 0 & 0 \\
.25 & 0 & .25 & .25 & .25 & 0 & 0 & 0 & 0 & 0 \\
.25 & .25 & 0 & .25 & .25 & 0 & 0 & 0 & 0 & 0 \\
.25 & .25 & .25 & 0 & .25 & 0 & 0 & 0 & 0 & 0 \\
.25 & .25 & .25 & .25 & 0 & 0 & 0 & 0 & 0 & 0 \\
.25 & .25 & .25 & .25 & 0 & 0 & 0 & 0 & 0 & 0 \\
.25 & .25 & .25 & .25 & 0 & 0 & 0 & 0 & 0 & 0 \\
.25 & .25 & .25 & .25 & 0 & 0 & 0 & 0 & 0 & 0 \\
.25 & .25 & .25 & .25 & 0 & 0 & 0 & 0 & 0 & 0 \\
.25 & .25 & .25 & .25 & 0 & 0 & 0 & 0 & 0 & 0
\end{pmatrix}.
\]

The counterexample showed that:
\begin{equation}
    H^{\text{Ours}}(Y \mid \bar{Y}) = 3.0529 \text{ bits} 
    > 
    H^{\text{Bias}}(Y \mid \bar{Y}) = 1.1975 \text{ bits}
    \nonumber
\end{equation}

\subsection{Simulation Details for Entropy Comparison}
\label{app:simulation}

To further validate the advantage of our proposed \textbf{sparsity} constraint over a general dense biased scenario (as discussed in Section~\ref{sec:bias_advantage}), we conducted a Monte Carlo simulation.

\paragraph{Setup.}
We performed $N=10,000$ independent trials across varying problem scales, specifically testing class sizes of $C \in \{10, 100, 200\}$. In each trial:
\begin{enumerate}
    \item \textbf{Dense Bias:} We generated a random transition matrix $Q^{\text{Bias}} \in \mathbb{R}^{C \times C}$ where $Q^{\text{Bias}}_{ij} \sim U[0,1]$ for $i \neq j$ and $Q^{\text{Bias}}_{ii}=0$. Rows were normalized to sum to 1.
    \item \textbf{Sparse Bias (Ours):} From $Q^{\text{Bias}}$, we derived a sparse matrix $Q^{\text{Ours}}$ by retaining $k$ ($k=4$) randomly selected elements per row and re-normalizing.
\end{enumerate}

\paragraph{Results.}
We computed the conditional entropy $H(Y \mid \bar{Y})$ for both matrices. The simulation revealed that:
\begin{equation}
    H^{\text{Ours}}(Y \mid \bar{Y}) \leq H^{\text{Bias}}(Y \mid \bar{Y})
    \nonumber
\end{equation}
holds in \textbf{100 percentage point} of the trials across all tested dimensions ($10 \times 10$, $100 \times 100$, and $200 \times 200$). This confirms that explicitly enforcing sparsity concentrates the probability mass, thereby minimizing supervision ambiguity regardless of the number of classes.

\section{Related Work}
\label{background}

\textbf{CL and MCL.}
The seminal work of \citep{ishida2017learning} introduced CLL under a uniform complementary-label assumption.
Later, \citep{ishida2019complementarylabel} proposed a general unbiased risk estimator along with non-negative and gradient-ascent variants to mitigate negative empirical risk and overfitting.
\citep{scl2020} proposed the surrogate complementary loss (SCL) framework that aligns empirical gradients with the true gradients to mitigate overfitting.
Besides complementary label (CL), multiple complementary-label learning (MCL) allows collecting several CLs per instance. 
\citep{mcl2020} formalized MCL and proposed an unbiased risk estimator to minimize all MCLs of each instance altogether.
A recent work connects CLL to other weakly supervised paradigms to improve representations \citep{ComCo2023}.

\textbf{Biased and noisy CL.}
\citep{fwd2018} modeled biased complementary-label generation with a transition matrix and introduced forward correction (FWD).
Subsequent studies targeted on the estimation of the transition matrix.
The Selected-Completely-at-Random (SCARCE) \cite{wang2024learning} removed the need for a uniform complementary-label distribution and avoided requiring an auxiliary ordinary-labeled set for transition-matrix estimation.
\citep{cpe2023} quantified the effect of an inaccurate transition matrix.
The complementary one-versus-rest (COVR) loss further addresses biased distributions by increasing the complementary logit margin to improve robustness \citep{covr_2024}.
\citep{ishiguro2022learning} developed a noise-robust CLL framework by modeling the noise via a noise transition matrix and analyzing consistency under matrix misspecification.

\textbf{Benchmarks, toolkits, and class-imbalance.}
\citep{wang2024climage} and~\cite{aclimage2025} released real-world datasets (CLImage and ACLImage), highlighting practical challenges such as noise, bias, and imbalanced annotations.
To improve reproducibility and fair comparison, \citep{libcll_2024} released \texttt{libcll}, a Python toolkit that standardizes datasets, assumptions, and algorithm implementations under a unified interface.
\citep{wei2023classimbalanced} addressed class-imbalanced training data via weighted empirical risk minimization and provided supporting generalization analysis.



\section{Further Analysis and Discussion}
\label{further_discussion}

\subsection{Limitations and Future Directions}
\label{limitations}
Our study points to two useful extensions. First, although we evaluate the proposed protocol across multiple VLMs, human-in-the-loop CL collection remains an important next step for understanding how human semantic ambiguity and labeling bias affect constrained labeling. Second, our $Q$-aware methods follow the standard CLL protocol of estimating $Q$ from a small clean seed set, typically 2--5 true-labeled examples per class. 
Reducing this dependence on ordinary labels, for example via CL-based $Q$ estimation or robust objectives, is a promising direction for future work.

\subsection{Conditional Entropy and Mutual Information of BICL datasets}
\label{sec:entroy_mutualinfor}
In this section, we provide a quantitative analysis of the supervision quality of our proposed datasets by calculating the conditional entropy $H(Y \mid \bar Y)$ and mutual information $I(Y; \bar{Y})$ based on its transition matrices. As discussed in our theoretical justification (Section~\ref{sec:intuitive_justification}), we hypothesize that these two properties may indicate improved model performance. We expect to design a data collection protocol obtaining minimum conditional entropy and maximum mutual information. Our observation (as shown in Tables~\ref{tab:entropy_1},~\ref{tab:entropy_2}) confirms that among the compared complementary label datasets, BICL (ours) achieves the lowest conditional entropy and the highest mutual information. This explains the better performance of BICL.

\begin{table*}[ht]
\centering
\caption{Noise characteristics and information-theoretic measures under different annotation settings on CIFAR-10 and CIFAR-20.}
\label{tab:entropy_1}
\resizebox{\textwidth}{!}{
\small
\setlength{\tabcolsep}{4pt}
\renewcommand{\arraystretch}{0.9}
\begin{tabular}{l|cccc|cccc}
\toprule
\textbf{Datasets} 
& \multicolumn{4}{c|}{\textbf{CIFAR-10}} 
& \multicolumn{4}{c}{\textbf{CIFAR-20}} \\
\midrule
\textbf{Metric}
& Uniform & CLImage & ACLImage & \textbf{BICL}
& Uniform & CLImage & ACLImage & \textbf{BICL} \\
\midrule

$H(Y \mid \bar{Y})$ 
& 3.1688 & 3.2410 & 3.0191 & \textbf{2.1999} 
& 4.2426 & 4.2692 & 4.1996 & \textbf{3.6177} \\
$I(Y; \bar{Y})$ 
& 0.1531 & 0.0809 & 0.3028 & \textbf{1.1219} 
& 0.0793 & 0.0528 & 0.1224 & \textbf{0.7042} \\
\midrule
Noise ratio (\%)
& 0.0 & 3.93 & 0.24 & 0.23 
& 0.0 & 2.80 & 0.89 & 0.80 \\
\bottomrule
\end{tabular}
}
\begin{tablenotes}
\item \textbf{Note}: In our setting, a better complementary label dataset should yield lower conditional entropy $H(Y \mid \bar{Y})$ and higher mutual information $I(Y; \bar{Y})$.
\end{tablenotes}
\end{table*}

\begin{table*}[ht]
\centering
\caption{Noise characteristics and information-theoretic measures on CIFAR-100 and TinyImageNet-200.}
\label{tab:entropy_2}
\small
\setlength{\tabcolsep}{6pt}
\renewcommand{\arraystretch}{0.9}
\begin{tabular}{l|cc|cc}
\toprule
\textbf{Datasets} 
& \multicolumn{2}{c|}{\textbf{CIFAR-100}} 
& \multicolumn{2}{c}{\textbf{TinyImageNet-200}} \\
\midrule
\textbf{Metric}
& Uniform & \textbf{BICL}
& Uniform & \textbf{BICL} \\
\midrule

$H(Y \mid \bar Y)$ 
& 6.4640 & \textbf{4.5727} 
& 7.2824 & \textbf{5.8390} \\

$I(Y; \bar{Y})$ 
& 0.1798 & \textbf{2.0711 }
& 0.3615 & \textbf{1.8048} \\
\midrule
Noise ratio (\%)
& 0.0 & 0.026 
& 0.0 & 0.006 \\
\bottomrule
\end{tabular}
\begin{tablenotes}
    \item \textbf{Note}: CIFAR-100 and TinyImageNet-200 are not available for CLImage and ACLImage. Most prior work on CLL has focused on datasets with at most 20 classes.
\end{tablenotes}
\end{table*}

\subsection{Ablation on the Source of Improvement in BICL: Label Selection Bias vs. VLM Prior Knowledge}
\label{source_bias}

To examine where the resource of improvement in BICL comes from the biased transition or the prior knowledge of the VLM as annotator. We conduct an ablation study using two types of annotators:

\begin{enumerate}
    \item \textbf{VLM annotator}: Which are provided in Appendix~\ref{number_sampled_labels} “Effect of the Number of Sampled Labels”. We selected 4 candidate labels as CLs for each class. Take note that the Appendix~\ref{number_sampled_labels} is also discard the true label from candidate labels.
    \item \textbf{A rule-based annotator}: discard the true label (reduce the candidate set to 4 classes), and then uniformly select one from the remaining 4 (all 4 are CLs) classes. We can do so since in Figure~\ref{fig:1}, we have the true class.
\end{enumerate}
\begin{table}[ht]
\centering
\caption{Comparison between the VLM annotator and the rule-based annotator on CIFAR-20 (accuracy (\%), mean$\pm$std).}
\label{tab:annotator_ablation}
\vspace{4pt}
\begin{tabular}{l|l|l|c}
\toprule
\textbf{Annotator} & \textbf{Method} & \textbf{Dataset} & \textbf{Accuracy (\%)} \\
\midrule
\multirow{2}{*}{(1) VLM annotator}
& FWD   & CIFAR-20 & 68.89\scriptsize{$\pm$0.87} \\
& CPE-F & CIFAR-20 & 71.13\scriptsize{$\pm$1.00} \\
\midrule
\multirow{2}{*}{(2) A rule-based annotator}
& FWD   & CIFAR-20 & 69.80\scriptsize{$\pm$0.79} \\
& CPE-F & CIFAR-20 & 70.29\scriptsize{$\pm$1.24} \\
\bottomrule
\end{tabular}
\end{table}

The results of (1) and (2) (as shown in Table~\ref{tab:annotator_ablation}) are similar, suggesting that the selection bias introduced by the VLM itself is not the \textbf{main factor} driving the performance gain. Instead, the results support our claim that the improvement primarily comes from the biased transition structure. This is also because what we obtained are instance-dependent CLs, the global unique $Q$ is just a model of what we obtained and cannot reflect all the features of those CLs. We need to show that, as long as $Q$ is the same, the improvement is guaranteed no matter who the annotator is.

\subsection{Effect of Biased Data Collection on Minority-Class Performance}
\label{minority_class_analysis}

To better understand the effect of the imbalance induced by BICL, we conduct an ablation study to compare the per-class accuracy under a uniform complementary-label distribution, where each complementary class is balanced, with that under the BICL distribution, as shown in the following Figure~\ref{fig:minority_class}.

\begin{figure}[htb]
  \centering
  \includegraphics[width=0.4\textwidth]{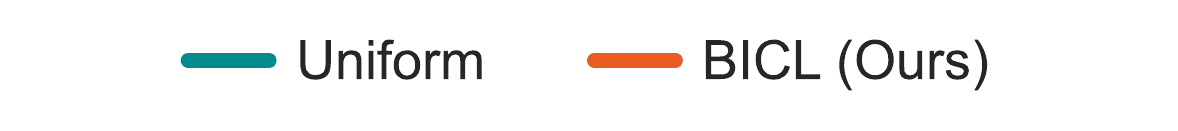} \\
  \vspace{0.05cm} 
  \begin{subfigure}[b]{0.34\textwidth}
    \includegraphics[width=\textwidth]{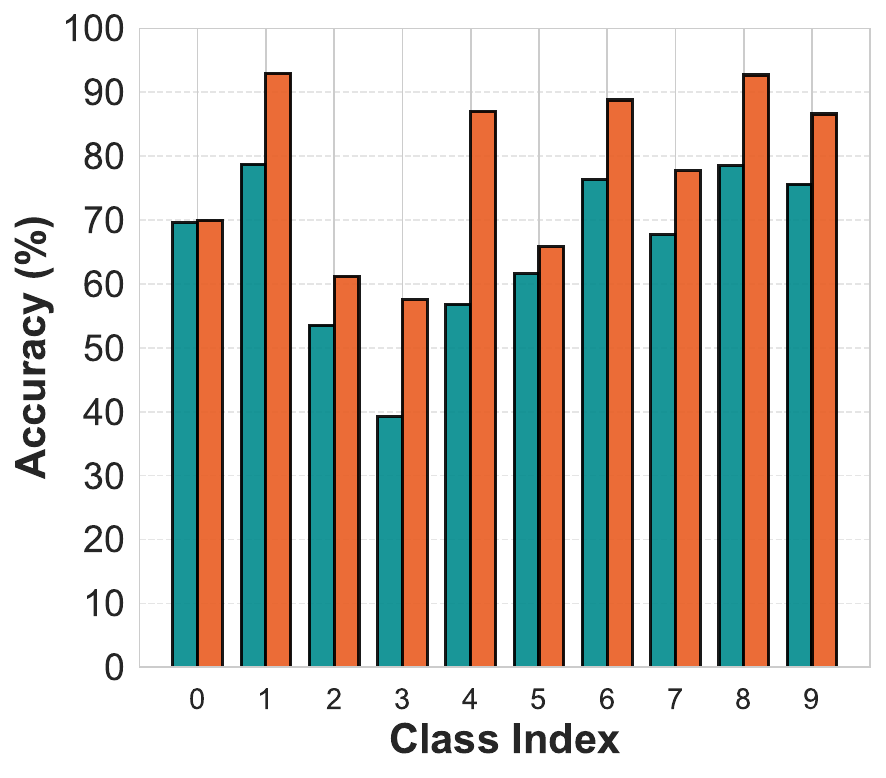}
    \caption{CIFAR-10}
    \label{fig:minority_1}
  \end{subfigure}
  \begin{subfigure}[b]{0.57\textwidth}
    \includegraphics[width=\textwidth]{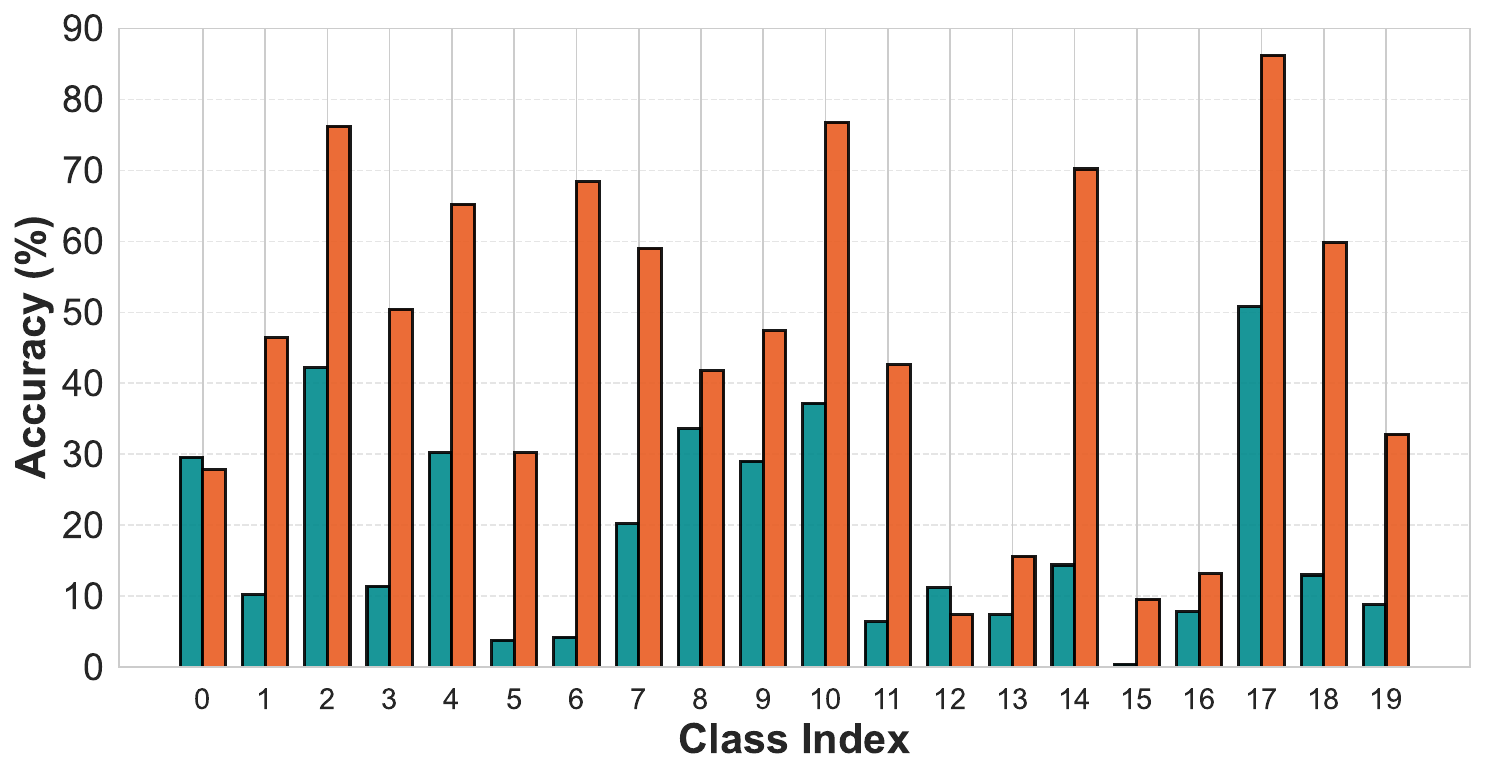}
    \caption{CIFAR-20}
    \label{fig:minority_2}
  \end{subfigure}
\caption{Per-class accuracy comparison between the uniform complementary-label distribution and the BICL distribution on CIFAR-10 (left) and CIFAR-20 (right) under the FWD algorithm. Although BICL induces an imbalanced complementary-label distribution, it does not degrade performance on minority classes; instead, it often achieves better minority-class performance than the uniform distribution.}
  \label{fig:minority_class}
\end{figure}

These results suggest that the imbalanceness in CLs generation process in BICL does not affect performance; in fact, it can even lead to better performance. This highlights that the imbalance in CLL, especially in the complementary-label space, behaves differently from class imbalance in ordinary label learning, where minority classes are often severely affected.

\subsection{The Optimal of Number Cluster}
\label{number_cluster}

In the label selection phase, the granularity of $K$-means clustering, determined by the number of clusters $Z_c$, is a key hyperparameter for grouping visually similar images prior to VLM annotation. 
We conduct an ablation study to examine the sensitivity of BICL to the number of clusters used in the label-selection stage (Figure~\ref{fig:2}). Specifically, we vary the cluster count while keeping all other components of BICL unchanged. For each setting, the entire BICL pipeline is rerun, including feature clustering and re-annotation of CLs by the VLM under the newly induced candidate label sets. Although this procedure is computationally expensive, it allows for a controlled evaluation of how the cluster granularity affects performance.

We explore cluster sizes ranging from $|C|$ to $4|C|$, where $|C|$ denotes the number of classes in the dataset. The results in Table~\ref{tab:num_cluster} show that setting the number of clusters equal to the class cardinality ($|C|$) yields the best or near-best performance across all four datasets. Increasing the number of clusters beyond $|C|$ consistently degrades performance, particularly on large-scale datasets such as CIFAR-100 and TinyImageNet-200.

This trend can be attributed to the trade-off between semantic coherence and granularity in the clustering process. When the number of clusters matches the number of classes, clusters tend to capture class-level structure, leading to more informative and stable candidate label sets for complementary labeling. In contrast, using a larger number of clusters fragments semantically similar samples into smaller groups, increasing annotation noise and reducing the effectiveness of the induced bias. From a practical perspective, choosing $|C|$ clusters also offers a favorable balance between accuracy and computational cost, making it a natural and robust default choice for BICL.

\begin{table*}[ht]
\centering
\caption{Impact of cluster size on model performance across four datasets. $|C|$ denotes the original number of classes in each dataset.}
\label{tab:num_cluster}
\setlength{\tabcolsep}{6pt}
\renewcommand{\arraystretch}{0.9}
\begin{tabular}{l|ccccc}
\toprule
\textbf{Dataset} & $|C|$ & $1.5|C|$ & $2|C|$ & $3|C|$ & $4|C|$ \\ \midrule
CIFAR-10  & \textbf{81.23}{\tiny $\pm$0.05} & 80.80{\tiny $\pm$0.39} & 78.68{\tiny $\pm$0.46} & 79.01{\tiny $\pm$0.18} & 77.34{\tiny $\pm$0.26} \\
CIFAR-20  & 50.84{\tiny $\pm$0.32} & \textbf{51.22}{\tiny $\pm$0.25} & 50.52{\tiny $\pm$0.11} & 47.95{\tiny $\pm$0.28} & 47.52{\tiny $\pm$0.37} \\
CIFAR-100 & \textbf{46.70}{\tiny $\pm$0.60} & 36.98{\tiny $\pm$0.40} & 36.72{\tiny $\pm$0.33} & 36.25{\tiny $\pm$0.17} & 34.15{\tiny $\pm$0.64} \\
TinyImageNet-200  & \textbf{32.15}{\tiny $\pm$0.30} & 31.79{\tiny $\pm$0.23} & 32.10{\tiny $\pm$0.23} & 30.78{\tiny $\pm$0.40} & 30.45{\tiny $\pm$0.31} \\ 
\bottomrule
\end{tabular}
\end{table*}

\subsection{Effect of the Number of Sampled Labels}
\label{number_sampled_labels}

We conduct an ablation study to investigate how the number of sampled candidate labels affects the performance of BICL and to justify our choice of using a fixed number of four labels per instance for VLM annotation. Specifically, we vary the number of sampled labels $n$ from 2 to 10 for each class.
Similar to the ablation in Section~\ref{number_cluster}, the entire BICL pipeline is rerun for each setting: we resample the candidate label sets, request the VLM to re-annotate CLs under the new label space, and retrain the downstream classifier. Although this procedure is computationally expensive, it enables a controlled evaluation of the impact of label-space size on complementary-label quality and downstream performance.

\begin{figure}[ht]
  \centering
  \includegraphics[width=0.38\textwidth]{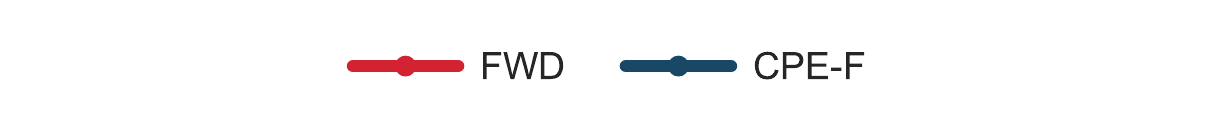} \\
  \vspace{0.05cm} 
  \begin{subfigure}[b]{0.245\textwidth}
    \includegraphics[width=\textwidth]{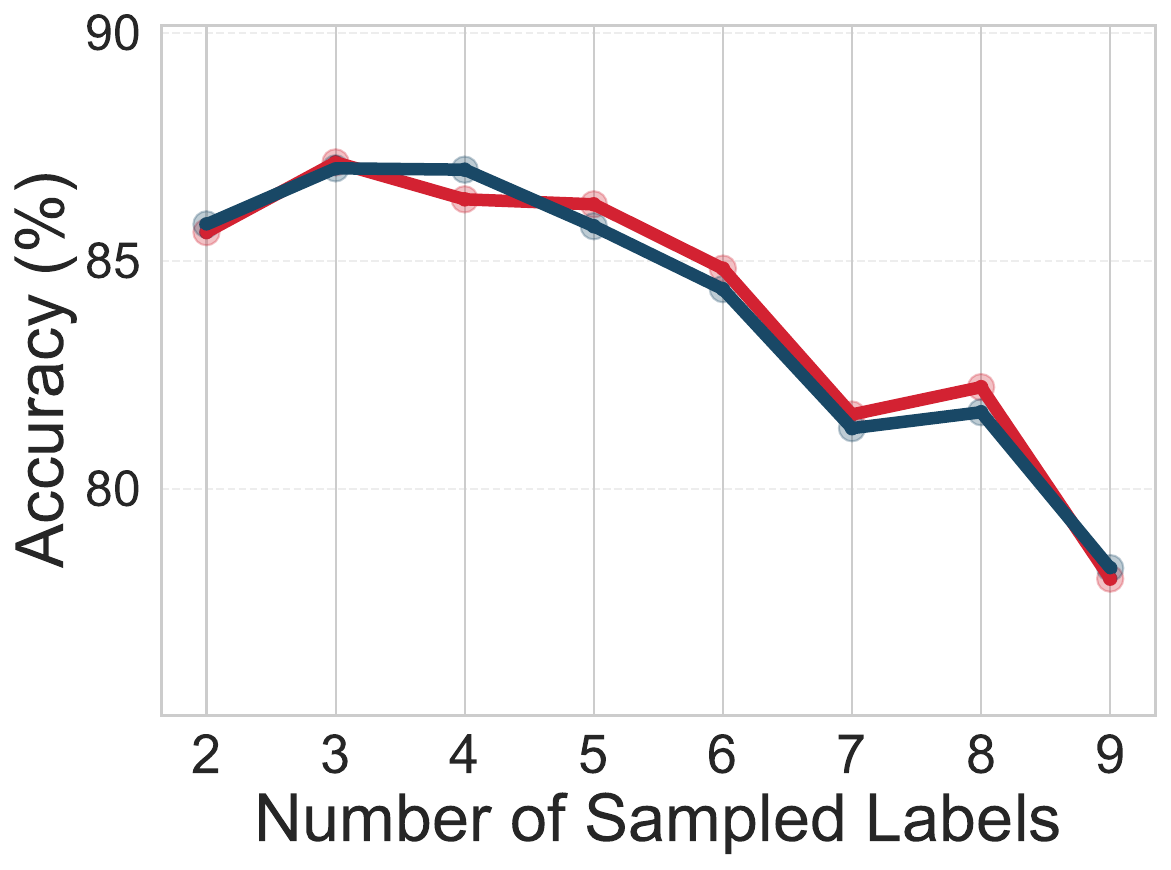}
    \caption{CIFAR-10}
  \end{subfigure}
  \begin{subfigure}[b]{0.236\textwidth}
    \includegraphics[width=\textwidth]{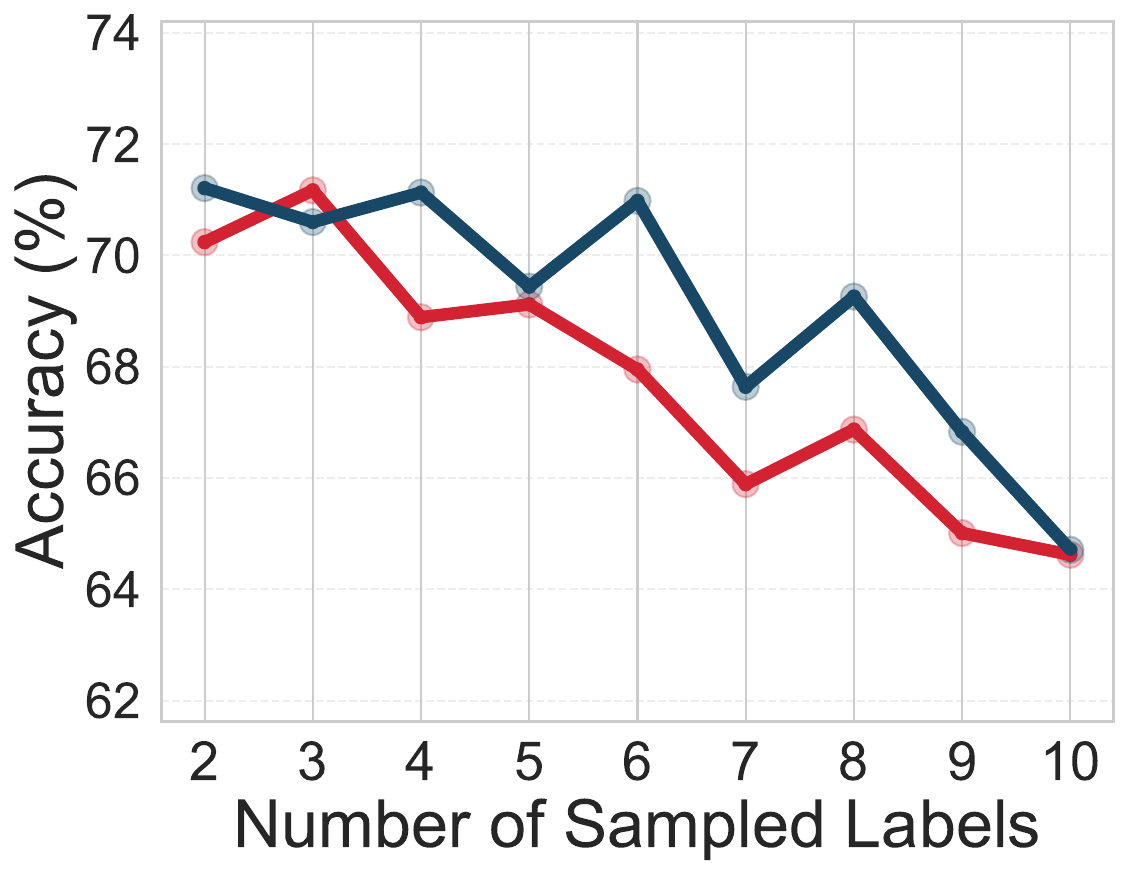}
    \caption{CIFAR-20}
  \end{subfigure}
  \begin{subfigure}[b]{0.245\textwidth}
    \includegraphics[width=\textwidth]{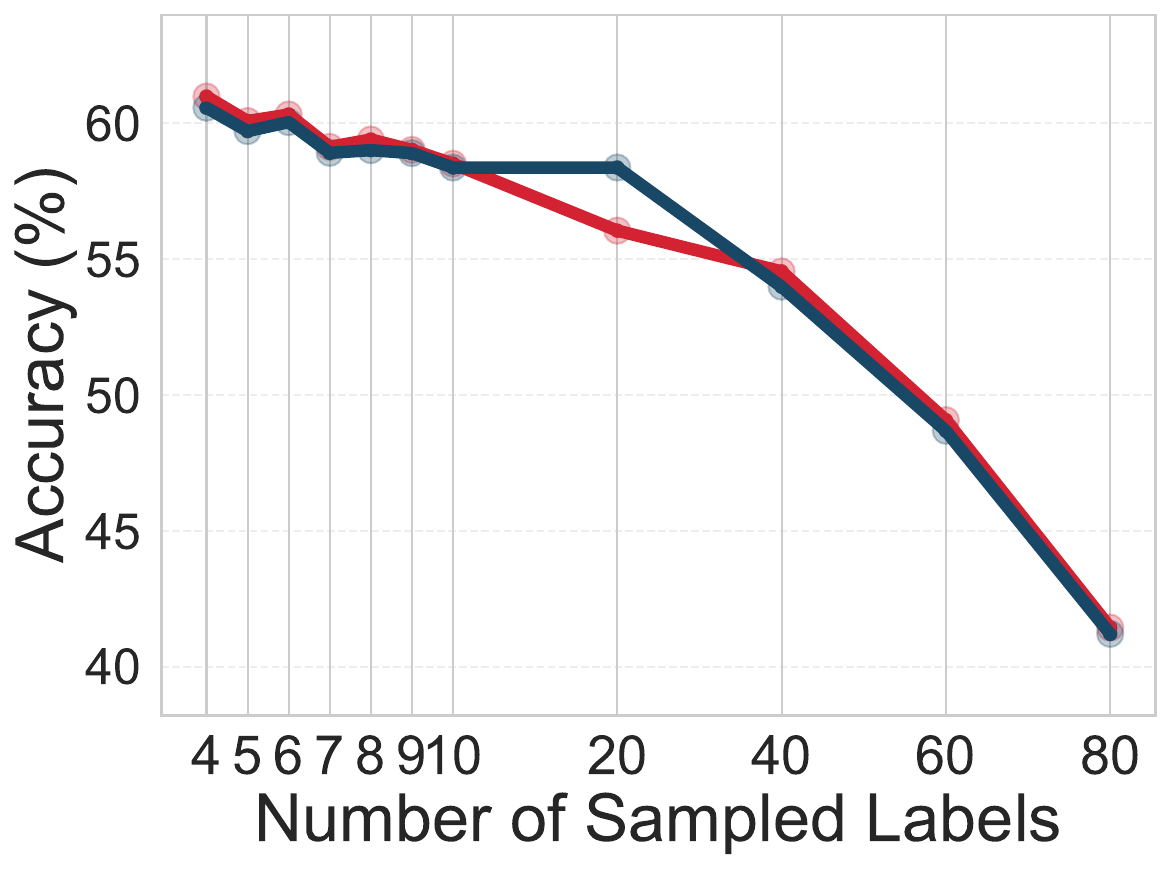}
    \caption{CIFAR-100}
  \end{subfigure}
  \begin{subfigure}[b]{0.245\textwidth}
    \includegraphics[width=\textwidth]{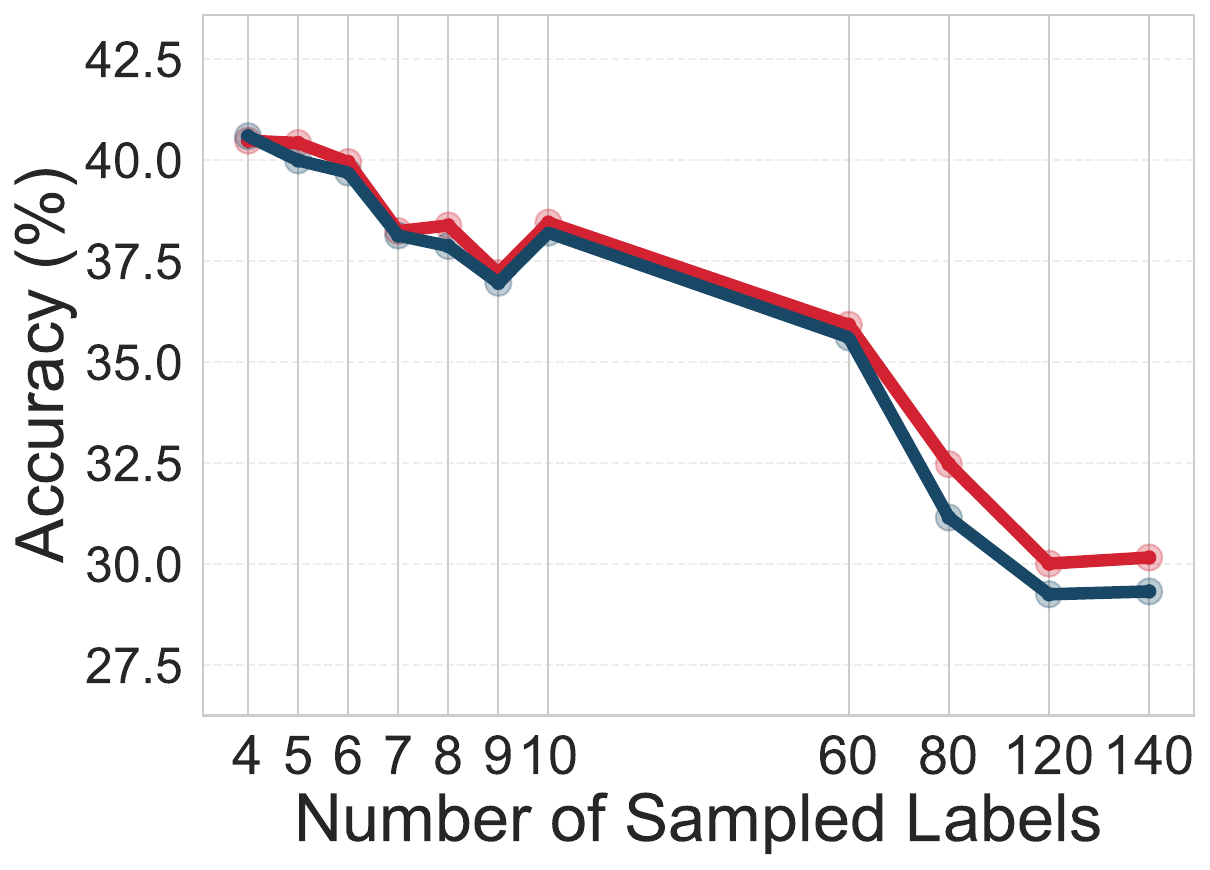}
    \caption{TinyImageNet-200}
  \end{subfigure}  \caption{Performance comparison of different number of sampled label with BICL.}
\label{fig:sampled_label}
\end{figure}

The results, illustrated in Figure~\ref{fig:sampled_label}, show that sampling a small number of labels ($n=3$ or $n=4$) consistently yields the best or near-best performance across all datasets and algorithms. In contrast, increasing $n$ beyond this range leads to a steady degradation in accuracy. This trend becomes particularly pronounced as $n$ approaches the Uniform setting, where the full label space is used. These findings help explain why the Uniform assumption performs poorly in practice: presenting too many candidate labels increases ambiguity for the VLM, resulting in noisier and less informative complementary supervision.
Between $n=3$ and $n=4$, we adopt $n=4$ in our main experiments. This choice follows prior work~\cite{aclimage2025,wang2024climage,wei2022learning}, which also uses four candidate labels per instance, ensuring a fair comparison. Moreover, $n=4$ provides a robust balance between reducing label uncertainty and maintaining sufficient diversity in the candidate set, making it a practical and effective default for BICL.

\subsection{Statistical Significance of Experimental Results}
\label{statistical}
To verify the performance improvements of our proposed method BICL against uniform baseline assumption, we conducted paired t-tests. We performed this analysis across four representative algorithms: FWD, CPE-F, CPE-T, CPE-I for each dataset.

\begin{table}[ht]
    \centering
    \caption{Statistical significance ($p$-values) of the performance improvements of BICL compared to the baseline across datasets. The $p$-values are calculated using a one-tailed paired t-test with 3 different random seeds. In all cases, $p < 0.05$, confirming the significance of the results.}
    \label{tab:significance_test}
    \vspace{5pt}
    \resizebox{0.8\textwidth}{!}{
    \begin{tabular}{l|cccc}
        \toprule
        \textbf{Algorithm} & \textbf{CIFAR-10} & \textbf{CIFAR-20} & \textbf{CIFAR-100} & \textbf{TinyImageNet-200} \\
        \midrule
        FWD   & $8.13 \times 10^{-4}$ & $1.41 \times 10^{-4}$ & $3.42 \times 10^{-4}$ & $3.59 \times 10^{-4}$ \\
        CPE-F & $1.9 \times 10^{-3}$ & $9.84 \times 10^{-5}$ & $3.36 \times 10^{-4}$ & $3.75 \times 10^{-4}$ \\
        CPE-I & $3.13 \times 10^{-3}$ & $3.84 \times 10^{-5}$ & $4.32 \times 10^{-5}$ & $3.59 \times 10^{-5}$ \\
        CPE-T & $3.88 \times 10^{-2}$ & $3.82 \times 10^{-4}$ & $6.18 \times 10^{-4}$ & $4.74 \times 10^{-4}$ \\
        \bottomrule
    \end{tabular}
    }
\end{table}

The experiments were repeated three times with different random seeds. For the significant test, we adopted a confidence threshold $\alpha = 0.05$. We setup all the hypotheses as follow:
\begin{itemize}
    \item \textbf{Null Hypothesis ($H_0$):} The mean accuracy difference between BICL and the baseline - uniform distribution - is less than or equal to zero.
    \item \textbf{Alternative Hypothesis:} BICL achieves higher  accuracy mean than the baseline.
\end{itemize}
As shown in Table \ref{tab:significance_test}, the calculated $p-$values from all evaluated settings are substantially lower then $0.05$. Base on this result, we firmly reject the null hypothesis in favor of the alternative hypothesis. This statistically analysis confirms that the performance gains achieved by BICL reported in Section~\ref{sec:6} of the main text are significant.

\subsection{Extended BICL Dataset Analysis}
\label{extended_bicl_dataset}
        
In this section, we will provide additional information about the properties of all remaining BICL datasets. We validate that the three key characteristics observed in CIFAR-10, biased transitions, label imbalance, and low noise rates, are consistent across datasets with varying numbers of classes and complexity.

\begin{figure*}[htb]
  \centering
  \begin{subfigure}[b]{0.24\textwidth}
\includegraphics[width=\textwidth]{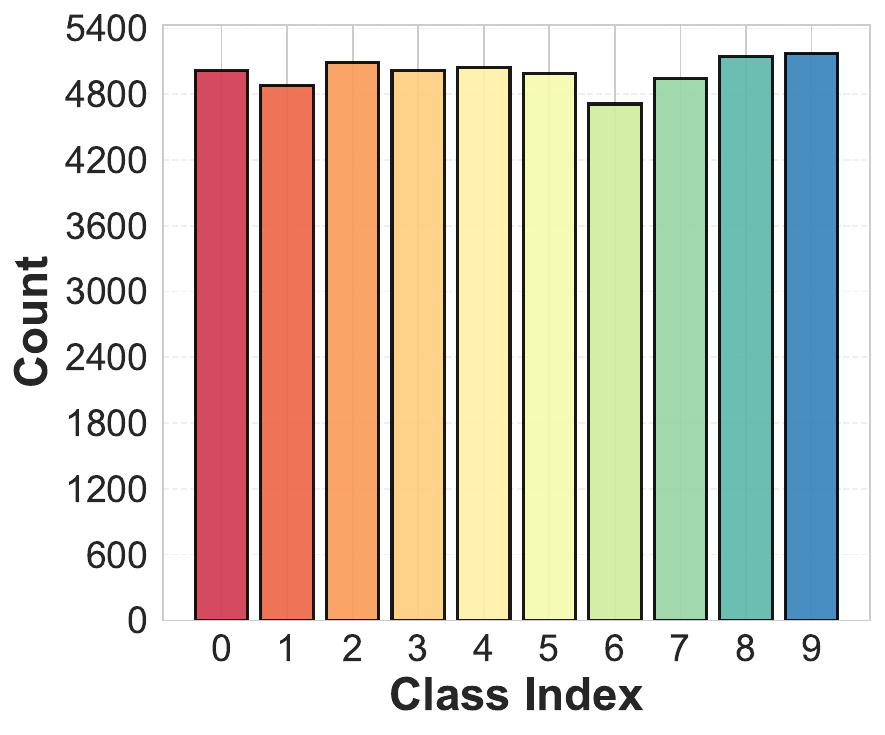}
    \caption{CIFAR-10}
  \end{subfigure}
  \begin{subfigure}[b]{0.24\textwidth}
\includegraphics[width=\textwidth]{Sections/Figures/dist_cifar10_random.pdf}
    \caption{CLCIFAR-10}
  \end{subfigure}
  \begin{subfigure}[b]{0.24\textwidth}
\includegraphics[width=\textwidth]{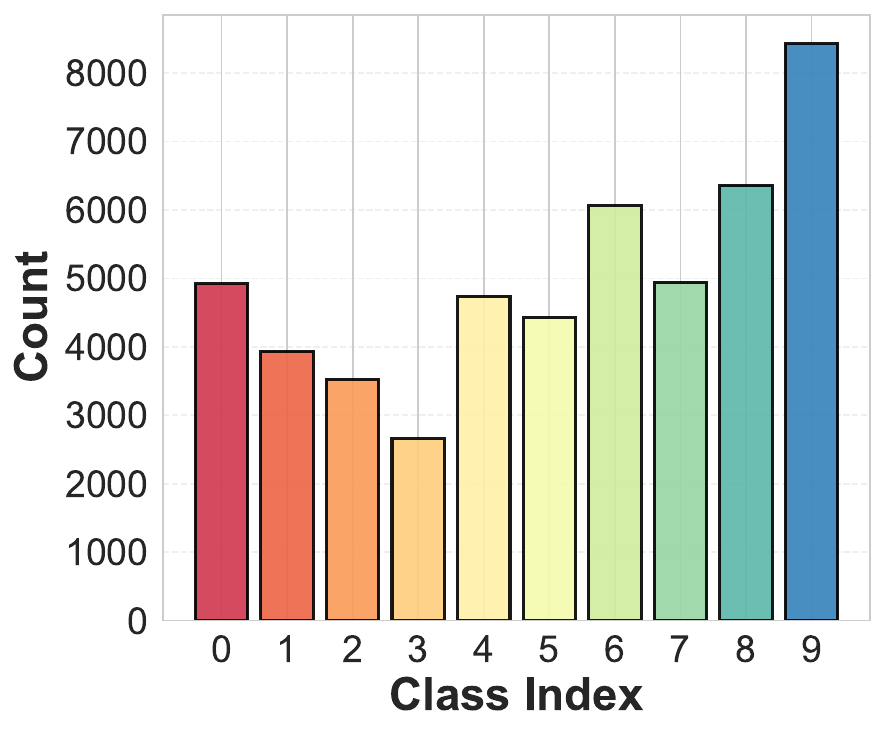}
    \caption{ACLCIFAR-10}
  \end{subfigure}
  \begin{subfigure}[b]{0.24\textwidth}
\includegraphics[width=\textwidth]{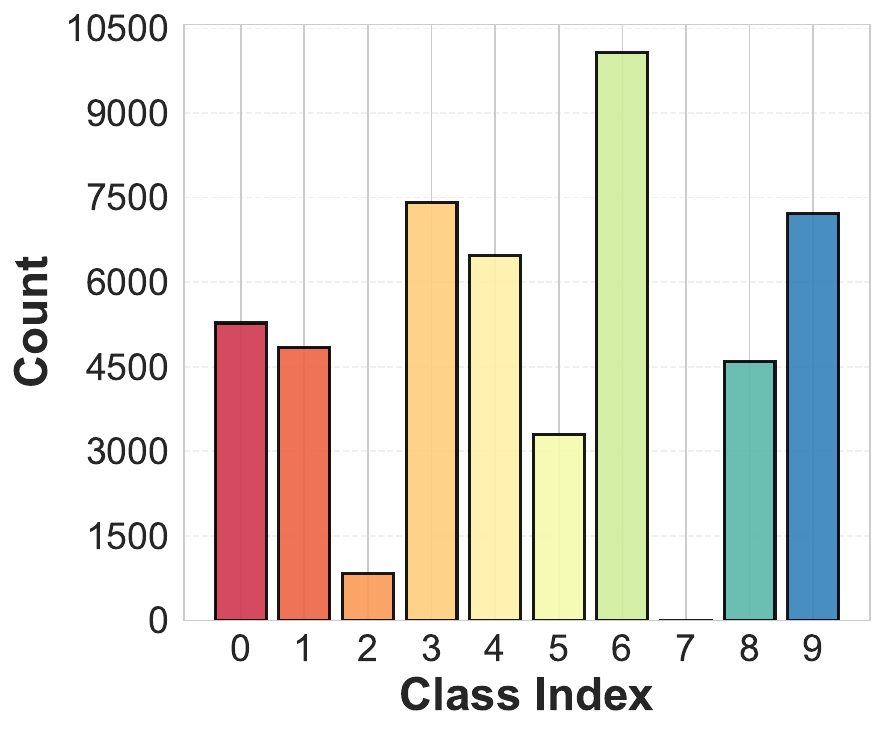}
    \caption{BCLCIFAR-10}
  \end{subfigure}
  \caption{\textbf{Label distributions across CIFAR-10 variants}.
CIFAR-10 represents an idealized setting with noiseless, uniformly distributed labels. CLCIFAR-10 corresponds to a human-annotated and ACLCIFAR-10 is VLM-annotated under a uniform distribution design.}
  \label{fig:label-dist-cl10}
\end{figure*}

\begin{figure}[htb]
  \centering
  \begin{subfigure}[b]{0.24\textwidth}
\includegraphics[width=\textwidth]{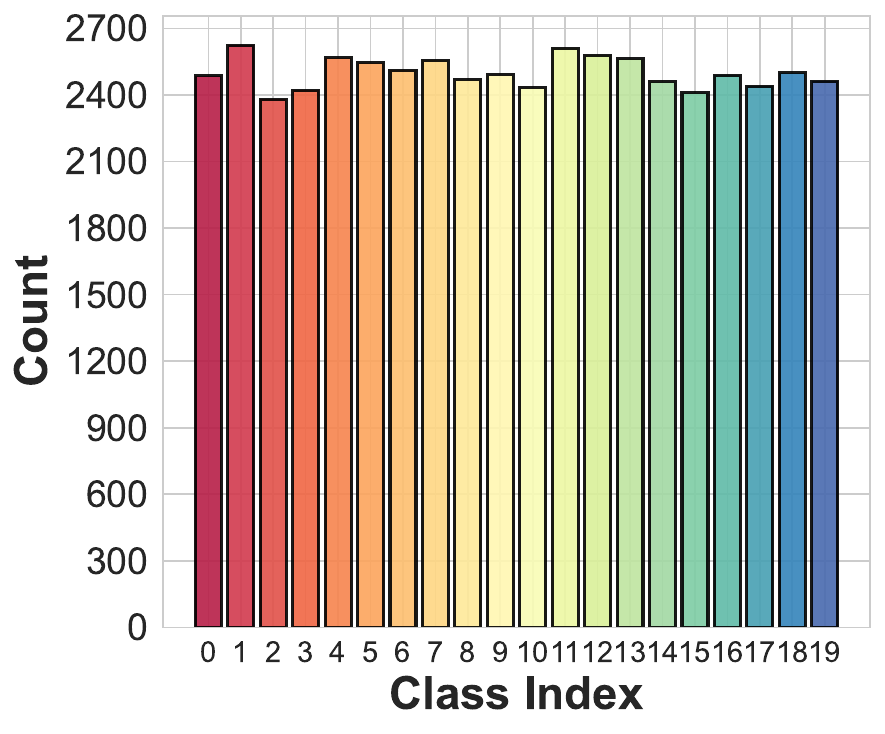}
    \caption{CIFAR-20}
  \end{subfigure}
  \begin{subfigure}[b]{0.24\textwidth}
\includegraphics[width=\textwidth]{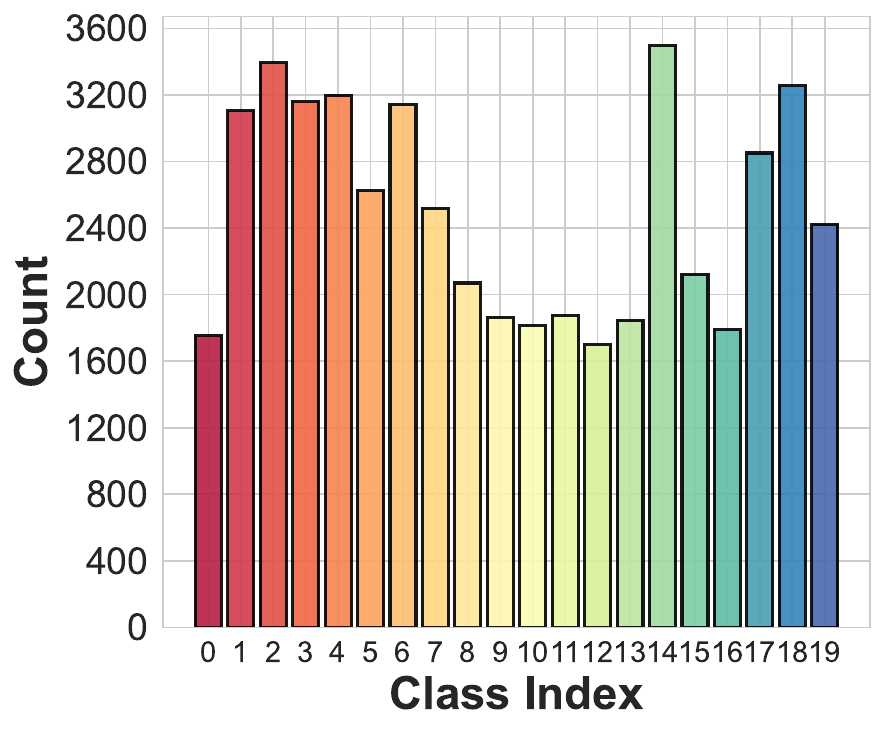}
    \caption{CLCIFAR-20}
  \end{subfigure}
  \begin{subfigure}[b]{0.24\textwidth}
\includegraphics[width=\textwidth]{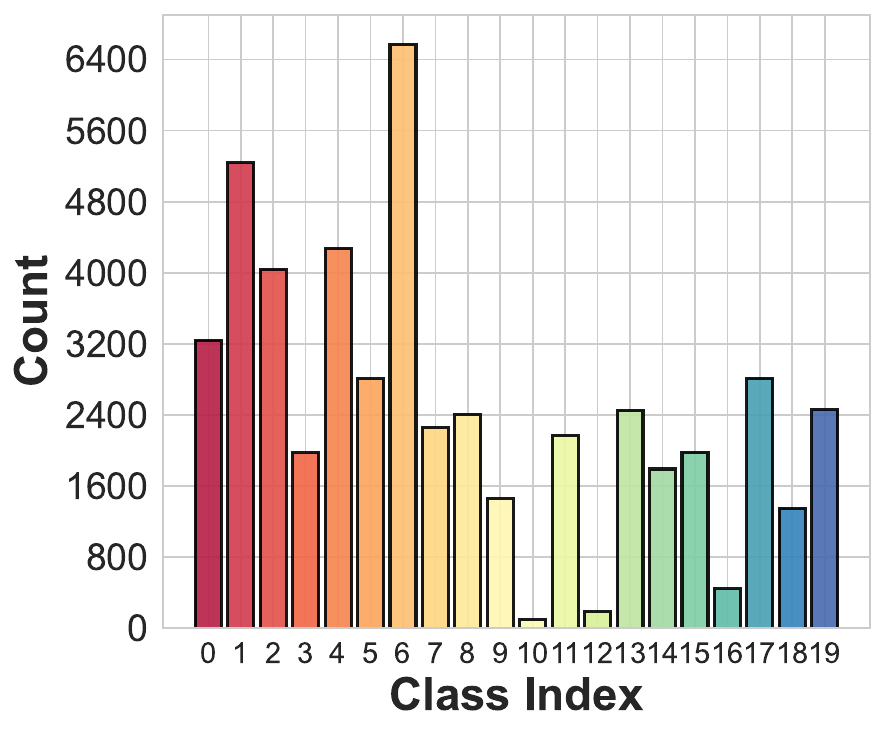}
    \caption{ACLCIFAR-20}
  \end{subfigure}
  \begin{subfigure}[b]{0.24\textwidth}
\includegraphics[width=\textwidth]{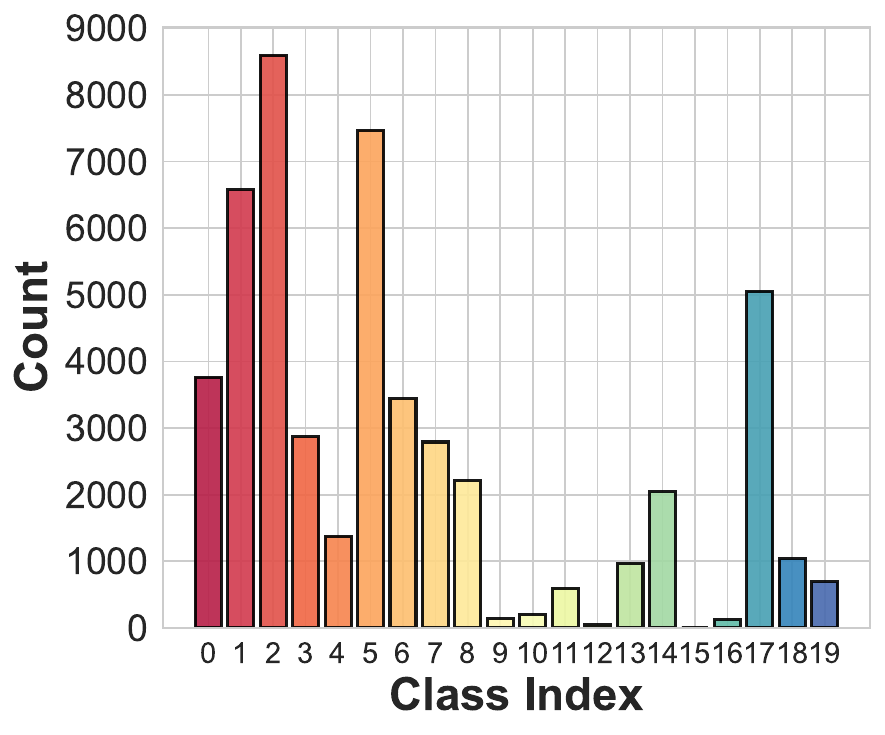}
    \caption{BCLCIFAR-20}
  \end{subfigure}
  \caption{Label distributions across CIFAR-20 variants.}
  \label{fig:label-dist-cl20}
\end{figure}

As illustrated in Figure~\ref{fig:transition-matrix-cifar20}, Figure~\ref{fig:transition-matrix-cl100} the transition matrix for BCLCIFAR-20, BCLCIFAR-100 exhibit distinct, non-uniform patterns similar to the BCLCIFAR-10 results in the main text. This resulted from our design where we deliberately constrain the labeling space combine with innate inductive biases characteristic of VLMs.

\begin{figure}[htb]
  \centering
  \begin{subfigure}[b]{0.48\textwidth}
\includegraphics[width=\textwidth]{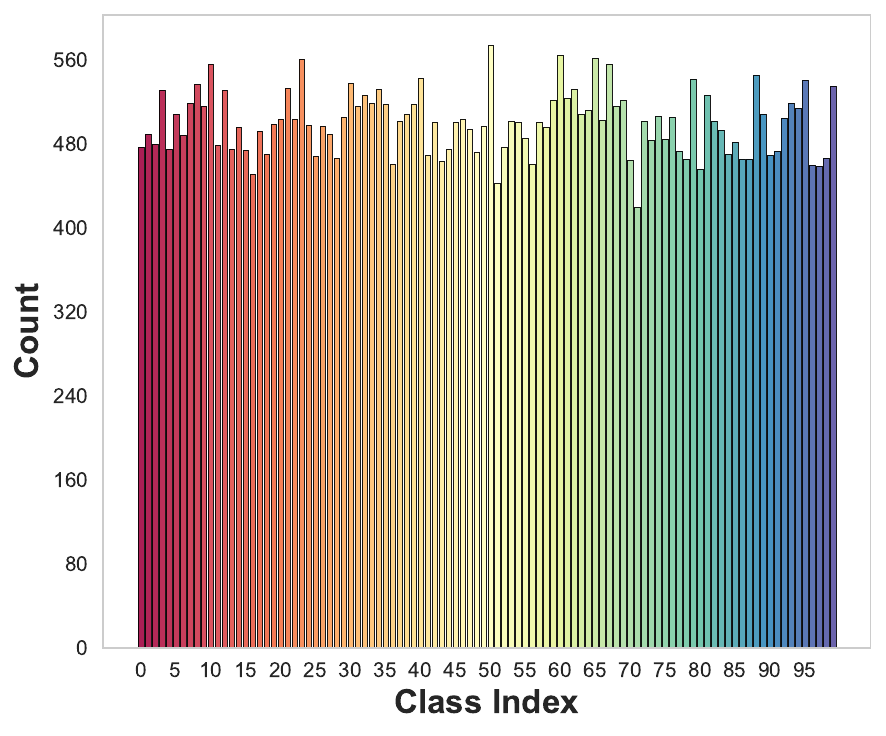}
    \caption{CIFAR-100}
  \end{subfigure}
  \begin{subfigure}[b]{0.48\textwidth}
\includegraphics[width=\textwidth]{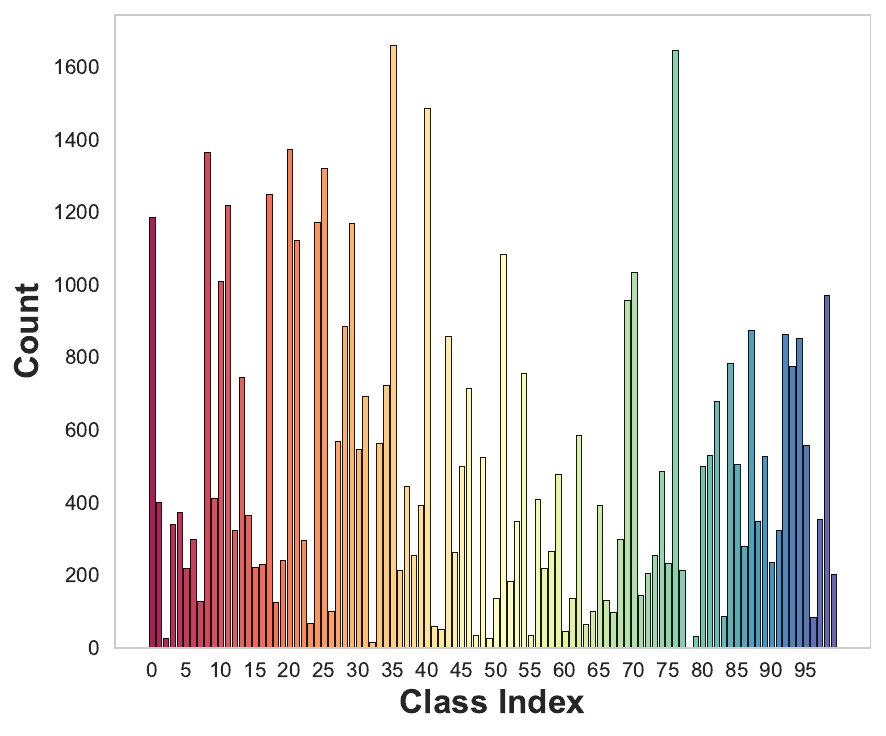}
    \caption{BCLCIFAR-100}
  \end{subfigure}
  \caption{Label distributions across CIFAR-100 variants.}
  \label{fig:label-dist-cl100}
\end{figure}

\begin{figure}[ht]
  \centering
  \begin{subfigure}[htb]{0.45\textwidth}
    \includegraphics[width=\textwidth]{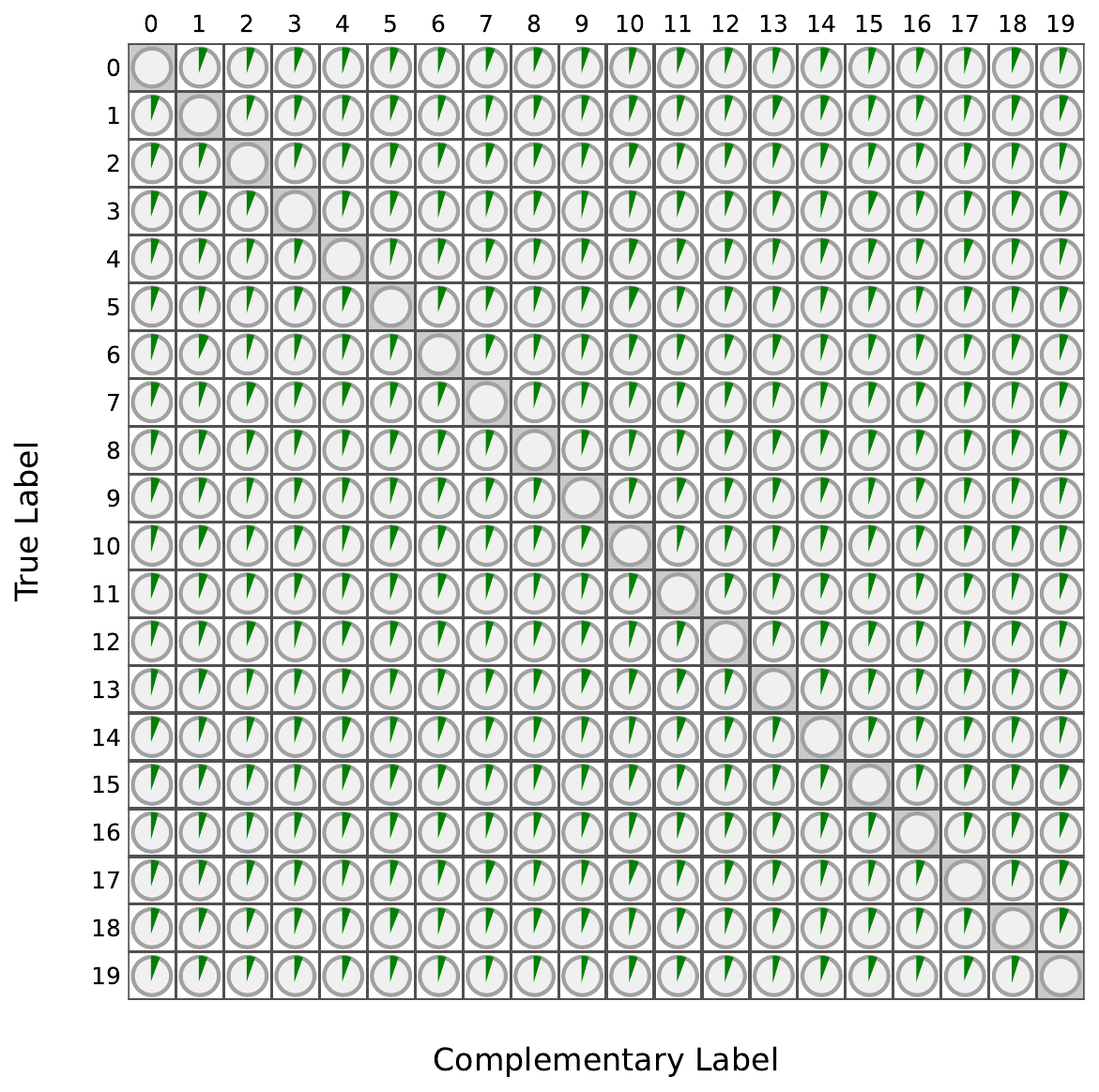}
    \caption{CIFAR-20}
  \end{subfigure}
  \begin{subfigure}[htb]{0.45\textwidth}
    \includegraphics[width=\textwidth]{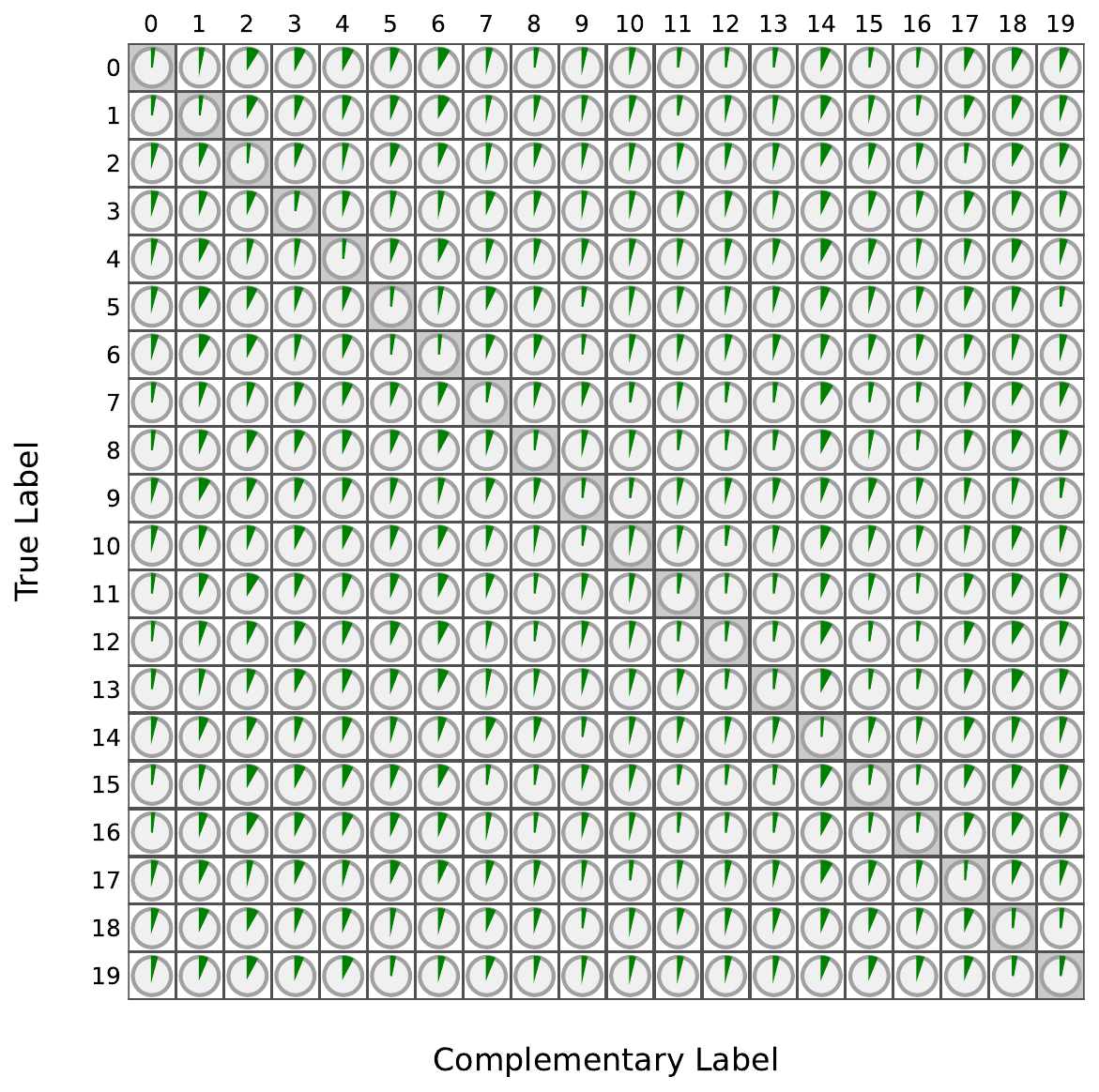}
    \caption{CLCIFAR-20}
  \end{subfigure}
  \begin{subfigure}[htb]{0.45\textwidth}
    \includegraphics[width=\textwidth]{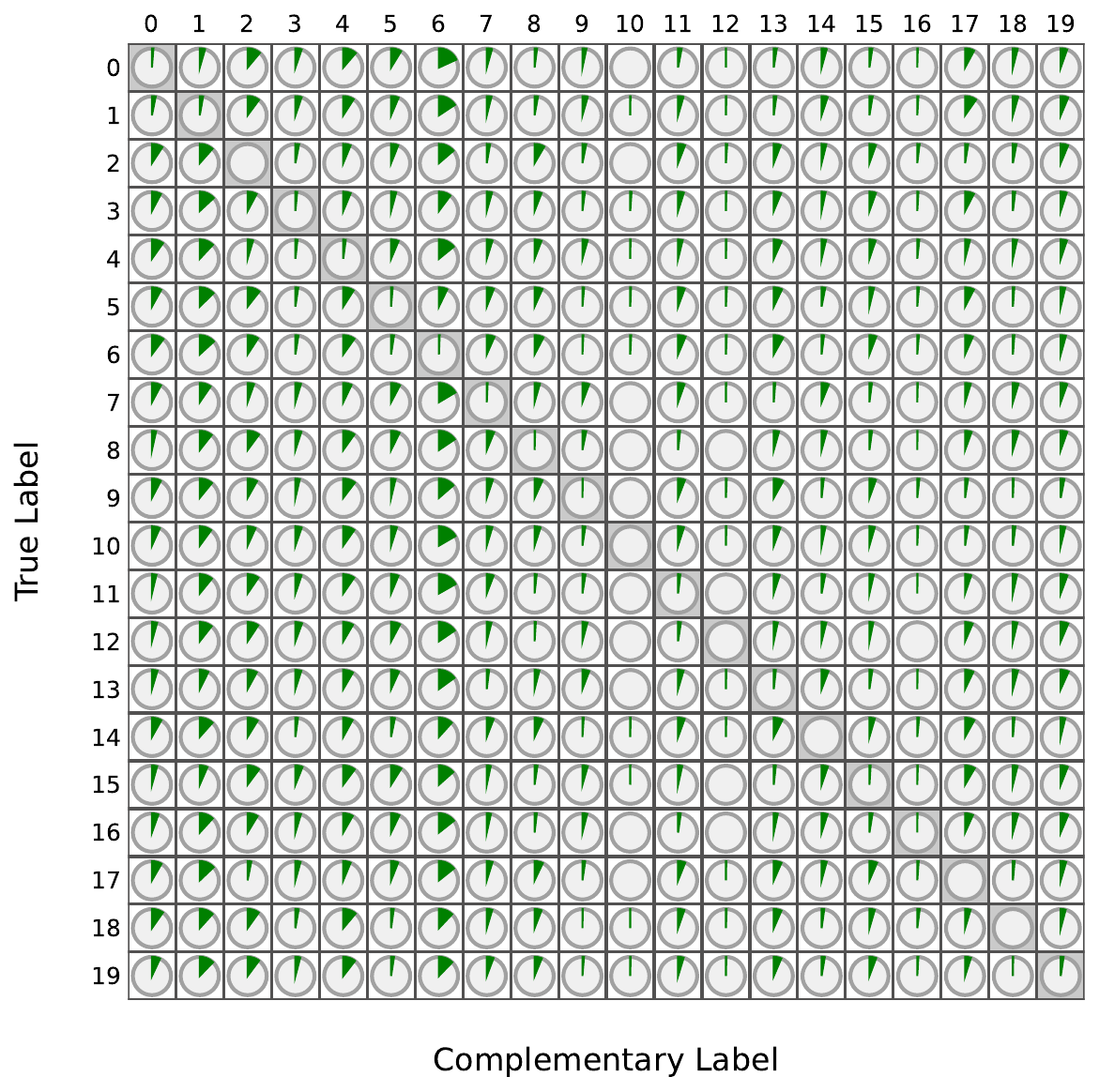}
    \caption{ACLCIFAR-20}
  \end{subfigure}
  \begin{subfigure}[htb]{0.45\textwidth}
    \includegraphics[width=\textwidth]{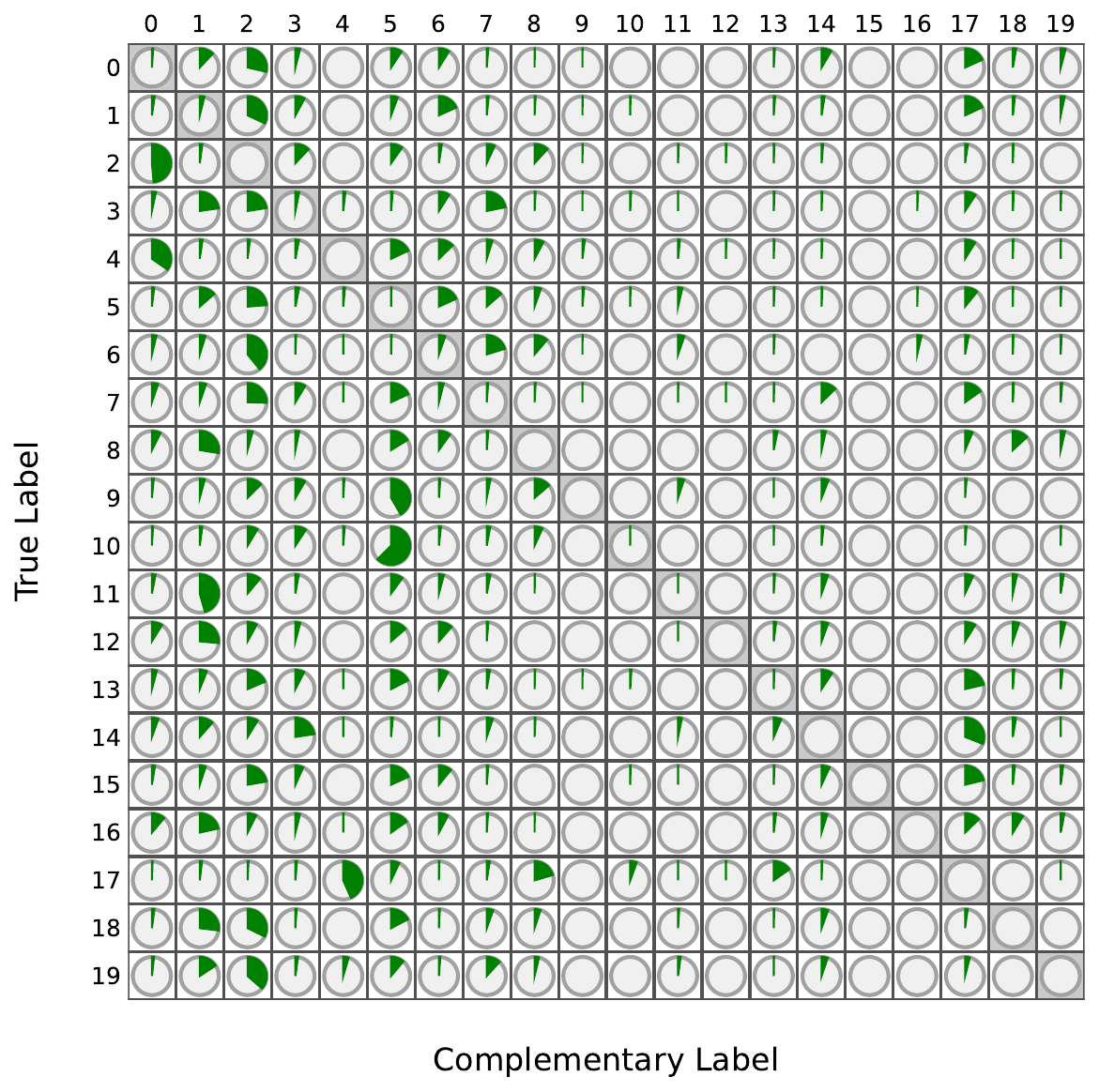}
    \caption{BCLCIFAR-20}
  \end{subfigure}
  \caption{Bias-Induced Constrained Labeling transition matrix on CIFAR-20 variants.}
  \label{fig:transition-matrix-cifar20}
\end{figure}

The imbalance label distribution characteristic of our generated datasets induced by our method is also present through out.

\textbf{Low label noise rate.}
Despite its pronounced bias and imbalance, our dataset maintains a low label noise rate. Specifically, the average noise rate of BCLCIFAR-10 is 0.23 percentage points, which is comparable to that of ACLCIFAR-10 (0.24 percentage points) and substantially lower than that of the human-annotated CLCIFAR-10 dataset (3.93 percentage points). This result suggests that VLM-based annotation, when combined with a controlled bias design, can achieve favorable noise characteristics while enabling structured complementary-label generation.

\begin{table}[htb]
    \centering
    \caption{Imbalance ratios and noise rates for the proposed BICL datasets.}
    \label{tab:extended-bicl}
    \vspace{5pt}
    \begin{tabular}{l|c|c}
        \toprule
        \textbf{Dataset} & \textbf{Imbalance Ratio} & \textbf{Noise Rate (\%)} \\
        \midrule
        BCLCIFAR-10          & 4,966.5 & 0.23 \\
        BCLCIFAR-20          & 200.12  & 0.80 \\
        BCLCIFAR-100         & 119.59 & 0.026 \\
        BCLTinyImageNet-200  & 416.35 & 0.006 \\
        \bottomrule
    \end{tabular}
\end{table}

Figure \ref{fig:label-dist-cl10}, \ref{fig:label-dist-cl20} and \ref{fig:label-dist-cl100} present the histogram of complementary label. The specific imbalance ratio~\cite{CITEKC2019,CITEYC2019} of our dataset can be found in Table \ref{tab:extended-bicl}. These findings align with inherent heavy bias nature of our transition matrices. However, larger labels create a counter-intuitive phenomenon-a subset of classes receives zero probability in the final complementary-label distribution. This situation could arise from the interaction between our candidate sampling strategy and VLM's inherent biases. Specifically, the process involves instructing the VLM with a small subset of candidates along side a negative prompt. Consequently, low-profile and less semantically distinct classes are excluded from the final dataset.

Crucially, this low-noise characteristic is preserved regardless of the label space size. As summarized in Table \ref{tab:extended-bicl}, our method achieves consistently low error rates across all evaluated datasets.

\section{Additional Ablation Studies}
\label{additional_ablation}

\subsection{Result Analysis of Different Neural Network Architectures}

We conducted an ablation study comparing different neural network architectures, including ResNet18, ResNet34, and ResNet50 to determine the most suitable architecture for our complementary dataset. Results shown in Table~\ref{tab:performance} indicate that ResNet34 achieved the best performance.

\newcommand{\meanstd}[2]{#1{\tiny$\pm$#2}}
\newcommand{\bestcell}[2]{\textbf{#1}{\tiny$\pm$#2}}
\begin{table}[htbp]
\centering
\caption{Performance of various neural network architectures using FWD and CPE-F algorithms on our BICL complementary datasets.
The highest accuracy for each dataset–architecture pair is shown in bold.}
\label{tab:performance}
\vspace{5pt}
\resizebox{1.02\textwidth}{!}{%
\setlength{\tabcolsep}{2pt}
\renewcommand{\arraystretch}{0.9}
\begin{tabular}{l|ccc|ccc|ccc|ccc}
\toprule
\multirow{2}{*}{Method} 
& \multicolumn{3}{c|}{\textbf{CIFAR-10}} 
& \multicolumn{3}{c|}{\textbf{CIFAR-20}} 
& \multicolumn{3}{c|}{\textbf{CIFAR-100}} 
& \multicolumn{3}{c}{\textbf{TinyImageNet-200}} \\

\cmidrule(lr){2-4}
\cmidrule(lr){5-7}
\cmidrule(lr){8-10}
\cmidrule(lr){11-13}

& ResNet18 & ResNet34 & ResNet50
& ResNet18 & ResNet34 & ResNet50
& ResNet18 & ResNet34 & ResNet50
& ResNet18 & ResNet34 & ResNet50 \\

\midrule
\textbf{FWD}
& \meanstd{78.79}{0.91} & \bestcell{81.23}{0.05} & \meanstd{78.46}{0.20}
& \meanstd{47.68}{0.35} & \bestcell{50.84}{0.32} & \meanstd{45.68}{0.99}
& \meanstd{44.81}{0.98} & \bestcell{46.70}{0.60} & \meanstd{42.04}{0.82}
& \meanstd{31.13}{0.33} & \bestcell{32.15}{0.30} & \meanstd{30.12}{1.07} \\

\midrule
\textbf{CPE-F}
& \meanstd{78.36}{0.67} & \bestcell{80.98}{0.18} & \meanstd{77.90}{0.40}
& \meanstd{46.42}{0.54} & \bestcell{50.89}{0.30} & \meanstd{44.52}{1.20}
& \meanstd{44.36}{0.94} & \bestcell{46.57}{0.09} & \meanstd{41.59}{0.85}
& \meanstd{30.37}{0.39} & \bestcell{31.89}{0.07} & \meanstd{29.74}{0.75} \\

\bottomrule
\end{tabular}
}
\end{table}

\subsection{Result Analysis of Data Augmentation}

We conduct a further analysis on the integration of BICL with various data augmentation techniques to enhance model robustness. Specifically, we evaluate a spectrum of augmentation intensities, ranging from weaker methods, such as Cutout~\cite{cutout}, to stronger automated policies like AutoAug~\cite{autoaugment} and RandAug~\cite{randaugment}. These are compared against the Flipflop baseline, which serves as the standard augmentation strategy employed throughout our experiments in Section~\ref{sec:6}.

\begin{figure*}[ht]
  \centering
  \vspace{0.05cm} 
  \begin{subfigure}[b]{0.24\textwidth}
    \includegraphics[width=\textwidth]{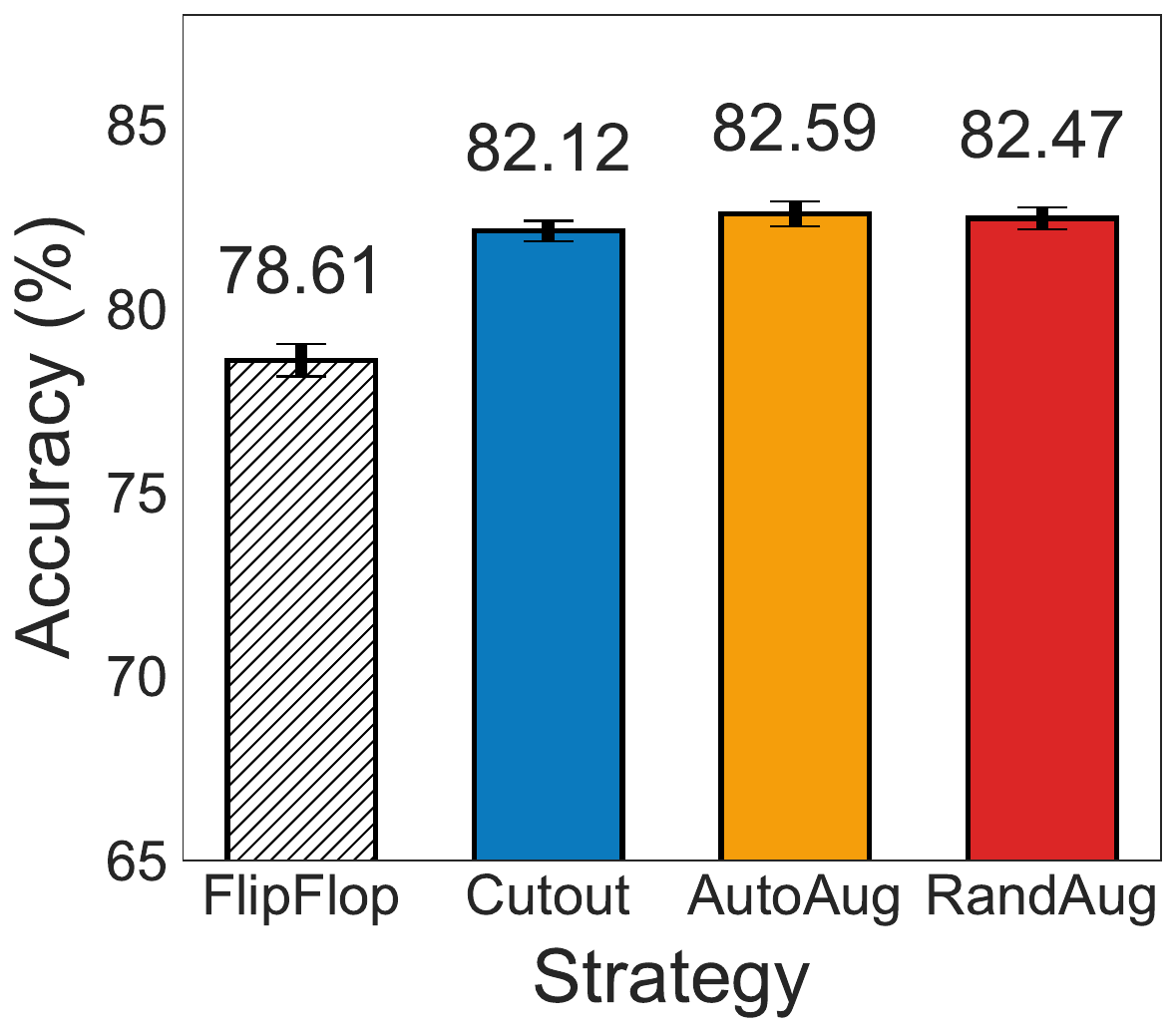}
    \caption{CIFAR-10}
    \label{fig:aug1}
  \end{subfigure}
  \begin{subfigure}[b]{0.24\textwidth}
    \includegraphics[width=\textwidth]{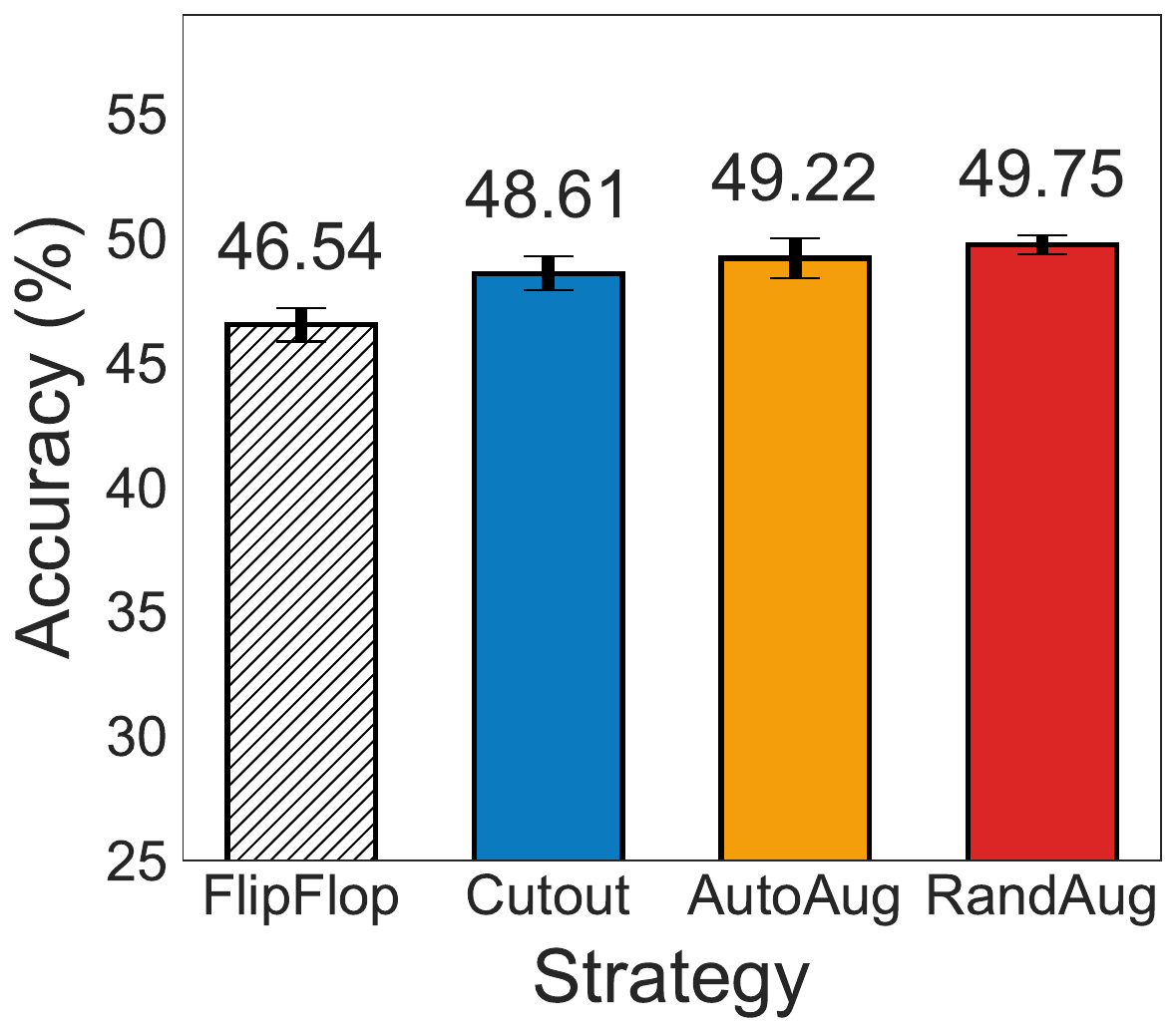}
    \caption{CIFAR-20}
    \label{fig:aug2}
  \end{subfigure}
  \begin{subfigure}[b]{0.24\textwidth}
    \includegraphics[width=\textwidth]{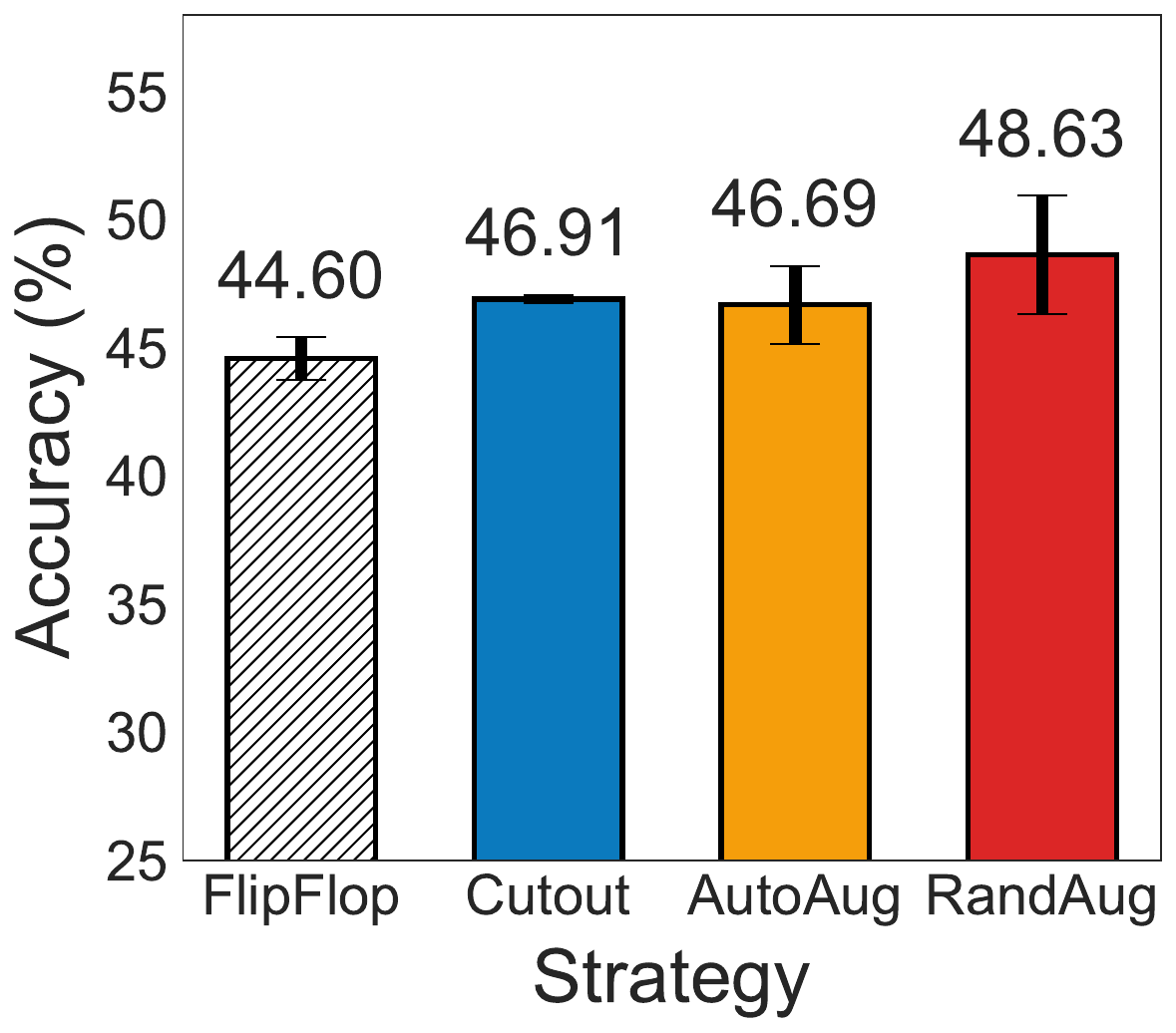}
    \caption{CIFAR-100}
    \label{fig:aug3}
  \end{subfigure}
  \begin{subfigure}[b]{0.24\textwidth}
    \includegraphics[width=\textwidth]{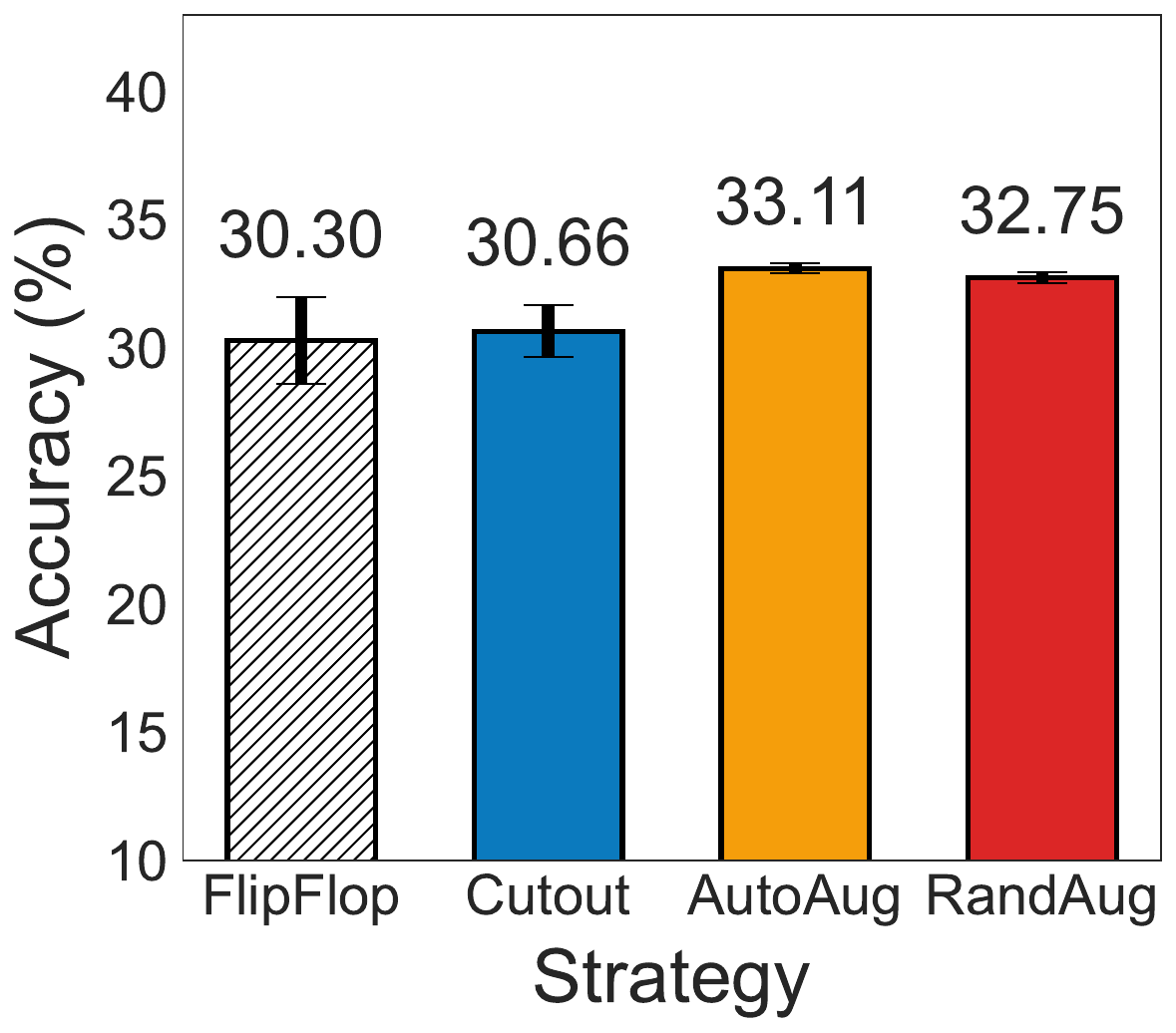}
    \caption{TinyImageNet-200}
    \label{fig:aug4}
  \end{subfigure}  \caption{Performance comparison of different data augmentation strategies integrated with BICL when training on FWD algorithm.}
\label{fig:augmentation}
\end{figure*}

As illustrated in Figure~\ref{fig:augmentation} demonstrates a clear performance improvement of BICL across datasets by incorporating data augmentations compared to the baseline simple Flipflop. Cutout serves as an effective immediate step, performing competitively on easy dataset like CIFAR-10. However, as the complexity of the dataset increases, the gap between Cutout and the stronger augmentation technique widens. RandAug and AutoAug prove to be the most robust out of all strategies tested, constantly delivering significant boost across all datasets. 
Overall, the addition of data augmentations to the our approach yields an increase to the model's accuracy regardless of the dataset intricacy.

\subsection{Effect of Encoder Choice on BICL Performance}
\label{effect_encoder}

In the CLL setting, ordinary labels are unavailable or costly to obtain. As a result, the clustering induced by the encoder plays a critical role in grouping semantically similar images into collections waiting for VLM annotation.

\begin{figure}[ht]
  \centering
  \begin{subfigure}[htb]{0.24\textwidth}
    \includegraphics[width=\textwidth]{Sections/Figures/bclcifar10_pie.pdf}
    \caption{SimSiam}
  \end{subfigure}
  \begin{subfigure}[htb]{0.24\textwidth}
    \includegraphics[width=\textwidth]{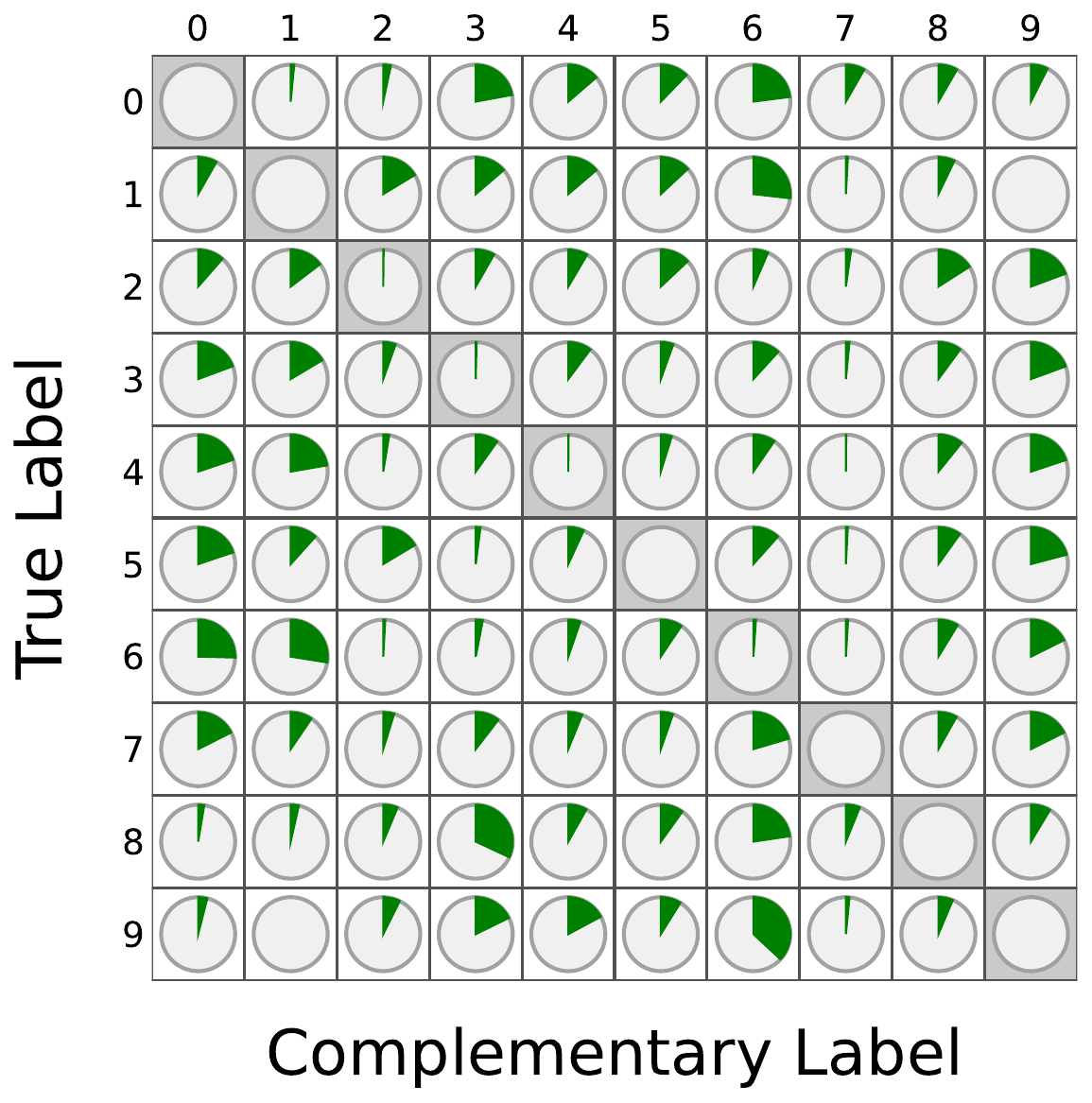}
    \caption{SimCLR}
  \end{subfigure}
  \begin{subfigure}[htb]{0.24\textwidth}
    \includegraphics[width=\textwidth]{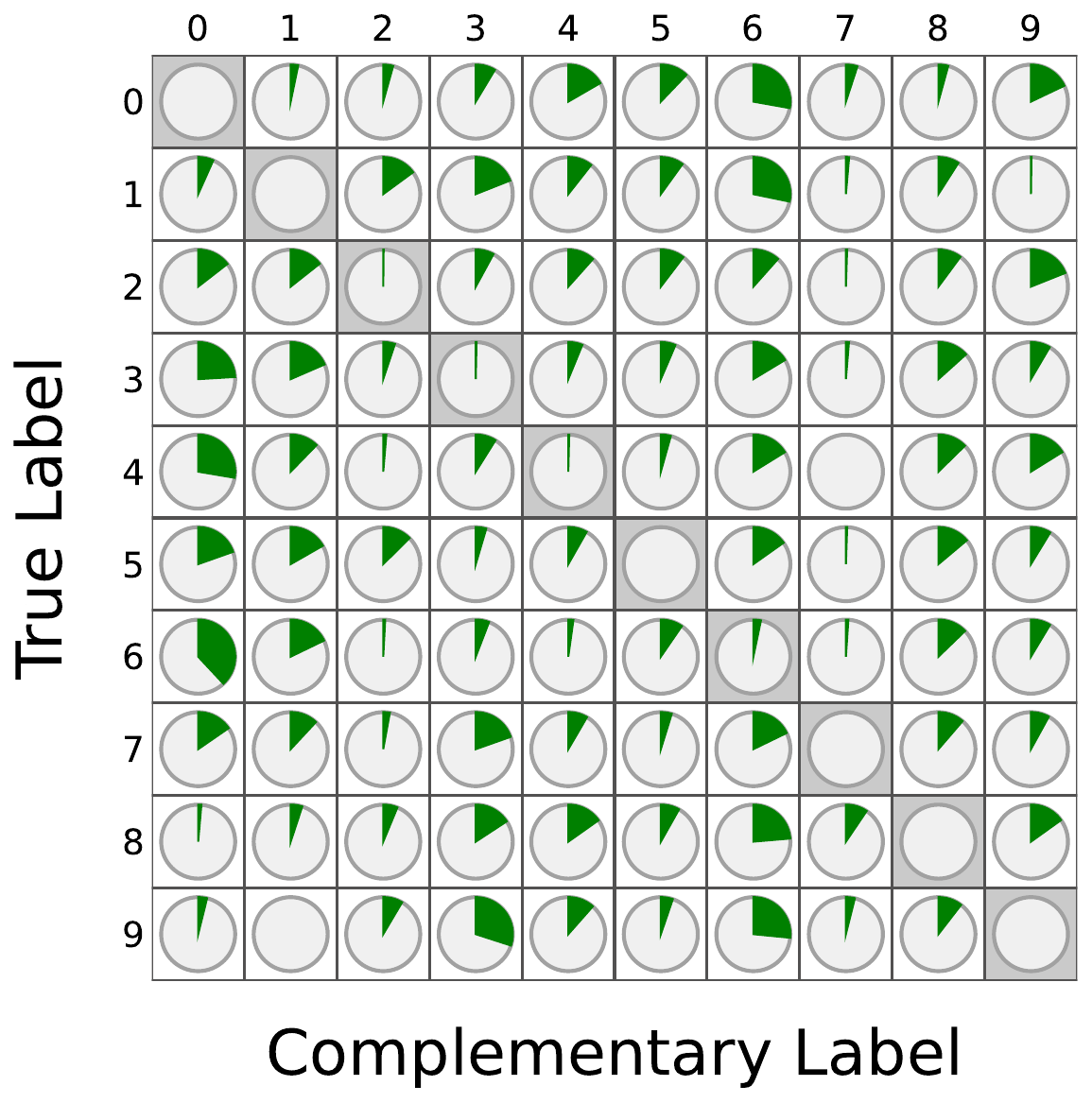}
    \caption{BYOL}
  \end{subfigure}
  \begin{subfigure}[htb]{0.24\textwidth}
    \includegraphics[width=\textwidth]{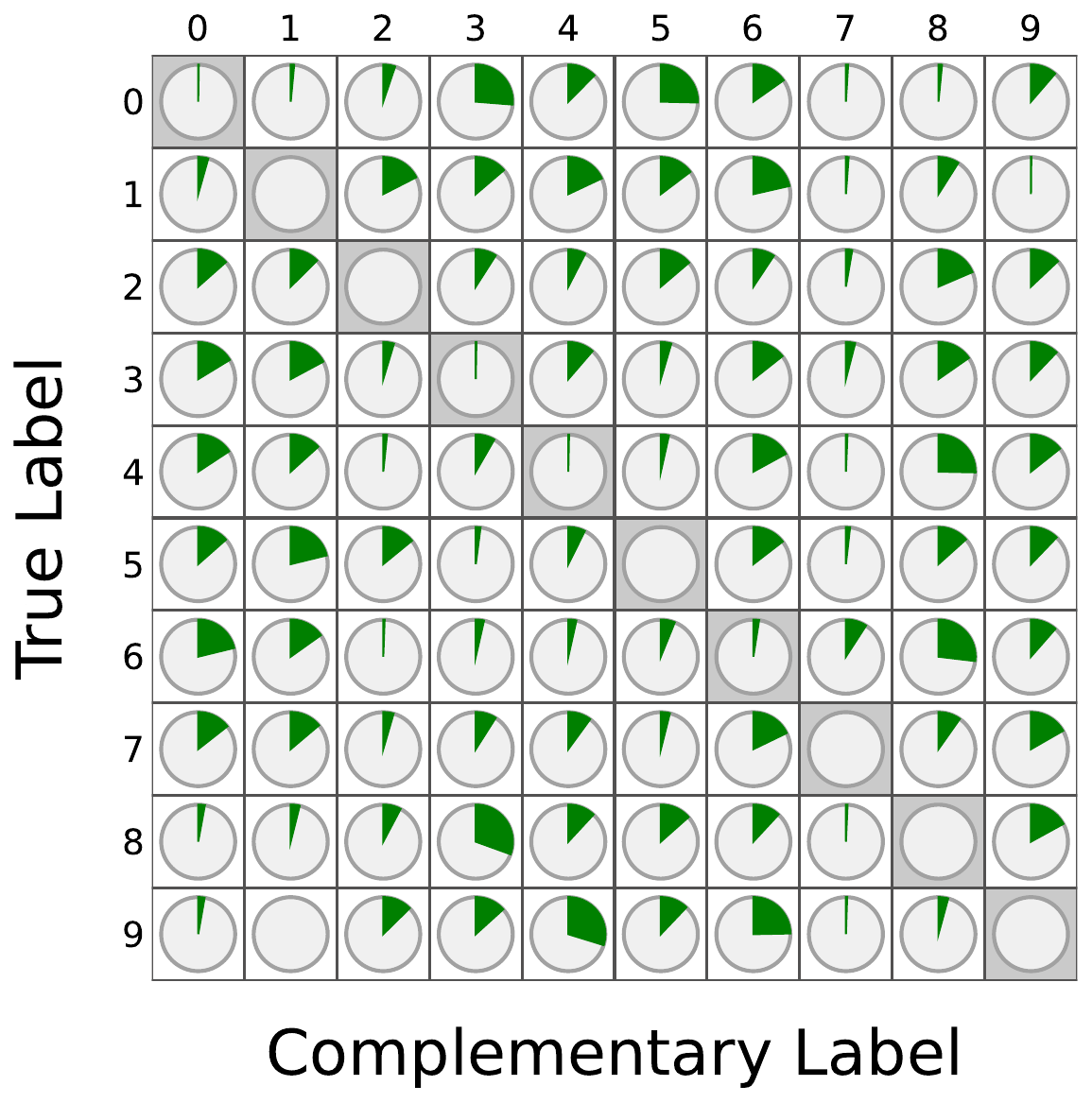}
    \caption{MoCov3}
  \end{subfigure}
  \caption{Transition matrix induced by the BICL protocol when using different encoder backbones for label selection.}
  \label{fig:transition-matrix-encoder}
\end{figure}

To investigate this, we conducted an ablation study comparing SimSiam~\cite{chen2020simsiam} with other self-supervised encoders, including SimCLR~\cite{simclr}, BYOL~\cite{byol} and MoCov3~\cite{mocov3}. All additional pretrain models were trained with a ResNet18 backbone on CIFAR10 for 800 epochs, using the same hyperparameters as for SimSiam.


\begin{table}[ht]
\centering
\caption{Analysis of conditional entropy and mutual information across different encoder networks used for clustering in BICL on CIFAR-10 dataset.}
\label{tab:encoder_analysis}
\vspace{5pt}
\begin{tabular}{l|ccc}
\toprule
\textbf{Clustering Encoder} & Noise Ratio of CLs & $H(Y \mid \bar{Y})$ & $I(Y; \bar{Y})$ \\
\midrule
SimSiam & 0.23\% & \textbf{1.5720} & \textbf{1.2778} \\
SimCLR  & 0.21\% & 2.7835 & 0.4244 \\
BYOL    & 0.44\% & 2.7906 & 0.3908 \\
MoCov3  & 0.36\% & 2.8091 & 0.4278 \\
\bottomrule
\end{tabular}
\begin{tablenotes}
\item \textbf{Note}: In our setting, a better clustering encoder should yield lower conditional entropy $H(Y \mid \bar{Y})$ and higher mutual information $I(Y; \bar{Y})$. Among the compared encoders, SimSiam achieves the lowest conditional entropy and the highest mutual information.
\end{tablenotes}
\end{table}

\begin{figure*}[htb]
\centering
\includegraphics[width=0.8\textwidth]{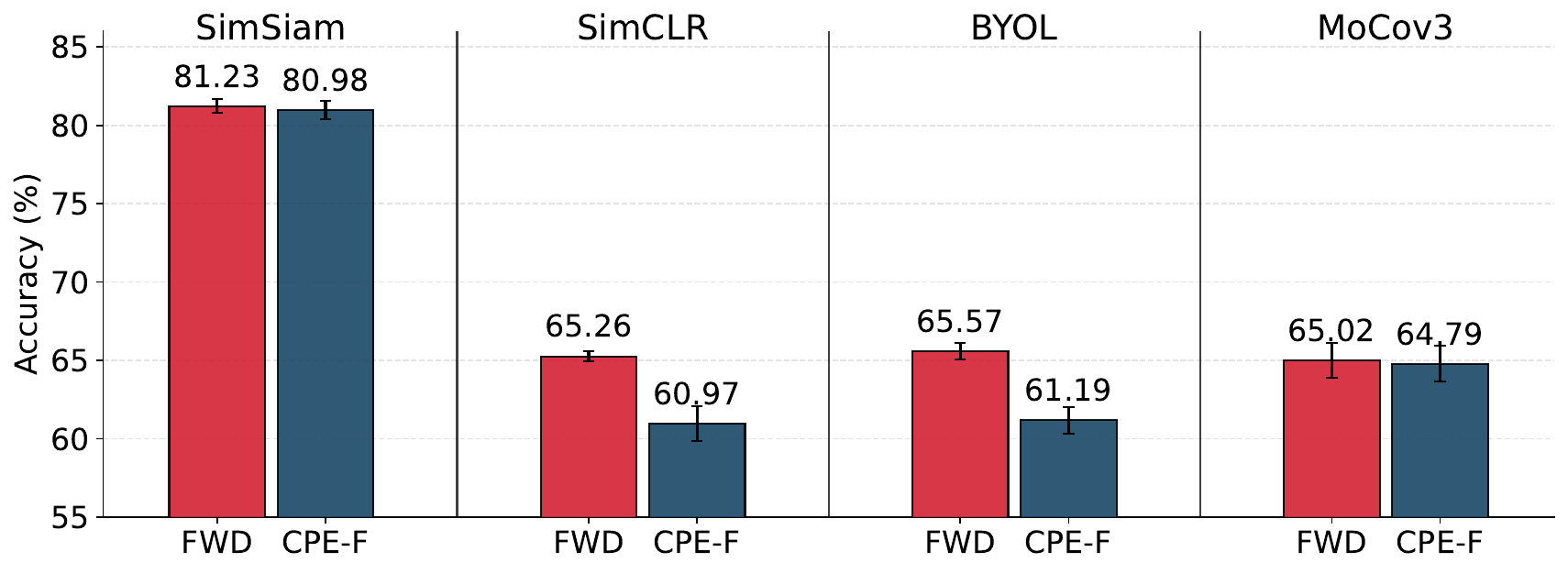}
\caption{BICL performance with different encoder networks used in the label-selection stage, evaluated with FWD and CPE-F.}
  \label{fig:encoder_networks}
\end{figure*}

The efficacy of BICL depends on the encoder's ability to group semantically similar images into coherent clusters. In our proposed method, a highly biased transition matrix is much more desirable. Ideally, all sample of the same class should fall into a single cluster assigned with specific random CLs. In Figure \ref{fig:transition-matrix-encoder} and \ref{fig:encoder_networks}, one can make the correlation between a biased matrix with concentrated probability mass and improved model performance. A spread-out distribution in the transition matrix indicates that the encoder has wrongly scattered instances of the same class across multiple clusters. We can observe that stronger encoders produce significantly less sparse matrices, therefore reducing ambiguity and enhancing the robustness of the model.

In our ablation study, the SimSiam encoder network produced exceptional transition matrices - its ability to cluster images is superior compared to others. Specifically, SimSiam induces the strongest bias and also achieves the best performance in Figure~\ref{fig:encoder_networks}, suggesting that bias is the main factor behind the performance differences. Consistently, stronger bias is also associated with higher mutual information and lower entropy. Our further analysis can be found in Table~\ref{tab:encoder_analysis}.
 
The resulting transition matrix of SimSiam exhibits probability concentrations, with true labels mapping almost deterministically to specific complementary label.

\subsection{Effect of VLM Choice on BICL Performance}
\label{vlm_choice}
A core component of our data collection protocol is the leveraging of systematic biases inherent to VLM-based annotation. Different model architectures or pre-training objectives may produce varying degrees of ``preference'' for certain negative labels, which in turn shapes the sparsity and informativeness of the resulting transition matrix.

To evaluate the sensitivity of the BICL framework to the underlying VLM, we conduct an ablation study across a diverse suite of contemporary architectures. While our primary experiments utilize LLaVA-v1.6-7B~\cite{LLava}, we extend this analysis to include several frontier open-weight families including: Qwen3~\cite{yang2025qwen3technicalreport} (with various parameters version) and Gemma4-E4B~\cite{gemmateam2024gemmaopenmodelsbased}. Every setting for this extended experiment is kept the same as in Section \ref{sec:6.1}, except that the VLM is replaced with alternative models.

\begin{table*}[htb]
\centering
\caption{Accuracy (\%) comparison between the uniform baseline and different VLM-based complementary-label generation strategies. Best performance for each dataset--method pair is in bold.}
\label{tab:vlm_compare_methods}
\small
\setlength{\tabcolsep}{3pt}
\renewcommand{\arraystretch}{0.9}

\begin{tabular}{l|l|l|ccccc}
\toprule
\textbf{Method} & \textbf{Dataset} 
& \textbf{Uniform} & \textbf{LLaVA-7B} & \textbf{Qwen2-7B} & \textbf{Qwen3-4B} & \textbf{Qwen3-8B} & \textbf{Gemma4-E4B} \\
\midrule
\multirow{4}{*}{\textbf{FWD}}
 & CIFAR-10  & 64.99\tiny{$\pm$0.64}&\textbf{81.23}{\tiny $\pm$0.05} & 79.45\tiny{$\pm$0.82} & 80.65\tiny{$\pm$0.53} & 81.00\tiny{$\pm$0.39} & 78.67\tiny{$\pm$0.72} \\
 & CIFAR-20  & 20.50\tiny{$\pm$0.75}&\textbf{50.84}{\tiny $\pm$0.32} & 48.09\tiny{$\pm$0.93} & 49.58\tiny{$\pm$0.91} & 51.10\tiny{$\pm$0.41} & 51.06\tiny{$\pm$0.14} \\
 & CIFAR-100 & 5.53\tiny{$\pm$0.47}&\textbf{46.70}{\tiny $\pm$0.60} & 38.26\tiny{$\pm$1.40} & 36.15\tiny{$\pm$2.69} & 38.92\tiny{$\pm$0.34} & 39.13\tiny{$\pm$0.58} \\
 & TinyImageNet-200 & 4.00\tiny{$\pm$0.43}&\textbf{32.15}{\tiny $\pm$0.30} & 32.30\tiny{$\pm$1.56} & 33.50\tiny{$\pm$0.29} & 34.24\tiny{$\pm$0.25} & 34.06\tiny{$\pm$1.84} \\
\midrule
\multirow{4}{*}{\textbf{CPE-F}}
 & CIFAR-10  & 64.97\tiny{$\pm$0.61}&\textbf{80.98}{\tiny $\pm$0.18} & 79.57\tiny{$\pm$0.86} & 80.12\tiny{$\pm$0.45} & 80.68\tiny{$\pm$0.37} & 77.61\tiny{$\pm$0.69} \\
 & CIFAR-20  & 20.73\tiny{$\pm$1.17}&\textbf{50.89}{\tiny $\pm$0.30} & 47.46\tiny{$\pm$1.29} & 48.36\tiny{$\pm$0.86} & 49.96\tiny{$\pm$0.25} & 49.62\tiny{$\pm$0.24} \\
 & CIFAR-100 & 5.06\tiny{$\pm$0.76}&\textbf{46.57}{\tiny $\pm$0.09} & 37.81\tiny{$\pm$1.77} & 36.01\tiny{$\pm$2.73} & 38.63\tiny{$\pm$0.40} & 39.04\tiny{$\pm$0.46} \\
 & TinyImageNet-200 & 2.14\tiny{$\pm$0.87}&\textbf{31.89}{\tiny $\pm$0.07} & 30.41\tiny{$\pm$1.72} & 33.37\tiny{$\pm$0.34} & 33.83\tiny{$\pm$0.13} & 33.68\tiny{$\pm$2.05} \\
\bottomrule
\end{tabular}
\begin{tablenotes}
    \item \textit{\textbf{Note}:} 
    \textbf{Uniform} denotes the non-VLM baseline, where CLs are generated under a uniform assumption~\cite{ishida2017learning}. Our proposed VLM-based settings generally achieve substantially higher accuracy than the uniform baseline.
\end{tablenotes}
\end{table*}
The results of this ablation study are reported in Table~\ref{tab:vlm_compare_methods}. 
Our designed method consistently outperforms the uniform baseline across all evaluated VLM models, datasets, and learning methods, suggesting that the proposed strategy is robust to different VLM backbones.
The model used in our main experiments (LLaVA-v1.6-7B) serves as a surprisingly dominant annotator for CIFAR-based datasets, achieving the highest accuracy in nearly all CIFAR-10, 20, and 100 trials. This suggests that its visual-linguistic alignment is particularly well-suited for the semantic granularity of these specific label spaces. A critical finding occurs in the TinyImageNet-200 results. While LLaVA performs well, it is outperformed by Qwen3-8B and Gemma4-E4B. Newer architectures demonstrate more advanced reasoning capabilities, enabling them to better leverage the rich features present in high-resolution TinyImageNet images.

These findings confirm that BICL is not only robust across various VLMs but actually scales in effectiveness as the underlying annotator models gain the ability to process more detailed visual information. This underscores the potential for BICL to continue improving as visual encoder technology advances toward even higher-definition image reasoning.

\subsection{BICL Performance on Additional Negative Prompt}
\label{additional_negative_prompt}
The success of the BICL framework relies on the VLM-based complementary-label annotation phase, where the model is tasked with identifying an incorrect class from a constrained candidate set. As prompting is a critical component of our data collection protocol, we conduct an ablation study to investigate the fragility of prompt engineering and its impact on the resulting downstream classifier performance.

While the base negative prompt used in our main experiments specifically instructs the VLM to ``answer the question with a single word'' (as detailed in Section~\ref{sec:3.1}), we observe that even semantically similar prompts can lead to variations in the transition matrix $Q$, which leads to variations in final classifier performance. For this experiment, we keep all settings identical to those in Section \ref{sec:6.1}, but replace the prompt with an alternative formulation in Appendix \ref{neg_prompt2}. Notably, the VLM annotation phase is conducted using LLaVA-v1.6-7B.

\scalebox{1.0}{
\begin{tcolorbox}[colback=white, colframe=gray, title=\texttt{\textbf{New Prompt}}]
\small
\texttt{\textbf{Question:}<image> Which label does not belong to this image? (1) {labels[0]} (2) {labels[1]} (3) {labels[2]} (4) {labels[3]} Please respond with only the number of the correct answer.}
\label{neg_prompt2}
\end{tcolorbox}
}

\scalebox{1.0}{
\begin{tcolorbox}[colback=white, colframe=gray, title=\texttt{\textbf{Current Prompt}}]
\small
\texttt{\textbf{Question:}<image> Which label does not belong to this image? Answer the question with a single word: \{{label[0]}, {label[1]}, {label[2]}, {label[3]}\}.}
\label{neg_prompt1}
\end{tcolorbox}
}

Instead of directly naming the complementary label for each image, this prompt requires the VLM to respond with the index corresponding to a given option. From a human perspective, both prompts describe essentially the same task, and we would expect them to yield the same answer for a given image. However, for the VLM, the transition from semantic identification to index-based mapping introduces several layers of inductive and positional bias. Unlike direct naming, which leverages the VLM’s strong visual-textual alignment, the multiple-choice format requires an intermediate step of mapping a semantic class to an arbitrary number index.




\begin{figure*}[htb]
  \centering
  \includegraphics[width=0.38\linewidth]{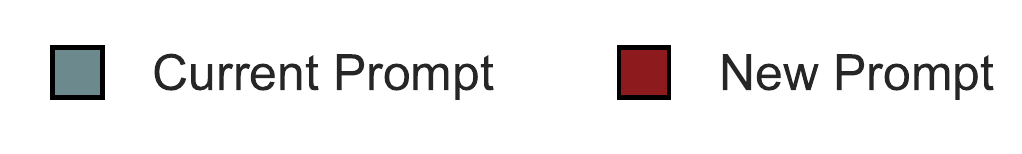}\\
  \begin{subfigure}[b]{0.24\textwidth}
    \includegraphics[width=\textwidth]{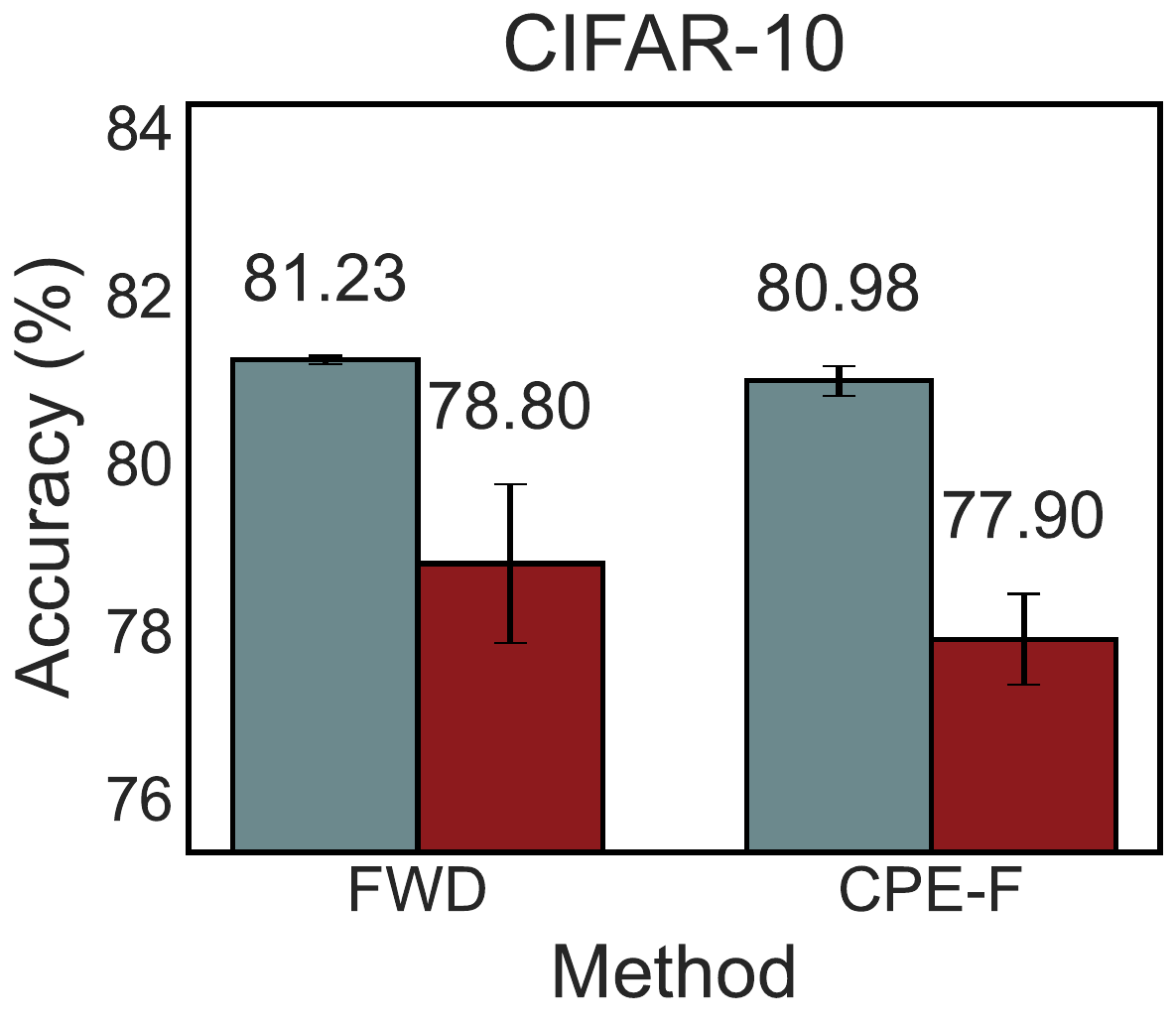}
    \caption{CIFAR-10}
  \end{subfigure}
  \begin{subfigure}[b]{0.24\textwidth}
    \includegraphics[width=\textwidth]{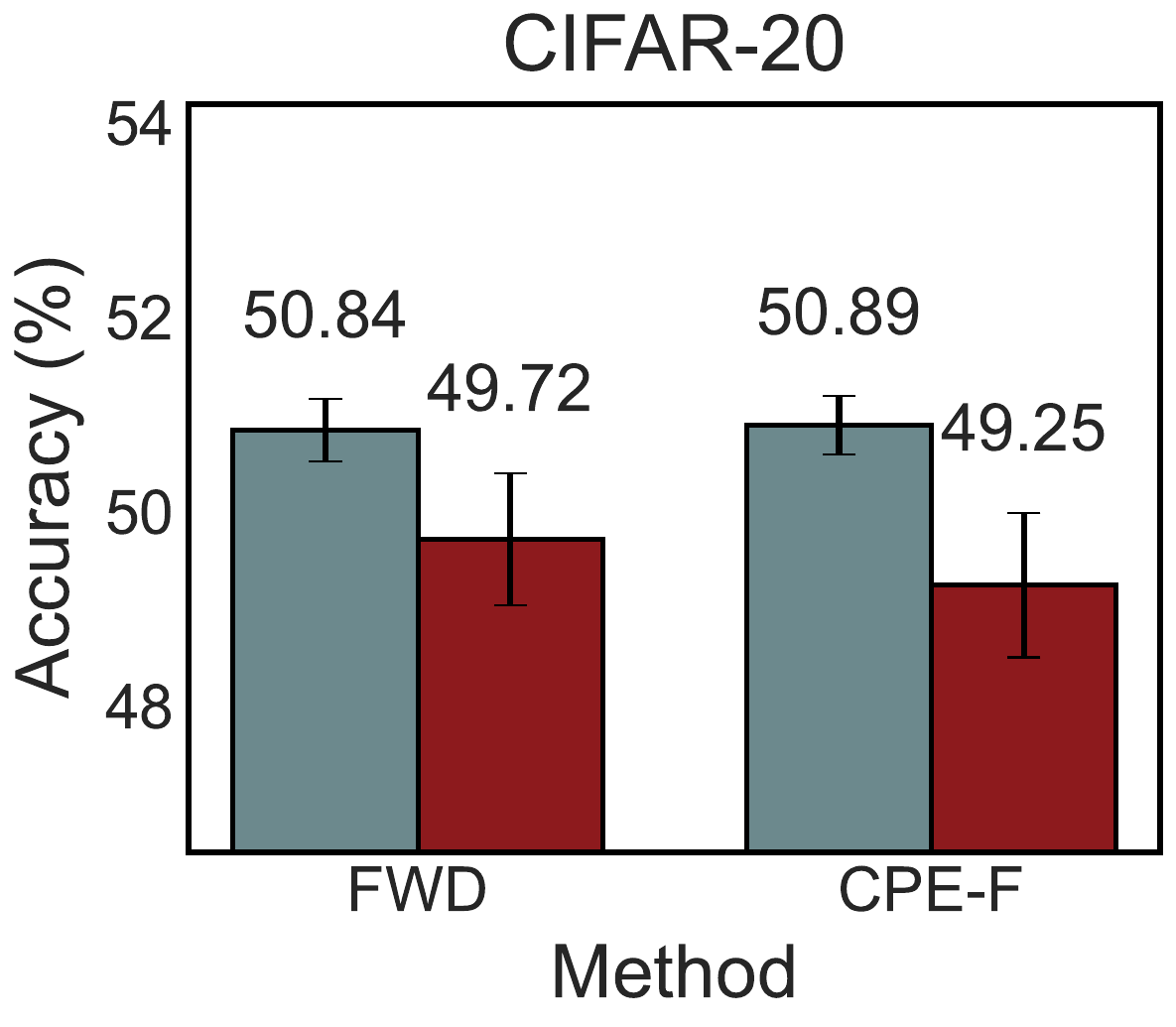}
    \caption{CIFAR-20}
  \end{subfigure}
  \begin{subfigure}[b]{0.24\textwidth}
    \includegraphics[width=\textwidth]{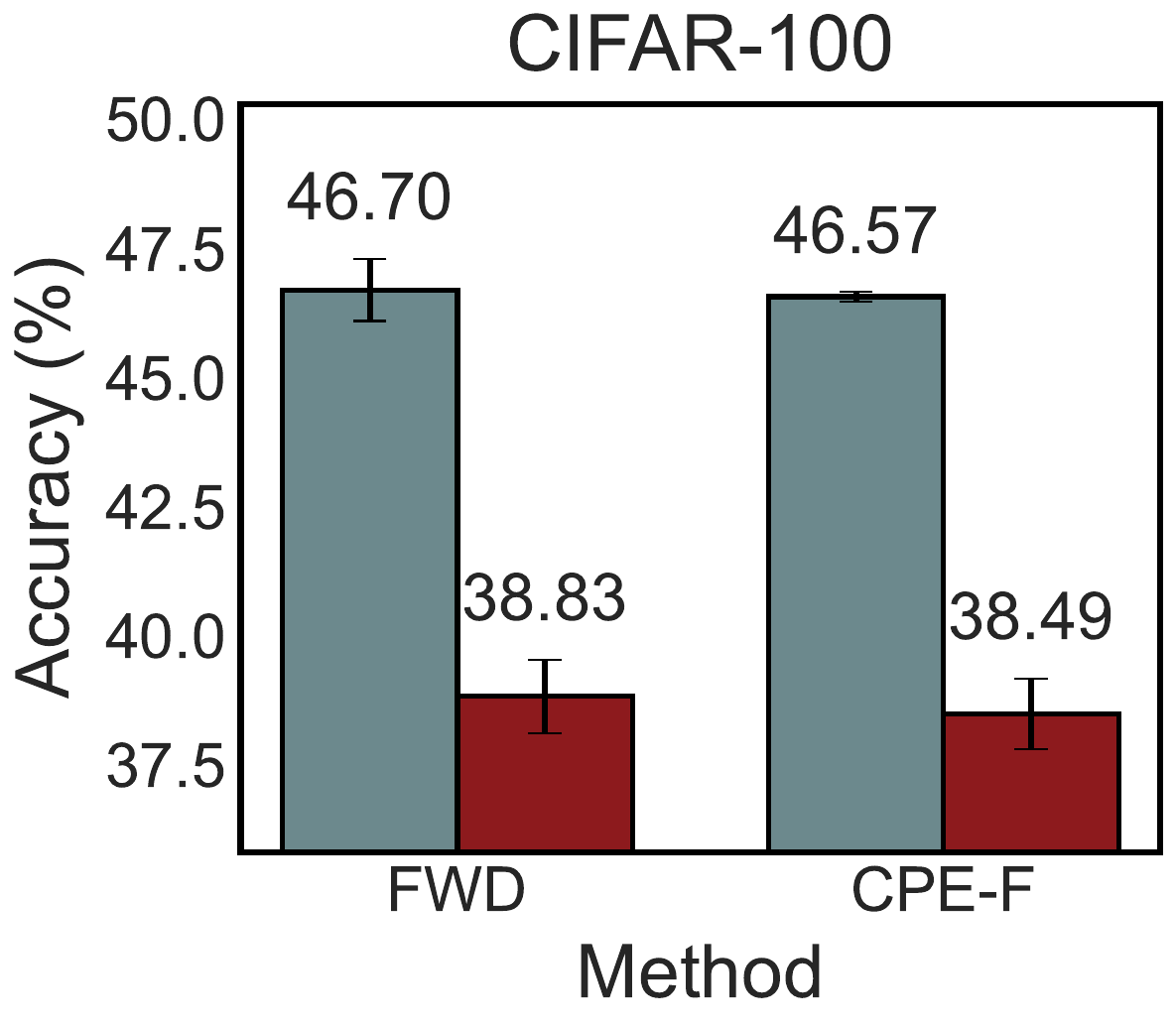}
    \caption{CIFAR-100}
  \end{subfigure}
  \begin{subfigure}[b]{0.24\textwidth}
    \includegraphics[width=\textwidth]{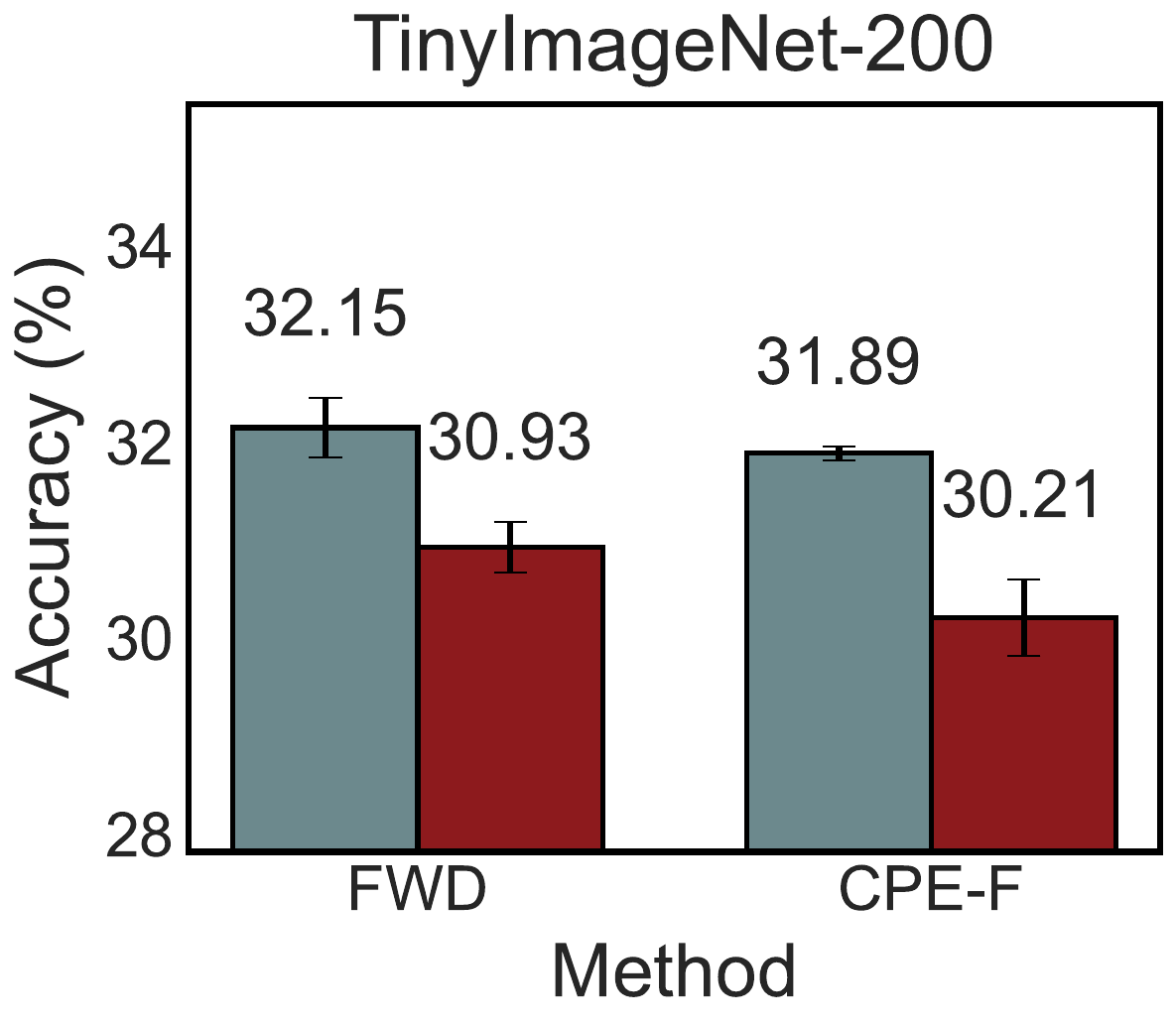}
    \caption{TinyImageNet-200}
  \end{subfigure}  \caption{Performance comparison of different prompt on model performance across datasets.}
\label{fig:prompt_design_4charts}
\end{figure*}

\begin{table*}[ht]
\centering
\caption{Impact of prompt design on noise and mutual information characteristics across datasets.}
\label{tab:prompt_comparison_multi}
\setlength{\tabcolsep}{4pt}
\renewcommand{\arraystretch}{0.9}
\resizebox{0.9\textwidth}{!}{
\begin{tabular}{l|cc|cc|cc|cc}
\toprule
\textbf{Datasets}
& \multicolumn{2}{c|}{\textbf{CIFAR-10}}
& \multicolumn{2}{c|}{\textbf{CIFAR-20}}
& \multicolumn{2}{c|}{\textbf{CIFAR-100}}
& \multicolumn{2}{c}{\textbf{TinyImageNet-200}} \\
\midrule
\textbf{Prompt}
& Current & New
& Current & New
& Current & New
& Current & New \\
\midrule
\midrule

Noise ratio
& 0.23\% & 0.64\%
& 0.80\% & 1.03\%
& 0.026\% & 0.014\%
& 0.006\% & 0.036\% \\
\midrule
$H(Y \mid \bar Y)$
& \textbf{2.1999} & 2.4604
& 3.6177 & \textbf{3.6016}
& \textbf{4.5727} & 5.5809
& \textbf{5.8390} & 5.8683 \\

$I(Y; \bar{Y})$
& \textbf{1.1219} & 0.8616
& 0.7041 & \textbf{0.7202}
& \textbf{2.0711} & 1.0630
& \textbf{1.8048} & 1.7749 \\

\bottomrule
\end{tabular}
}
\end{table*}

The choice of negative prompt significantly influences the VLM's ability to generate informative supervision, illustrating in Table~\ref{tab:prompt_comparison_multi} and Figure~\ref{fig:prompt_design_4charts} that how the labeling task is presented can result in a vast discrepancy in downstream performance. As presented in our mathematical justification, dataset which has transition matrix with higher mutual information tends to perform better in inference phase. We can see that trends most clearly here where applying ``New Prompt'' to CIFAR-100 produce a dataset with halved the mutual information metric compare to the ``Current Prompt''. This sharp reduction in supervision informativeness results in a significant 8 percentage points drop in final accuracy. Across the majority of benchmark datasets, iterations utilizing this alternative prompt induce less favorable information metrics characterized by higher conditional entropy $H$ or reduced mutual information $I$ which consistently correlates with degraded classifier performance. These findings underscore the fragility of prompt engineering and suggest that the VLM's ability to provide a high-quality supervision signal is sensitive to the linguistic framing of the task.

\subsection{BICL Performance on Recent CLL Methods}
\label{additional_cll_methods}

In addition to the standard CLL algorithms for biased designed, we conduct additional experiments with recent CLL baselines, including SCARCE~\cite{conu2023} and OP~\cite{orderpreserving2023}. We would like to emphasize that these methods belong to two different technical directions in the CLL literature. Methods such as SCL-NL, SCARCE, COVR~\cite{covr_2024}, and OP are primarily developed for the standard CLL setting, where the complementary-label transition is assumed to be relatively uniform or unbiased. In contrast, methods such as FWD, CPE, URE-TNN, and URE-TGA are designed to model or handle biased transition structures.

Since BICL explicitly induces an extremely biased complementary-label distribution, we consider the second group to be more appropriate for the main comparison. For completeness, we also evaluated recent methods such as SCARCE and OP, although their assumptions are less well aligned with the setting induced by BICL, as shown in Table~\ref{tab:recent_CLL}.

\begin{table}[ht]
\centering
\caption{Performance comparison of recent CLL methods under uniform and BICL distributions on CIFAR-20 and CIFAR-100 (accuracy (\%), mean$\pm$std).}
\vspace{5pt}
\label{tab:recent_CLL}
\resizebox{0.65\textwidth}{!}{
\small
\setlength{\tabcolsep}{2.5pt}
\begin{tabular}{l|cc|cc}
\toprule
\textbf{Datasets}
& \multicolumn{2}{c|}{\textbf{CIFAR-20}}
& \multicolumn{2}{c}{\textbf{CIFAR-100}} \\
\midrule
\textbf{Method}
& Uniform (\%) & BICL (\%)
& Uniform (\%) & BICL (\%) \\
\midrule
SCARCE
& 17.53\scriptsize{$\pm$0.49} & 5.02\scriptsize{$\pm$0.10}
& 1.41\scriptsize{$\pm$0.07} & 1.01\scriptsize{$\pm$0.02} \\

OP
& 20.71\scriptsize{$\pm$1.04} & 5.48\scriptsize{$\pm$0.43}
& 1.21\scriptsize{$\pm$0.18} & 1.02\scriptsize{$\pm$0.05} \\

\midrule
FWD
& 20.50\scriptsize{$\pm$0.75} & {50.84}\scriptsize{$\pm$0.32}
& 5.53\scriptsize{$\pm$0.47} & {46.70}\scriptsize{$\pm$0.60} \\

CPE-F
& 20.73\scriptsize{$\pm$1.17} & {50.89}\scriptsize{$\pm$0.30}
& 5.06\scriptsize{$\pm$0.76} & {46.57}\scriptsize{$\pm$0.09} \\
\bottomrule
\end{tabular}
}
\end{table}

Table \ref{tab:recent_CLL} shows that recent CLL methods such as SCARCE and OP perform poorly under the BICL distribution, suggesting limited robustness to highly biased complementary-label collection. On CIFAR-20, their accuracy drops substantially from 17.53\% and 20.71\% under the uniform setting to 5.02\% and 5.48\%, respectively. A similar trend is observed on CIFAR-100, where both methods remain near chance level under BICL. In contrast, FWD and CPE-F benefit markedly from the BICL distribution, achieving large improvements over the uniform setting on both datasets. These results indicate that extremely biased complementary-label distributions can severely hinder some recent CLL methods, while our setting remains effective for more robust baselines such as FWD and CPE-F.

\subsection{Comparison with Approaches Using Few True Labels}
\label{compared_few_true_labels}

In CLL, true labels are not accessible during training. However, several \emph{transition-aware} ($Q$-aware) algorithms~\cite{fwd2018,cpe2023} assume that the complementary-label transition matrix $Q$ is known or can be reliably estimated. In practice, this assumption is often violated. When learning solely from CLs, information about the true label distribution is largely lost, making it fundamentally difficult to estimate $Q$ without additional supervision, as $Q$ is closely tied to the class priors of true labels.

On the other hand, while annotating large-scale datasets with full supervision is costly, it is often feasible in practice to obtain a small number of confidently labeled examples per class. Motivated by this observation, we consider a setting in which five true labels per class are manually provided and used to estimate the transition matrix. In this sense, $Q$-aware CLL methods implicitly rely on a small amount of true-label supervision.

To ensure a fair comparison, we conduct an ablation study in which five true labels per class are used to train a classifier under standard supervision, while the remaining data are treated as complementarily labeled. Both the true-label and complementary-label distributions are generated uniformly and without noise. Following~\cite{ishida2017learning}, we combine the losses from ordinary labels and CLs as
\begin{equation}
\label{combine_ol_cl}
\hat{R}(f)
= \underbrace{\frac{\alpha}{m}\sum_{j=1}^{m} \mathcal{L}\!\big(f(\mathbf{x}_j),y_j\big)}_{CE}
+ \underbrace{\frac{(1-\alpha)(K-1)}{n}\sum_{i=1}^{n} \mathcal{L}\left(Q^\top \sigma(g(\mathbf{x}_i)), \bar{y}_i\right)}_{FWD},
\nonumber
\end{equation}
where $\{(\mathbf{x}_j,y_j)\}_{j=1}^{m}$ denotes the ordinarily labeled data and $\{(\mathbf{x}_i,\bar{y}_i)\}_{i=1}^{n}$ denotes the complementarily labeled data,  $\alpha \in [0,1]$ is a hyper-parameter that interpolates between the two risks, we set $\alpha = 0.5$ in this experiment. $Q$ is estimated transition matrix, and $g(\mathbf{x})$ denote the classifier outputs (logits) and $\sigma(\cdot)$ be the softmax function, $\mathcal{L}(\cdot, \cdot)$ is a standard loss function (i.e., Cross-Entropy).

\begin{table}[ht]
\centering
\caption{Performance comparison with methods using a small number of true labels per class.}
\label{tab:true_label_ablation}
\vspace{5pt}
\small
\setlength{\tabcolsep}{6pt}
\begin{tabular}{l|l|c|c|c|c}
\toprule
\textbf{} & \textbf{Algorithm} & \textbf{CIFAR-10} & \textbf{CIFAR-20} & \textbf{CIFAR-100} & \textbf{TinyImageNet-200} \\
\midrule
5 (True Labels) & CE + FWD & 77.93 & 45.99 & 44.52 & 27.18\\
\midrule
\textbf{BICL} (Ours)  & FWD & \textbf{81.23} & \textbf{50.84} & \textbf{46.70} & \textbf{32.15}\\
\midrule  
& $\Delta$ & \textcolor{blue}{$\uparrow$ 3.30} & \textcolor{blue}{$\uparrow$ 4.85} & \textcolor{blue}{$\uparrow$ 2.18} & \textcolor{blue}{$\uparrow$ 4.97}\\

\bottomrule
\end{tabular}
\begin{tablenotes}
\item \textbf{\textit{CE+FWD}}: In this setup, we train on the small set of true-labeled samples (five per class) using cross-entropy (CE), and on the remaining samples using FWD. The final objective is the sum of the CE and FWD losses. Notably, the CLs used for FWD in this baseline are synthetically generated under the noiseless, uniform assumption of~\citep{ishida2017learning}. In contrast, FWD in BICL (ours) is trained on biased CLs collected using our BICL protocol.
\end{tablenotes}
\end{table}

We use ResNet18 as the backbone network and adopt the same hyperparameter settings described in Section~\ref{sec:6.1}. The results in Table~\ref{tab:true_label_ablation} show that BICL consistently outperforms the combined-loss baseline that leverages five true labels per class across all datasets. These findings indicate that our approach is more effective than relying on a small amount of true-label supervision combined with massive CLs, highlighting the advantage of BICL in settings where limited true labels are available.

\subsection{Compared to True Label Collection Approach with VLM}
\label{sec:compare_true}
To better quantify the gap between collecting \emph{true labels} and \emph{complementary labels} with a VLM, we conduct an ablation study in which we replace complementary-label collection with true-label collection using the same VLM. We then train a classifier under \emph{standard supervision} using the cross-entropy loss and compare its performance against CLL under our BICL protocol. 
Across datasets, BICL outperforms this standard-supervision baseline, where the model is trained directly on VLM-collected true labels. We attribute this gap to the fact that, in large label spaces, VLMs often struggle to precisely identify the correct class among many visually similar alternatives. In contrast, BICL simplifies the VLM's decision process by framing annotation as a constrained \emph{rejection} task and restricting the candidate label space through our data collection protocol. 

This advantage yields large gains on challenging multi-class benchmarks: compared with standard supervision, BICL improves accuracy by 17.65 percentage points on CIFAR-20, 10.11 percentage points on CIFAR-100, and nearly doubles accuracy on TinyImageNet-200 (see the \emph{Std-Supervision} row in Tables~\ref{tab:2} and~\ref{tab:3}). Overall, these results suggest that in scalable settings, determining what an image is \emph{not} can be more reliable for VLMs than determining what it \emph{is}.

\subsection{Partial-Label Learning (PLL) vs. CLL under BICL}
\label{compared_pll_cll_methods}

While CLL is technically a subclass of PLL, but it is the most difficult problem of PLL. To examine this connection empirically, we conducted additional experiments with representative PLL methods, including PiCO~\cite{wang2022pico}, PRODEN~\cite{lv2020progressive}, and SoLar~\cite{wang2022solar}, on the BICL datasets.

In these experiments, we converted each single complementary label provided by the VLM into a partial-label set by treating all remaining labels in the label space as candidate labels. Thus, each instance in BCLCIFAR-20 is associated with a partial-label set containing 19 candidate labels, and each instance in BCLCIFAR-100 is associated with a partial-label set containing 99 candidate labels. We trained each PLL method for 800 epochs with learning rate $10^{-3}$, weight decay $10^{-4}$, and momentum 0.9. The corresponding results are reported in Table~\ref{tab:pll_cll_bicl}.

\begin{table}[ht]
\centering
\caption{Comparison of PLL and CLL methods under BICL on BCLCIFAR-20 and BCLCIFAR-100.}
\label{tab:pll_cll_bicl}
\vspace{5pt}
\resizebox{0.9\textwidth}{!}{
\small
\setlength{\tabcolsep}{5pt}
\begin{tabular}{l|l|l|c|c}
\toprule
\textbf{Setting} & \textbf{Method} & \textbf{Dataset} & \textbf{Label Set Size} & \textbf{Accuracy (\%)} \\
\midrule
\multirow{6}{*}{PLL Algorithms}
& PiCO   & BCLCIFAR-20  & 19 & 6.01 \\
& PRODEN & BCLCIFAR-20  & 19 & 7.74 \\
& SoLar  & BCLCIFAR-20  & 19 & 4.74 \\
\cmidrule(lr){2-5}
& PiCO   & BCLCIFAR-100 & 99 & 1.15 \\
& PRODEN & BCLCIFAR-100 & 99 & 1.01 \\
& SoLar  & BCLCIFAR-100 & 99 & 1.10 \\
\midrule
\multirow{4}{*}{CLL Algorithms}
& FWD    & BCLCIFAR-20  & 1 & 50.84 \\
& CPE-F  & BCLCIFAR-20  & 1 & 50.89 \\
\cmidrule(lr){2-5}
& FWD    & BCLCIFAR-100 & 1 & 46.70 \\
& CPE-F  & BCLCIFAR-100 & 1 & 46.57 \\
\bottomrule
\end{tabular}
}
\end{table}

The results show that PLL methods perform very poorly in this setting. On BCLCIFAR-20, PiCO, PRODEN, and SoLar achieve 6.01\%, 7.74\%, and 4.74\% accuracy, respectively, which is far below the performance of the CLL method FWD (50.84\%). On BCLCIFAR-100, PLL performance further drops to around 1\%, while FWD still achieves 46.70\%.

This issue becomes even more severe as the number of classes increases. For example, the conversion yields candidate label sets of size 99 for CIFAR-100 and 199 for TinyImageNet-200, making the supervision increasingly ambiguous. Thus, although multiple complementary-label learning is more closely related to PLL, the single-complementary-label setting considered here is much more difficult for PLL methods after conversion, because the candidate set contains almost the entire label space.

In addition, most PLL methods assume that candidate labels are generated in a relatively random or unbiased manner. By contrast, BICL induces a highly biased candidate-label structure. Even with large candidate sets, PLL methods can still succeed when the inclusion pattern provides a reliable signal that helps distinguish the true label from incorrect ones. That signal becomes much weaker in our setting due to the strong bias introduced by BICL, which further explains the poor performance of PLL methods.

More generally, this highlights an interesting research direction connecting CLL and PLL. To the best of our knowledge, most PLL methods have not been systematically evaluated in large-scale settings with very large candidate label sets.

\subsection{Effect of the Number of Classes on Performance}
\label{effect_num_classes}

In CLL, the number of classes plays an important role, as a larger label space generally makes the learning problem more challenging. To better understand how performance degrades as a function of label-space size under CLL, we conduct additional ablation studies to further examine the impact of class count within a fixed dataset. Specifically, we construct two subsets from CIFAR-100, comprising 20 and 50 classes, respectively. The subset preserves the original images and their associated labels, while restricting the overall label space to a smaller set of classes. The 20-class subset of CIFAR-100 is constructed by selecting one randomly chosen fine-grained class from each coarse class (\textit{i.e., the CIFAR-20 superclass grouping}). For the 50-class subset, all classes are randomly selected. To further investigate the scalability of BICL on even larger label spaces, we introduce three additional subsets derived from TinyImageNet-200, containing 50, 100, and 150 classes, respectively. Similar to the 50-class CIFAR-100 configuration, all classes in these TinyImageNet subsets are randomly selected from the original global label space. 

In this experiment, the VLM is asked to provide a complementary label by selecting from four options, under the assumption that the \textit{true label} is not included (as BICL Analysis Case). All other parameters were kept the same as in our previous experiments. As shown in Figure~\ref{fig:label-space-effect}, the results confirm that model performance consistently decreases as the number of classes increases. 


\begin{figure*}[htb]
  \centering
  \includegraphics[width=0.95\linewidth]{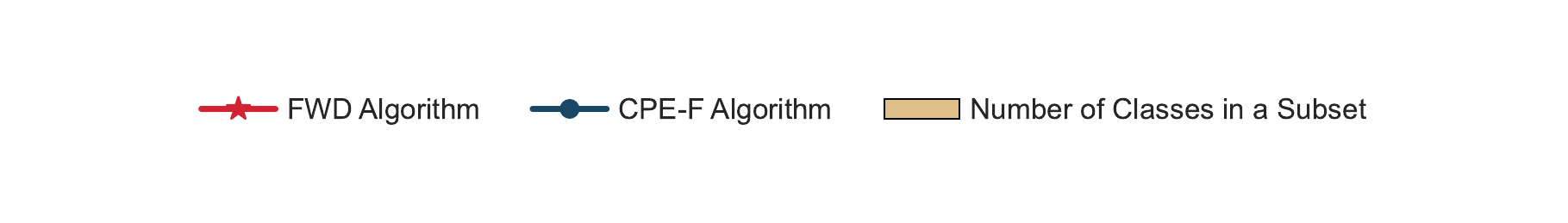}\\
  \vspace{-15pt} 
  \begin{subfigure}[b]{0.45\textwidth}
    \includegraphics[width=\textwidth]{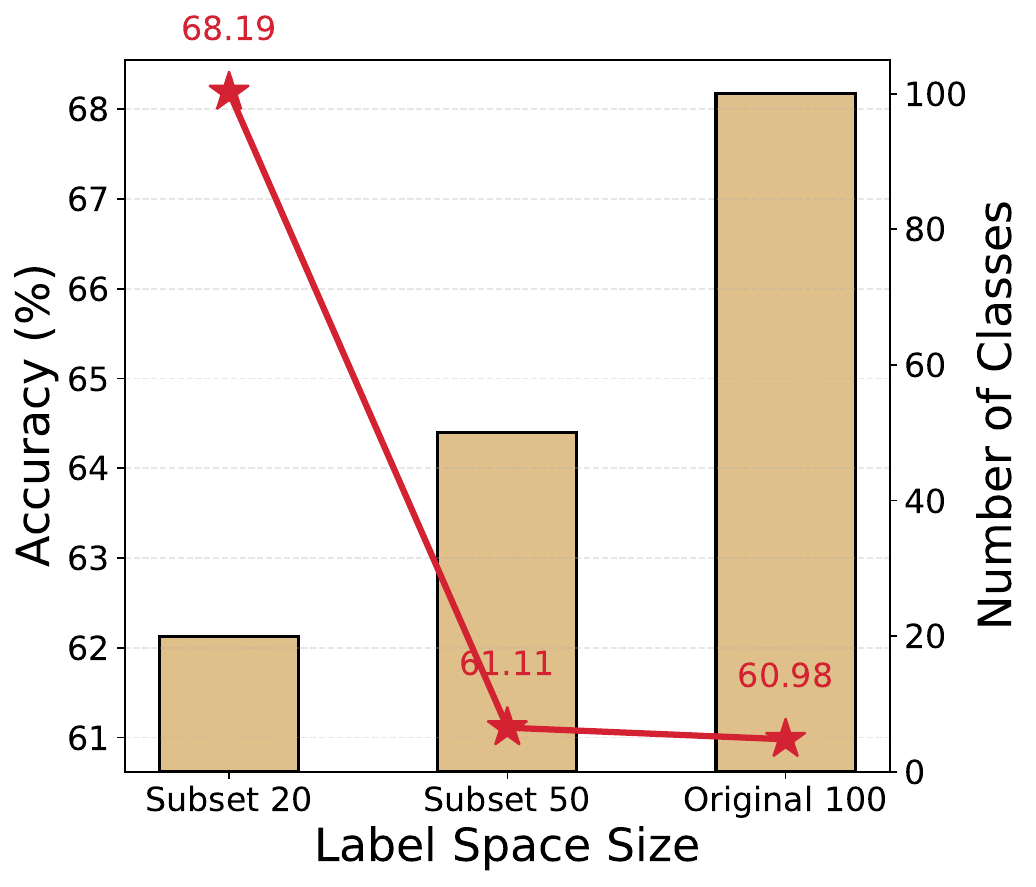}
    \caption{Subset of CIFAR-100 with FWD algorithm.}
  \end{subfigure}
  \begin{subfigure}[b]{0.45\textwidth}
  \centering
    \includegraphics[width=\textwidth]{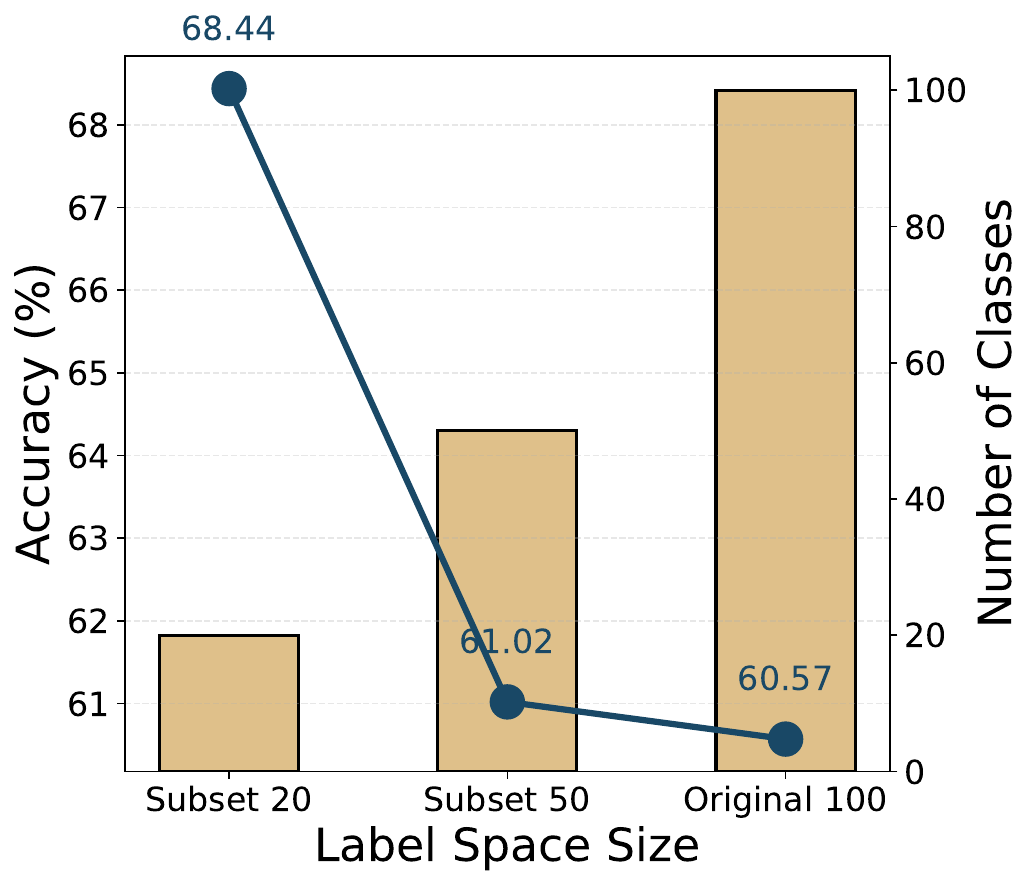}
    \caption{Subset of CIFAR-100 with CPE-F algorithm.}
  \end{subfigure}
  \begin{subfigure}[b]{0.45\textwidth}
    \includegraphics[width=\textwidth]{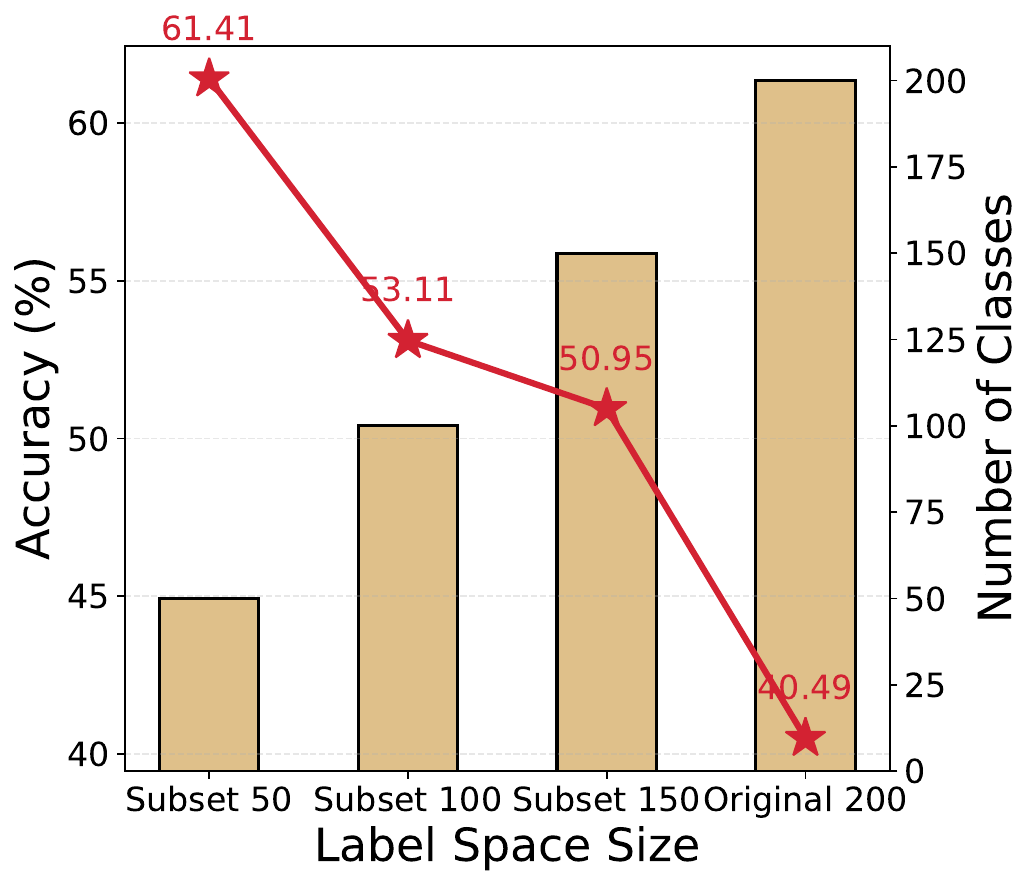}
    \caption{Subset of TinyImageNet-200 with FWD \\algorithm.}
  \end{subfigure}
  \begin{subfigure}[b]{0.45\textwidth}
  \centering
    \includegraphics[width=\textwidth]{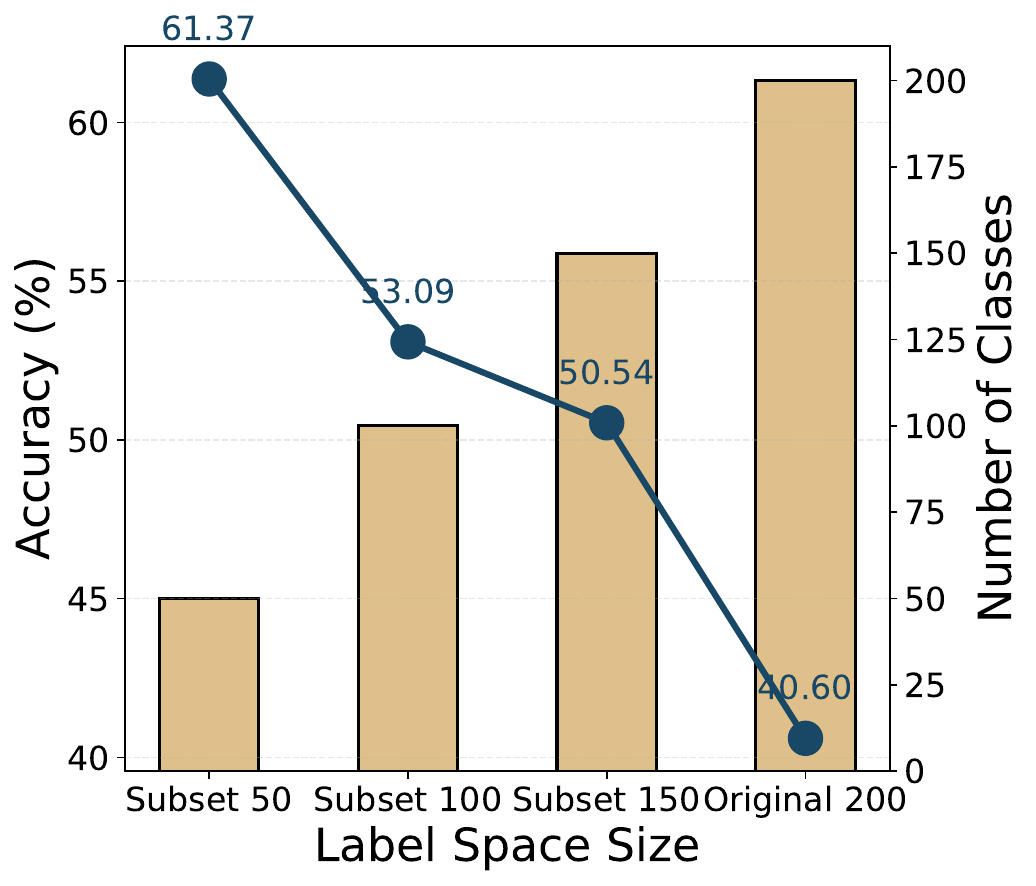}
    \caption{Subset of TinyImageNet-200 with CPE-F \\algorithm.}
  \end{subfigure}  
  \caption{Effect of label-space size on performance within CIFAR-100 and TinyImageNet-200.}
\label{fig:label-space-effect}
\end{figure*}

The most prominent trend across all four figures is the clear inverse correlation between label-space cardinality and classification accuracy. On both CIFAR-100 and TinyImageNet-200 subsets, we observe that inference model performance drops significantly as the number of classes increases, reflecting the escalating difficulty of the classification task. Notably, the FWD and CPE-F algorithms exhibit nearly identical performance profiles across all subsets, suggesting that these trends are driven by inherent data complexity. In conclusion, while the BICL framework successfully enables CLL to scale to high-cardinality environments, the mounting supervision ambiguity in larger label spaces remains a critical bottleneck for performance stability.

\section{Details of Experimental Setup}
\label{details_setup}
This section describes the experimental configuration used to evaluate the proposed method. We first detail the VLMs employed for automated annotation and the benchmark datasets used in our experiments. We then summarize the compared approaches and CLL algorithms, followed by the architectural details of the models used for training. Together, these descriptions ensure reproducibility and provide sufficient context for interpreting the empirical results.

\subsection{Details of Vision--Language Models}

For automated annotation, we use LLaVA-v1.6-7B~\cite{LLava} (Large Language-and-Vision Assistant). We choose LLaVA to enable a direct comparison with ACLImage~\cite{aclimage2025}, which also adopts this model for complementary-label collection. In particular, ACLImage reports that LLaVA offers a favorable trade-off between computational cost and labeling quality for CLL. From a technical perspective, the LLaVA architecture bridges a pre-trained vision encoder (such as CLIP ViT-L/14) and a large language model via a simple linear or MLP projection matrix. We use the 7B-parameter variant of the version 1.6, which introduces dynamic high-resolution image processing and enhanced visual reasoning, providing an optimal balance between vision reasoning capability and resource efficiency.

To further evaluate the sensitivity and robustness of our framework across diverse multimodal architectures, our extended ablation studies incorporate frontier open-weight models, specifically from the Qwen3~\cite{yang2025qwen3technicalreport} and Gemma~\cite{gemmateam2024gemmaopenmodelsbased} families. The Qwen3 architecture distinguishes itself through its native dynamic-resolution visual encoding and sophisticated visual token compression mechanism. Rather than forcing images into a fixed grid, this standout methodology allows Qwen3 to optimally preserve aspect ratios and fine-grained details at massive resolutions without overwhelming the context window. Similarly, we evaluate Gemma-based vision models (e.g., the Gemma4-E4B variant). The defining hallmark of this Gemma lineage is its deep, seamless integration of the highly optimized SigLIP (Sigmoid Loss for Language Image Pre-training) vision encoder. This architectural choice yields extreme parameter efficiency, allowing a compact model to achieve the granular spatial awareness and dense visual reasoning typically reserved for much larger architectures. Incorporating these modern models demonstrates that our data collection protocol scales effectively as underlying annotators gain specialized, cutting-edge abilities to extract rich features and process high-definition images.

All VLM queries are performed in a zero-shot setting with temperature set to 0. The VLM receives each input image at its original resolution (as downloaded). We instruct the model using a task-specific prompt derived from the base \textit{negative prompt} in Section~\ref{sec:3.1}, with minor modifications depending on the experiment. For Figure~\ref{fig:sampled_label}, we use the same base prompt while expanding the candidate label options to match the required number of labels $n$.

In the standard supervision experiments, the VLM acts as a zero-shot classifier over the global label space. The model is explicitly instructed to find the true label using a positive prompt in Appendix~\ref{pos_prompt}. Since all the images were raw, LLaVA actually struggle to answer for this task consistently and struggle frequently return back reasons such as ``\textit{This image is too blurry, I can not select the true label}''. After rigorous prompting, the final instruction enforced a forced-choice mechanism in the selection phase, eliminate uncertain outputs.

\scalebox{1.0}{
\begin{tcolorbox}[colback=white, colframe=gray, title=\texttt{\textbf{Prompt}}]
\small
\texttt{\textbf{Question:}<image> Select the closest label from [{labels}]. Output strictly one label, no explanation or reasoning. If the image is unclear, choose the most probable label.}
\label{pos_prompt}
\end{tcolorbox}
}

\subsection{Details of Benchmark Datasets}

In this section, we report the index-to-class-name mappings for the benchmark datasets used throughout our experiments. Table~\ref{cifar10-20-labelname} lists the class indices for CIFAR-10 and the corresponding superclass indices for CIFAR-20, following the standard CIFAR-100 taxonomy where CIFAR-20 groups the CIFAR-100 fine classes into 20 superclasses. Table~\ref{tab:cifar100_labels} provides the complete mapping between class indices (0--99) and fine-grained label names for CIFAR-100. 
These mappings are included to ensure clarity and reproducibility when interpreting results reported by class index (i.e., in error analysis and ablation studies) and to facilitate direct comparison across datasets with different label granularities.

\begin{table}[htb]
\centering
\caption{The correspondence between index and label names of CIFAR-10 and CIFAR-20 datasets.}
\label{cifar10-20-labelname}
\begin{tabular}{c|l|l}
\toprule
\textbf{Index} & \textbf{CIFAR-10 Label Name} & \textbf{CIFAR-20 Label Name} \\
\midrule
0  & airplane   & aquatic mammals \\
1  & automobile & fish \\
2  & bird       & flowers \\
3  & cat        & food containers \\
4  & deer       & fruit, vegetables and mushrooms \\
5  & dog        & household electrical devices \\
6  & frog       & household furniture \\
7  & horse      & insects \\
8  & ship       & large carnivores and bear \\
9  & truck      & large man-made outdoor things \\
10 &            & large natural outdoor scenes \\
11 &            & large omnivores and herbivores \\
12 &            & medium-sized mammals \\
13 &            & non-insect invertebrates \\
14 &            & people \\
15 &            & reptiles \\
16 &            & small mammals \\
17 &            & trees \\
18 &            & transportation vehicles \\
19 &            & non-transportation vehicles \\
\bottomrule
\end{tabular}
\end{table}

\begin{table}[ht]
\centering
\caption{Fine-grained class indices and names of the CIFAR-100 dataset.}
\label{tab:cifar100_labels}
\small
\setlength{\tabcolsep}{6pt}
\renewcommand{\arraystretch}{1.15}
\begin{tabular}{c|l|c|l|c|l|c|l}
\toprule
\textbf{Index} & \textbf{Label Name} &
\textbf{Index} & \textbf{Label Name} &
\textbf{Index} & \textbf{Label Name} &
\textbf{Index} & \textbf{Label Name} \\
\midrule
0  & apple            & 25 & couch          & 50 & mouse          & 75 & skunk \\
1  & aquarium\_fish   & 26 & crab           & 51 & mushroom       & 76 & skyscraper \\
2  & baby             & 27 & crocodile      & 52 & oak\_tree      & 77 & snail \\
3  & bear             & 28 & cup            & 53 & orange         & 78 & snake \\
4  & beaver           & 29 & dinosaur       & 54 & orchid         & 79 & spider \\
5  & bed              & 30 & dolphin        & 55 & otter          & 80 & squirrel \\
6  & bee              & 31 & elephant       & 56 & palm\_tree     & 81 & streetcar \\
7  & beetle           & 32 & flatfish       & 57 & pear           & 82 & sunflower \\
8  & bicycle          & 33 & forest         & 58 & pickup\_truck  & 83 & sweet\_pepper \\
9  & bottle           & 34 & fox            & 59 & pine\_tree     & 84 & table \\
10 & bowl             & 35 & girl           & 60 & plain          & 85 & tank \\
11 & boy              & 36 & hamster        & 61 & plate          & 86 & telephone \\
12 & bridge           & 37 & house          & 62 & poppy          & 87 & television \\
13 & bus              & 38 & kangaroo       & 63 & porcupine      & 88 & tiger \\
14 & butterfly        & 39 & keyboard       & 64 & possum         & 89 & tractor \\
15 & camel            & 40 & lamp           & 65 & rabbit         & 90 & train \\
16 & can              & 41 & lawn\_mower    & 66 & raccoon        & 91 & trout \\
17 & castle           & 42 & leopard        & 67 & ray            & 92 & tulip \\
18 & caterpillar      & 43 & lion           & 68 & road           & 93 & turtle \\
19 & cattle           & 44 & lizard         & 69 & rocket         & 94 & wardrobe \\
20 & chair            & 45 & lobster        & 70 & rose           & 95 & whale \\
21 & chimpanzee       & 46 & man            & 71 & sea            & 96 & willow\_tree \\
22 & clock            & 47 & maple\_tree    & 72 & seal           & 97 & wolf \\
23 & cloud            & 48 & motorcycle     & 73 & shark          & 98 & woman \\
24 & cockroach        & 49 & mountain       & 74 & shrew          & 99 & worm \\
\bottomrule
\end{tabular}
\end{table}

\subsection{Descriptions of Compared Approaches}
\label{app:compared_approaches}

In this subsection, we provide detailed descriptions of the existing real-world complementary-label datasets used for comparison: CLImage~\cite{wang2024climage} and ACLImage~\cite{aclimage2025}. We further discuss the methodological differences regarding the specific annotators (human vs. VLM) used in these approaches compared to our proposed BICL.

\subsubsection{CLImage: Human-Annotated Complementary Labels}
CLImage~\cite{wang2024climage} represents the first effort to collect CLs from human annotators to capture real-world noise and bias distributions. The dataset collection process involves utilizing crowdsourcing platforms (i.e., Amazon Mechanical Turk) where human workers are presented with images and asked to select a label that does \textit{not} match the image content.

\textbf{Protocol.}
Annotators are presented with a randomly sampled subset of candidate labels (specifically 4 labels) from the label space and asked to select one incorrect class. This process relies on human subjectivity, where annotators may select CLs based on visual dissimilarity or random choice.

\paragraph{Characteristics.}
According to recent benchmarks~\cite{libcll_2024}, CLImage exhibits a relatively high label noise rate (i.e., 3.93 percentage points on CIFAR-10) due to human error. However, interestingly, it maintains a relatively low imbalance ratio (approx. 1.56), suggesting that human errors and choices are somewhat more evenly distributed across classes compared to model-based annotation.

\subsubsection{ACLImage: VLM-Annotated Complementary Labels}
To address the scalability issues and high costs of human annotation, ACLImage~\cite{aclimage2025} automates the collection process using VLMs, specifically leveraging models like \textbf{LLaVA}~\cite{LLava}.

\textbf{Protocol.}
ACLImage adopts the same protocol as CLImage to generate CLs. For each image:
\begin{enumerate}
    \item A candidate set of 4 labels is uniformly sampled from the label space.
    \item The VLM (LLaVA) is prompted to select the label from this set that does \textit{not} describe the image.
\end{enumerate}

\textbf{Characteristics.}
While using the same protocol, the VLM annotator significantly reduces label noise, achieving a noise rate of approximately 0.24 percentage points on CIFAR-10, which is much lower than CLImage. However, contrary to the expectation that uniform candidate sampling would yield a uniform distribution, ACLImage exhibits a \textit{higher} imbalance ratio (approx. 3.26) compared to CLImage (1.56). This indicates that even with restricted candidates, VLMs exhibit inherent biases towards selecting specific ``safe'' or common negative labels.

\subsubsection{Comparison with BICL}
Our proposed BICL differs fundamentally from these approaches in design intent and resulting distribution.

\textbf{Bias Structure.}
While ACLImage exhibits moderate bias (imbalance ratio $\approx$ 3.26) as an unintended side effect of VLM preferences, BICL explicitly \textit{induces} extreme bias (imbalance ratio $\approx$ 4966.5) through our deterministic cluster-based candidate selection. We argue that this structured, high-bias signal is more informative for learning than the lower-bias distribution of CLImage or the moderate bias of ACLImage.

\textbf{Noise vs. Bias Trade-off.}
CLImage represents a high-noise, low-bias regime (Human). ACLImage represents a low-noise, moderate-bias regime (LLaVA). BICL represents a low-noise, high-bias regime (LLaVA + Clustering). Our experiments demonstrate that in the context of complementary-label learning, pushing the bias to extreme levels (as in BICL) while maintaining low noise yields superior classifier performance.

\subsection{Details of Model Architecture}
For all experiments, we employed ResNet architectures (ResNet-18, ResNet-34, and ResNet-50) as the backbone networks. To accommodate the lower resolution of the CIFAR ($32 \times 32$) and TinyImageNet ($64 \times 64$) datasets compared to standard ImageNet inputs, we modified the initial stem of the network. Specifically, we replaced the original $7 \times 7$ convolutional layer (stride 2) with a $3 \times 3$ convolutional layer with stride 1 and padding 1. This modification prevents the loss of fine-grained spatial information during the early downsampling stages. The network weights were initialized using Kaiming initialization~\cite{He2015}.

\subsection{Complementary-Label Algorithms}
\label{baseline-cll-algorithms}

In this section, we review the six CLL algorithms used as baselines in our evaluation. These methods can be broadly categorized into $Q$-aware methods, which utilize a transition matrix $Q$ to model the relationship between latent true labels and observed CLs.

\subsubsection{Forward Correction (FWD).}
Proposed by~\citep{fwd2018}, the Forward Correction (FWD) method provides a way to learn from CLs by correcting the loss function using the transition matrix $Q$. Instead of predicting the complementary label directly, FWD utilizes the transition matrix to map the predicted probability of the true label to the probability of the complementary label.

Let $g(\mathbf{x})$ denote the classifier outputs (logits) and $\sigma(\cdot)$ be the softmax function. The loss function for FWD is defined as:
\begin{equation}
    R_{\text{FWD}}(g) = \frac{1}{N}\sum_{i=1}^{N} \mathcal{L}\left(Q^\top \sigma(g(\mathbf{x}_i)), \bar{y}_i\right),
    \nonumber
\end{equation}
where $N$ is the number of samples, $\bar{y}_i$ is the complementary label, and $\mathcal{L}(\cdot, \cdot)$ is a standard loss function (i.e., Cross-Entropy). By incorporating $Q$, the model learns to approximate the true class distribution even when supervision is provided only via CLs.

\subsubsection{Unbiased Risk Estimator (URE).}
\citep{ishida2019complementarylabel} proposed a framework to estimate the classification risk unbiasedly using CLs. The core idea relies on an inverse transition matrix to recover the unbiased risk of the true classifier. The general loss formulation is:
\begin{equation}
    R_{\text{URE}}(g) = \frac{1}{N}\sum_{i=1}^{N} e_{\bar{y}_i}^\top Q^{-1} \mathcal{L}(g\mathbf{(x}_i)),
    \nonumber
\end{equation}
where $e_{\bar{y}_i}$ is the one-hot vector of the complementary label, and $\mathcal{L}(g(\mathbf{x}_i))$ denotes the vector of losses for all possible classes. However, this unbiased estimator can result in negative empirical risk values during training, leading to overfitting. To address this, two correction mechanisms were introduced:

\paragraph{URE-TNN (Non-Negative)}
The Non-Negative (NN) correction imposes a constraint that the estimated risk for any class cannot be negative. If the risk estimate for a specific class becomes negative during a training iteration, it is clipped to zero to maintain stability.

\paragraph{URE-TGA (Gradient Ascent)}
The Gradient Ascent (GA) correction strategy takes a different approach. When the empirical risk for a specific class is negative, the algorithm performs gradient ascent (reversing the gradient direction) on that component. This prevents the model from becoming over-confident in incorrect predictions caused by the negative risk estimation.

\subsubsection{Complementary Probability Estimation (CPE).}
\citep{cpe2023} introduced the Complementary Probability Estimation (CPE) framework. Unlike risk-correction methods, CPE focuses on directly estimating the probability of a label being complementary, denoted as $p(\bar{y}|\mathbf{x})$. The objective is to minimize the divergence between the model's output and the complementary target. CPE employs a surrogate complementary estimation loss (SCEL) as:
\begin{equation}
    R_{\text{CPE}}(\mathbf{g}) = \frac{1}{N} \sum_{i=1}^{N} \mathcal{L}(\mathbf{g}(\mathbf{x}_i), e_{\bar{y}}).
    \nonumber
\end{equation}

\paragraph{CPE-I (Identity)}
CPE-I utilizes a direct mapping where no explicit transition layer is applied to the model output before loss calculation. It treats the complementary probability estimation directly without structural modification involving the transition matrix $Q$.

\paragraph{CPE-F (Fixed)}
CPE-F incorporates a fixed transition layer derived from the known or estimated transition matrix $Q$. The model predicts the ordinary label distribution, which is then passed through this fixed layer to produce complementary probabilities $p(\bar{y}|\mathbf{x})$ for training.

\paragraph{CPE-T (Trainable)}
CPE-T extends CPE-F by making the transition layer trainable. This allows the model to jointly learn the classifier and fine-tune the transition matrix simultaneously, providing robustness against noise or inaccuracies in the initial transition matrix assumption.

\begin{figure}[htb]
  \centering
  \begin{subfigure}[b]{0.70\textwidth}
    \includegraphics[width=\textwidth]{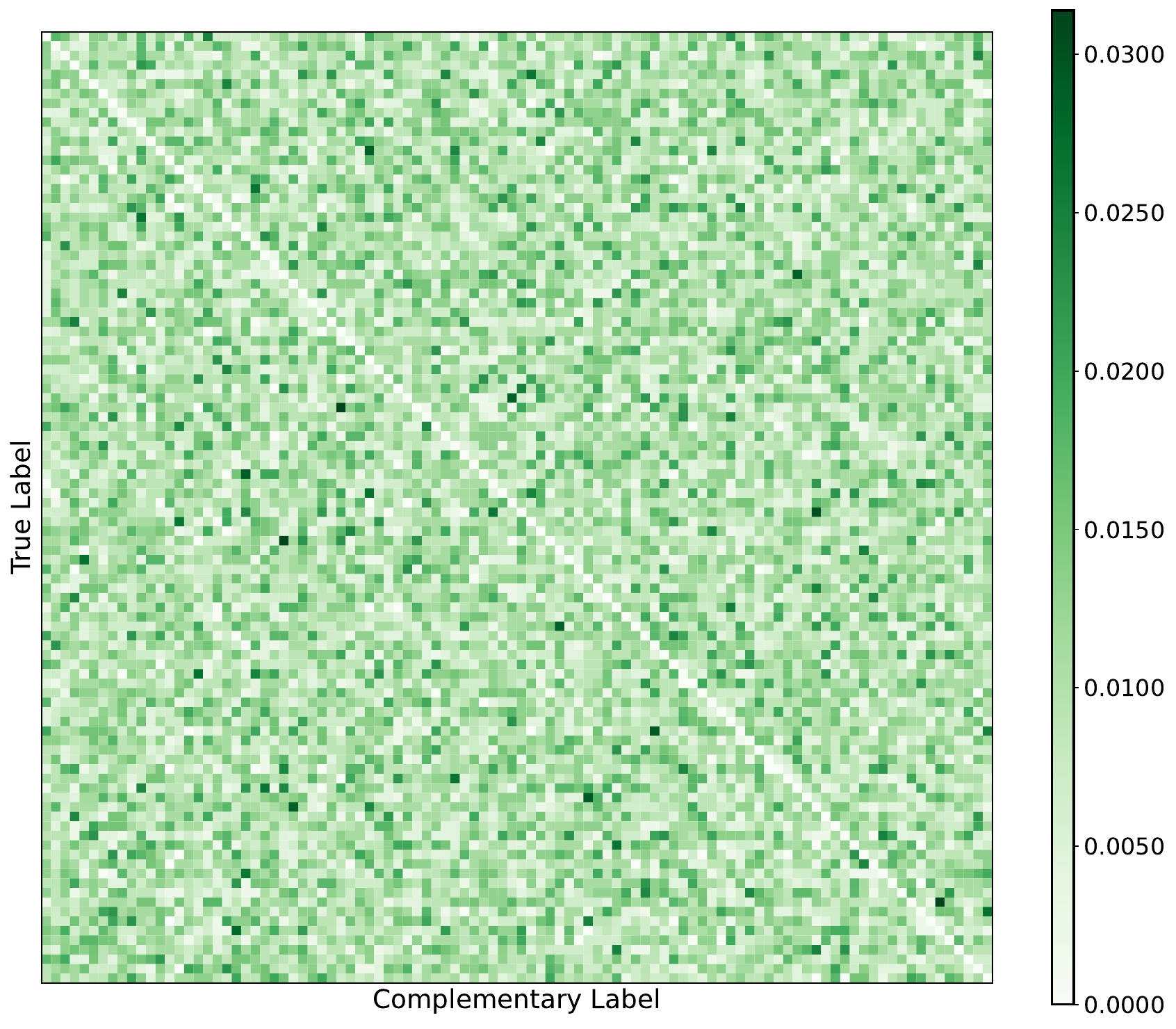}
    \caption{CIFAR-100}
  \end{subfigure}
  \begin{subfigure}[b]{0.70\textwidth}
    \includegraphics[width=\textwidth]{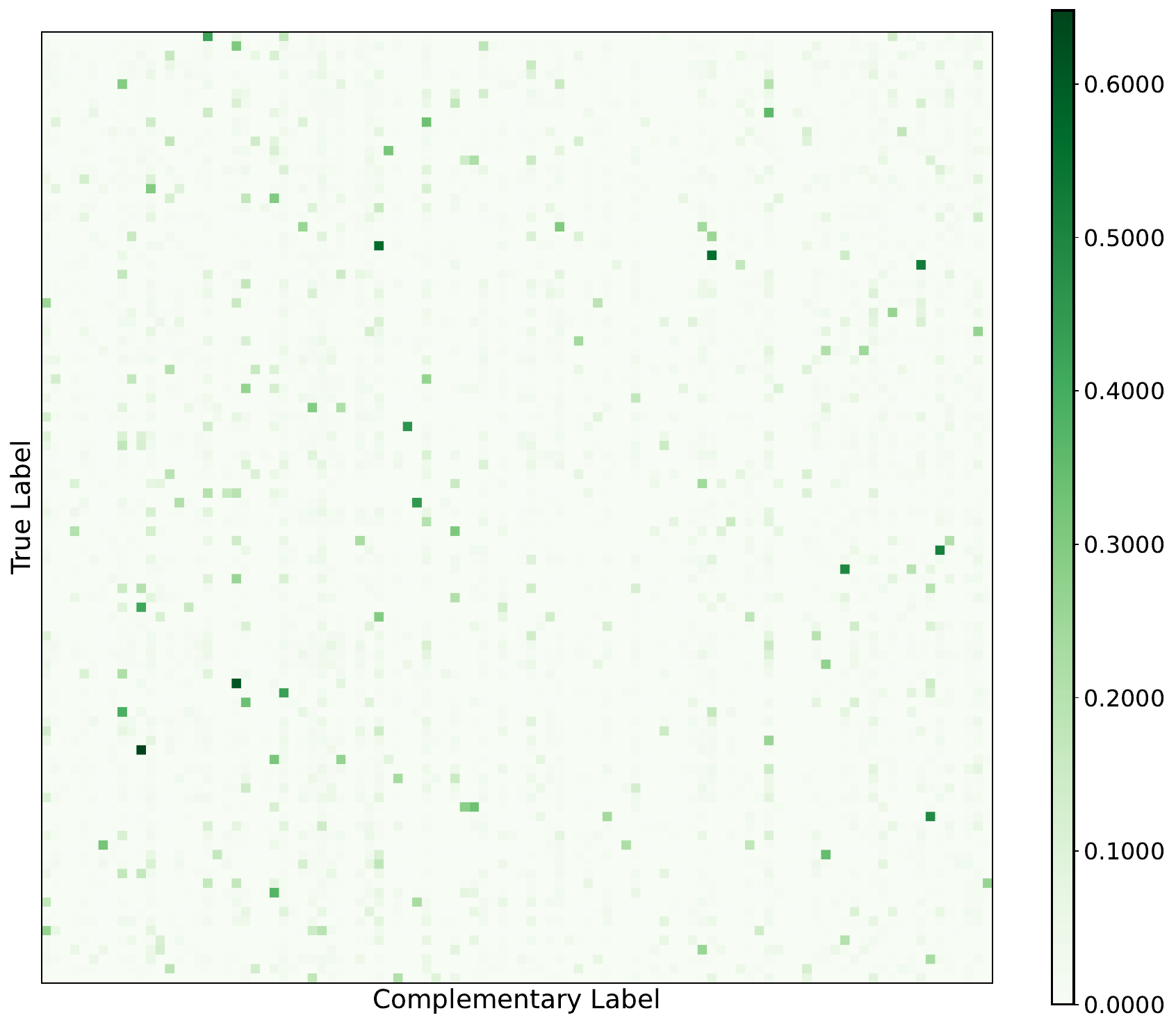}
    \caption{BCLCIFAR-100}
  \end{subfigure}
  \caption{Bias-Induced Constrained Labeling transition matrix on CIFAR-100 variants.}
  \label{fig:transition-matrix-cl100}
\end{figure}



\end{document}